\documentclass{article}
\usepackage[margin=1.2in]{geometry}
\usepackage{graphicx, epsfig,color}
\usepackage{amsthm}
\usepackage{multirow}
\usepackage{amsmath}
\usepackage{amsfonts}
\usepackage{comment}
\usepackage{color}
\newcommand{\vv}[1]{\mathbf{#1}}

\usepackage{todonotes}
\newtheorem{theorem}{Theorem}
\newtheorem{corollary}{Corollary}
\newtheorem{lemma}{Lemma}
\newtheorem{example}{Example}
\newtheorem{proposition}{Proposition}

\newtheorem{remark}{Remark}

\newcommand{\X} {{\bf X}}
\newcommand{\x} {{\bf x}}
\newcommand{\y} {{\bf y}}

\newcommand{\E}{{ \mathbb{E}}}

\usepackage{authblk}
\author[1]{Yue Xing}
\author[2]{Ruizhi Zhang}
\author[1]{Guang Cheng}
\affil[1]{Department of Statistics, Purdue University}
\affil[2]{Department of Statistics, University of Nebraska-Lincoln}

\title{
Adversarially Robust Estimate and Risk Analysis in Linear Regression
}
\begin{document}

\maketitle
\begin{abstract}
Adversarially robust learning aims to design algorithms that are robust to small adversarial perturbations on input variables. Beyond the existing studies on the predictive performance to adversarial samples, our goal is to understand statistical properties of adversarially robust estimates and analyze adversarial risk in the setup of linear regression models. By discovering the statistical minimax rate of convergence of adversarially robust estimators, we emphasize the importance of incorporating model information, e.g., sparsity, in adversarially robust learning. Further, we reveal an explicit connection of adversarial and standard estimates, and propose a straightforward two-stage adversarial learning framework, which facilitates to utilize model structure information to improve adversarial robustness. In theory, the consistency of the adversarially robust estimator is proven and its Bahadur representation is also developed for the statistical inference purpose. The proposed estimator converges in a sharp rate under either low-dimensional or sparse scenario. Moreover, our theory confirms two phenomena in adversarially robust learning: adversarial robustness hurts generalization, and unlabeled data help improve the generalization. In the end, we conduct numerical simulations to verify our theory.
\end{abstract}

\section{Introduction}\label{sec:intro}

The development of machine/deep learning methods has led to breakthrough performance in various areas of application. However, some recent research revealed that these powerful but delicate models are vulnerable to random perturbation and adversarial attacks. For example, well-designed malicious adversarial input may induce wrong decision making when filtering junk emails or detecting malicious binary programs \cite{DolphinAttack2017, papernot2017practical}. On the other hand, by studying adversarial samples, one can in turn improve adversarial robustness of algorithms in practice. The existing literature focus on generating adversarial samples, e.g., \cite{papernot2016crafting,papernot2017practical}, adversarial training, e.g., \cite{goodfellow2014explaining,kurakin2016adversarial,wang2019convergence},  invariance/interpretability to detect adversarial samples, e.g., \cite{xu2017feature,tao2018attacks,ma2019nic,etmann2019connection,carmon2019unlabeled} and  theoretical studies of adversarially robust learning, e.g., \cite{xu2009robust,xu2009robustness,xu2012robustness}. In particular, some studies \cite{yin2018rademacher,raghunathan2019adversarial} showed that adversarial training leads to a worse generalization performance, while  \cite{schmidt2018adversarially,zhai2019adversarially,najafi2019robustness} argued that the adversarial robustness requires more (labeled/unlabeled) data to enhance generalization performance. In addition, the trade-off between standard performance and adversarial performance is carefully characterized in \cite{zhang2019theoretically,adel2020precise}. 
	
Adversarially robust estimation in the literature is often formulated as an empirical ``min-max" problem: minimizing the empirical risk under the worst-case attack (which maximizes the loss) on the training data. Unfortunately, this formulation has not directly taken into account the structural information of the model such as sparsity and grouping, e.g., \cite{shaham2015understanding,sinha2018certifying,wang2019convergence}, which may be utilized to improve adversarial robustness. This is particularly needed in the high-dimensional regime, i.e., data dimension $p$ is much larger than sample size $n$, where the empirical (adversarial) risk may no longer converge to the population risk \cite{mei2018landscape}. 

The above concern raises two questions: (1) whether the statistical minimax\footnote{In this paper, ``min-max" refers to the optimization problem considered in adversarially robust learning, while ``minimax" refers to the statistical lower bound on the estimation error.} rate of the estimation error of \textit{any} linear adversarial estimator will get changed given certain structure information for the standard model, and (2) whether we can utilize these information to get better adversarially robust estimator.

Our contributions can be summarized as follows:
\begin{itemize}
    \item In Section \ref{sec:lower_bound}, by studying the form of adversarial risk, we figure out the minimax lower bound of estimation error, which reveals the potential to improve the estimation efficiency through using model information. 
    \item In Section \ref{sec:low}, we design a two-stage adversarially robust learning framework that nicely connects adversarially robust estimation with standard estimation. The model structure information can be easily embedded into the standard estimator, and is further carried over to the adversarially robust estimate through this two-stage learning procedure. For the purpose of statistical inference, we develop the Bahadur representation result \cite{he1996general} that implies the asymptotic normality of the proposed estimate under certain conditions. In addition, by analyzing the upper bound for the estimation error, we reveal the benefit of incorporating sparsity information into the adversarial estimation procedure, in which the estimator reaches the minimax optimal rate of convergence. 
    \item Besides the above two main contributions, in Section \ref{sec:two}, we utilize our theory to verify two arguments in adversarially robust learning: adversarially robust learning hurts generalization, and adversarial robustness can be improved using unlabeled data. 
\end{itemize}

There are two related works appearing very recently. The first one \cite{adel2020precise} mainly investigated the trade-off between adversarial risk and standard risk under an isotropic condition of the covariate. Rather, we focus on how to improve adversarial robustness through utilizing prior knowledge on the model, and study statistical properties of the adversarially robust estimate itself, in contrast with the generalization studies by \cite{schmidt2018adversarially,zhang2019theoretically,zhai2019adversarially,najafi2019robustness}. Another recent work \cite{chen2020sharp} studied the sharp statistical bound in adversarially robust {\em classification}. In the regression setup, our theorems reveal that an adversarially robust estimate is different from a standard estimate even in the {rate} of convergence: for noiseless case, standard model estimators can exactly recover the correct model, but the lower bound for adversarially robust model is always nonzero. Our lower bound for sparse model is also new. 

	{\bf Notation.} We use boldface font for vectors, e.g., $ \vv x$, and capital letters for matrices, e.g., $\vv A$. The $\ell_2$ norm of a vector $\vv u$ is denoted as $\|\vv u\|_2$ (or $\|\vv u\|$ for simplicity). The $p\times p$ identity matrix is denoted by $\vv I_p.$The induced spectral norm of a matrix $\vv A\in \mathbb{R}^{p\times p}$ is denoted by $\|\vv A\|$, i.e., $\|\vv A\| := \sup\{\|\vv A\vv x\|: \|\vv x\| = 1\}.$ We denote by $\lambda_i(\vv A),i\in\{1,2,\cdots,p\}$, its eigenvalues in decreasing
	order. For any symmetric matrix $\vv A,$ denote $\|\vv x\|_{\vv A}^2={\vv x^{\top}\vv A\vv x}.$ For two matrices $\vv A,\vv B,$ we denote $\left\langle \vv A,\vv B\right\rangle_{F}$ as the Frobenius inner product, which is the sum of component-wise inner product of two matrices. The Frobenius norm of a matrix $\vv A$ is denoted by $\|\vv A\|_{F}$.
	
	\section{Properties of Adversarial Risk}\label{sec:population}
	Consider a linear regression model
	\begin{eqnarray}\label{eqn:model}
	y=\x^{\top}\theta_0+\epsilon,
	\end{eqnarray}
	where $\mathbb{E}\x=\vv 0$, Var$(\x)=\Sigma$, and $\epsilon$ is a noise term (independent of $\x$) with $\E(\epsilon) =0$ and Var$(\epsilon)=\sigma^2$. Throughout this paper, we assume that $\x\in\mathbb{R}^{p}$ follows a $p$-dimensional Gaussian distribution and $\Sigma$ has a bounded largest eigenvalue (away from $\infty$) and a bounded smallest eigenvalue (away from $0$) as $p$ increases. The noise variance $\sigma^2$ and $\|\theta_0\|$ are allowed to diverge in $p$, and the signal-to-noise ratio $\|\theta_0\|_{\Sigma}/\sigma$ needs to be large enough, say bounded away from 0.
	
	The (population) adversarial risk is defined as follows
	\begin{eqnarray} \label{eqn:risk}
	R_0(\theta,\delta):=\E_{\vv x} \underset{\|\vv x^*-\vv x\|_2\le \delta}{\max}\left[(\vv (\x^*)^{\top}\theta-\vv x^{\top}\theta_0)^2\right]=\|\theta-\theta_0\|_{\Sigma}^2+2\delta c_0 \|\theta-\theta_0\|_{\Sigma}\|\theta\|+\delta^2\|\theta\|^2,
	\end{eqnarray}
	where $c_0:= \sqrt{2/\pi} $. The corresponding minimizer of (\ref{eqn:risk}) is denoted by $\theta^*(\delta),$ i.e., $$\theta^*(\delta):=\arg\min_{\theta}  R_0(\theta, \delta).$$ We may just use $\theta^*$ when no confusion arises.
	
	In the proposition below, we study the shape of $R_0$, and establish an analytical form of $\theta^*(\delta)$, which suggests the construction of adversarially robust estimator (to be specified later). Define $$\theta(\lambda):=(\Sigma+\lambda\vv I_p)^{-1}\Sigma\theta_0,$$ and two thresholds of $\delta$:
	\begin{eqnarray*}
		\delta_1=\frac{ c_0\|\theta_0\|}{\|\theta_0\|_{\Sigma^{-1}}}\;\;\mbox{and}\;\;\delta_2=\frac{\|\theta_0\|_{\Sigma^2}}{ c_0  \|\theta_0\|_{\Sigma}}.
	\end{eqnarray*}
	\begin{proposition}\label{thm:opt}
		The risk $R_0(\theta,\delta)$ is a convex function w.r.t. $\theta$, and has positive definite Hessian for any $\theta\neq \vv 0, \theta\neq \theta_0$. In addition, the global minimizer of $R_0(\theta,\delta)$ can be written as
		\begin{eqnarray}\label{thetas}
		\theta^*(\delta):=\theta(\lambda^*(\delta))
		\end{eqnarray}
		where $\lambda^*(\delta)$  depends on $(\delta,\Sigma,\theta_0)$. (1) If $\delta\leq\delta_1$, then $\lambda^*(\delta)=0$ such that $\theta^*=\theta_0$, and there is no stationary point for $R_0(\theta,\delta)$. (2) If $\delta\geq \delta_2$, then $\lambda^*(\delta)=\infty$ such that $\theta^*=\vv 0$, and there is no stationary point for $R_0(\theta,\delta)$. (3) If $\delta_1<\delta<\delta_2$, then there is a unique stationary point  $\theta(\lambda^*(\delta))$ of $R_0(\theta,\delta),$ which is the global optimum. Here $\lambda^*(\delta)$ is the solution of the following equation w.r.t. $\lambda$:
		\begin{eqnarray}
		\lambda\left(1+\delta  c_0  \frac{\|\theta(\lambda)\|}{\|\theta(\lambda)-\theta_0\|_{\Sigma}}\right)= \left(\delta  c_0  \frac{\|\theta(\lambda)-\theta_0\|_{\Sigma}}{\|\theta(\lambda)\|}+\delta^2\right).\label{eqn:cond}
		\end{eqnarray}

	\end{proposition}
	
	The proof of Proposition \ref{thm:opt} is postponed to Appendix \ref{sec:appendix:population}.  
	
For a general $\Sigma$, it is hard to obtain an explicit solution for $\theta^*$ by solving (\ref{eqn:cond}). However, when $\Sigma=\vv I_p$, one can write down the explicit formula of $\theta^*(\delta)$, which is actually a re-scaled version of $\theta_0$. In this case, 		 $\delta_1=c_0$, $\delta_2=1/c_0$, and $\lambda^*(\delta)=(\delta^2-\delta c_0)/(1-\delta c_0)$ when $\delta\in(\delta_1,\delta_2)$. Moreover, the adversarial risk and standard risk of the adversarially robust model become
		\begin{eqnarray*}
			R_0(\theta^*(\delta),\delta)=\left\{
			\begin{array}{ll}
				\delta^2\|\theta_0\|^2& \hbox{$\delta\leq c_0 $ } \\
				\frac{\delta^2(1-c_0^2)}{\delta^2+1-2\delta c_0 }\|\theta_0\|^2& \hbox{$c_0 \le\delta\leq  1/c_0$ }\\
				\|\theta_0\|^2& \hbox{$\delta\ge 1/c_0$ }
			\end{array}
			\right.
		\end{eqnarray*}
	\begin{eqnarray*}
	    R_0(\theta^*(\delta),0)=\left\{
			\begin{array}{ll}
				0& \hbox{$\delta\le c_0 $ } \\
				\frac{\delta^2(\delta- c_0)^2}{(\delta^2+1-2\delta c_0)^2 }\|\theta_0\|^2& \hbox{$ c_0 \le\delta\leq  1/c_0$ }\\
			\|\theta_0\|^2& \hbox{$\delta\ge 1/c_0$ }
			\end{array}
			\right. .
	\end{eqnarray*}
Similar as $R_0(\theta^*(\delta),\delta)$, the standard risk of the adversarially robust model $R_0(\theta^*(\delta),0)$ also increases as $\delta$ and reaches the same level as $R_0(\theta^*(\delta),\delta)$ when $\delta>1/c_0$; see Figure \ref{fig:coro} below. This result echoes with \cite{adel2020precise,raghunathan2019adversarial} that the adversarially robust model leads to a worse performance when testing data is un-corrupted.
	\begin{figure}[!ht]
		\centering
		\includegraphics[scale=0.3]{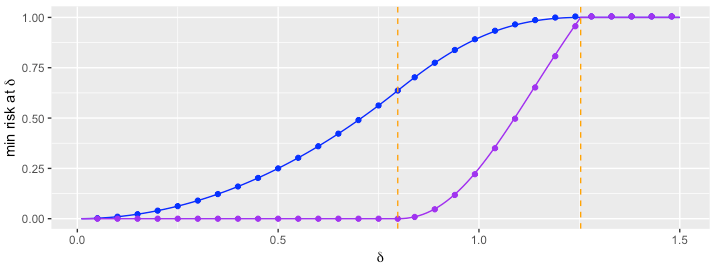}
		\caption{$R_0(\theta^*(\delta),\delta)$ and $R_0(\theta^*(\delta),0)$ correspond to blue and purple curves, respectively. Here, $\Sigma=\vv I_p$ and $\|\theta_0\|^2=1$. Dashed lines represent the two thresholds $\delta_1= c_0 $ (left) and $\delta_2=1/c_0$ (right). The curves are theoretical values for minimal adversarial risk solved from (\ref{eqn:cond}). The dots were obtained from simulations with $p=10$ and $n=10000$.}\label{fig:coro}
	\end{figure}
	
	\begin{remark}
		Besides adversarial risk, we define adversarial {\bf prediction} risk as $$R(\theta,\delta):=\E_{\vv x,y} \underset{\|\vv x^*-\vv x\|\le \delta}{\max}\left[\left(\vv (\x^*)^{\top}\theta-y\right)^2\right].$$
		The properties of $R$ are similar as $R_0$ when $\epsilon\sim N(0,\sigma^2)$, and we focus on $R_0$ in this paper.
	\end{remark}
	
\section{Minimax Lower Bound}\label{sec:lower_bound}
In this section, through figuring out the minimax lower bounds of the estimation error, we argue that it is essential to incorporate sparsity information of $(\theta_0,\Sigma)$ in $(\widehat\theta_0,\widehat\Sigma)$ in sparse model. For minimax lower bound in standard learning problems, studies can be found in 
\cite{dicker2016ridge,mourtada2019exact} for dense case and  \cite{verzelen2010high,ye2010rate,raskutti2011minimax} for sparse case.
	
	The following two theorems present the lower bounds of $\mathbb{E}\|\hat{\theta}-\theta^*\|^2$ 
	for dense/sparse models respectively. 
	
\begin{theorem}\label{thm:minimax:dense}
	When $\sigma/\|\theta_0\|<\infty$, $\sigma^2p/(\|\theta_0\|^2n)\rightarrow 0$, and $(p\log^2 n)/n\rightarrow 0$, if $\|\theta_0\|\leq R$, $0<c_1\leq\lambda_{\min}(\Sigma)\leq\lambda_{\max}(\Sigma)\leq c_2<\infty$, $\delta>0$, then there exists some constant $\delta>0$ such that
	\begin{eqnarray*}
	\inf_{\widehat{\theta}}\sup_{\Sigma,\theta_0,\delta} \mathbb{E}\|\widehat{\theta}-\theta^*\|^2=\Omega\left(  \frac{p\sigma^2}{n}\vee \frac{pR^2}{n}  \right),
	\end{eqnarray*}
	The estimator $\widehat{\theta}$ refers to any estimator $\widehat{\theta}(X,Y,\delta)$, and $\theta^*$ is a function of $(\theta_0,\Sigma,\delta)$.  
\end{theorem}

For sparse model, the sparsity of $\theta_0$ is directly controlled through the size of active set of $\theta_0$. In terms of the sparsity of $\Sigma$, we follow \cite{cai2016optimal} to consider a family of sparse covariance matrix as follows:
	\begin{eqnarray*}
	\mathcal{F}_{\alpha}=\bigg\{ \Sigma:\max_j\sum_{i}\{ |\sigma_{ij}|:|i-j|>k \} \leq Mk^{-\alpha}\;\forall k,\lambda_{\max}(\Sigma)\leq M_0,\;\lambda_{\min}(\Sigma)\geq m_0>0 \bigg\}.
	\end{eqnarray*}
	\begin{theorem}\label{thm:minimax:sparse}
When $\sigma/\|\theta_0\|<\infty$, if $\|\theta_0\|\leq R$ and $\|\theta_0\|_0\leq s$, $0<c_1\leq\lambda_{\min}(\Sigma)\leq\lambda_{\max}(\Sigma)\leq c_2<\infty$, $\delta>0$, then for any $0<s<p$ and $\alpha>0$, there exists some constant $\delta>0$ such that
	\begin{eqnarray*}
	\inf_{\widehat{\theta}}\sup_{\Sigma\in\mathcal{F}_{\alpha},\theta_0,\delta}\mathbb{E}\|\widehat{\theta}-\theta^*\|^2=\Omega\left( s\sigma^2\frac{1+\log (p/s)}{n}\vee R^2n^{-\frac{2\alpha}{2\alpha+1}} \right).
	\end{eqnarray*}
\end{theorem}
 The proof of the above two theorems utilize some tools in \cite{mourtada2019exact,verzelen2010high,cai2016optimal}. A difficulty compared with existing literature in standard learning is that the relationship between $\theta_0$ and $\theta^*$ is nonlinear, and $\theta^*$ further depends on $\Sigma$.  The details are in Appendix \ref{sec:appendix:low}.

To compare Theorem \ref{thm:minimax:dense} and \ref{thm:minimax:sparse}, the lower bound for sparse model is much smaller than the one for dense model. This indicates that there is a potential improvement for adversarially robust estimators if the algorithm can utilize the sparsity information (if there is). As discussed in \cite{belkin2019two,xing2020generalization}, for high-dimensional model, if we do not consider the sparsity information, the resulting model is not consistent in both standard and adversarially robust learning problems.

 To compare with standard learning problem, the results in Theorem \ref{thm:minimax:dense} and \ref{thm:minimax:sparse} are different from those in standard learning. Such a difference implies it is hard to train adversarially robust models. In standard learning, under either dense or sparse model, when $\sigma^2=0$, the lower bound is exact zero since some estimators of $\theta_0$ can achieve zero estimation error. However, when $\delta>0$, even if $\sigma^2=0$, the lower bound is not zero. 
 
\begin{remark}
Similar to our results, \cite{chen2020sharp} provided minimax lower bound of generalization error under the adversarially robust classification setup. However, they only considered the dense case corresponding to our Theorem \ref{thm:minimax:dense}, but not for the sparse case.
\end{remark}

\begin{remark}
Similar as for adversarial risk, to minimize adversarial prediction risk, the estimators have lower bounds in the same rate as in Theorem \ref{thm:minimax:dense} and \ref{thm:minimax:sparse}.
\end{remark}

	\section{Two-stage Adversarially Robust Estimator}\label{sec:low}
In this section, we propose a two-stage procedure for constructing adversarially robust estimators based on the explicit relation pointed out in the previous section. This relation allows us to incorporate specific model information, such as sparsity, into adversarially robust estimates  through standard estimate. The idea of the proposed method is similar to the estimators in \cite{chen2020sharp,carmon2019unlabeled} and the method is straightforward. We emphasize that such a simple two-stage method is powerful enough to achieve minimax optimal.

\subsection{Estimator description}
	There are two stages in the proposed method. In the first stage, consistent estimators of the true parameter $\theta_0$, denoted as  $\widehat{\theta}_0$, and  matrix $\Sigma$, denoted as $\widehat{\Sigma}$, are obtained from standard statistical procedures. In the second stage, the robust estimator of $\theta^*$, which minimizes the adversarial risk, is constructed as follows: 
	\begin{eqnarray}\label{eqn:our_estimator}
	\widehat\theta(\delta):=\widehat{\theta}(\widehat\lambda^*(\delta)):=(\widehat{\Sigma}+\widehat\lambda^*(\delta) \vv I_p)^{-1}\widehat\Sigma\widehat{\theta}_0,
	\end{eqnarray}
	where $\widehat\lambda^*(\delta)$ is a plug-in estimate of $\lambda^*(\delta)$ depending on $\widehat\theta_0$ and $\widehat\Sigma$. Alternatively speaking, $\widehat\theta(\delta)$ may be obtained by minimizing an empirical version of (\ref{eqn:risk}):
	\begin{eqnarray}\label{eqn:empirical}
	\widehat{R}_0(\theta,\delta):=\widehat R_0(\theta,\widehat{\theta}_0,\widehat{\Sigma},\delta)=\|\theta-\widehat{\theta}_0\|_{\widehat{\Sigma}}^2+2\delta  c_0  \|\theta-\widehat\theta_0\|_{\widehat{\Sigma}}\|\theta\|+\|\theta\|^2.
	\end{eqnarray}
According to the proof of Proposition \ref{thm:opt}, the empirical risk $\widehat{R}_0(\theta,\delta)$ shares similar properties as adversarial risk ${R}_0(\theta,\delta)$ in Proposition \ref{thm:opt}. 	We may simply use $\widehat{\theta}$ instead of $\widehat\theta(\delta)$  when no confusion arises. 
	
	\subsection{Consistency}
We first show that for any level of attack $\delta,$ the adversarial excess risk converges to zero, i.e., (\ref{eqn:thm3}), as long as the standard estimates of $\theta_0$ and $\Sigma$ are consistent with proper rates and $p$ does not grow too fast. Next, combining the convex properties of $R_0$, the upper bound in (\ref{eqn:thm3}) implies the consistency of $\widehat{\theta}$ in estimating $\theta^*$; see Theorem \ref{thm:bah}. This consistency result will be used in deriving the generalization error in Theorem \ref{coro:generalization} later.

	\begin{theorem}\label{thm:lsm}
		For any consistent estimators $\widehat{\theta}_0$ and $\widehat{\Sigma}$, with probability tending to 1,
		\begin{eqnarray}\label{eqn:thm3}
		\sup\limits_{\delta\geq 0} \left|R_0(\theta^*(\delta),\delta)-R_0(\widehat{\theta}(\delta),\delta)\right|=O\left( \|\widehat{\theta}_0-\theta_0\|\|\theta_0\| \right)+O\left(\|\theta_0\|^2\sqrt{\|\widehat{\Sigma}-\Sigma\|}\right).\nonumber
		\end{eqnarray}
	\end{theorem}
	
	To illustrate Theorem~\ref{thm:lsm} in details, we use $\widehat\theta_0=(\X^{\top}\X)^{-1}\X^{\top}\y$ and $\widehat\Sigma=\X^{\top}\X/n$ to construct $\widehat\theta.$ Based on Theorem 2 in \cite{hsu2012random} (taking ridge penalty as zero) and Theorem \ref{thm:lsm},  we have with probability tending to 1,
	\begin{equation}\label{eg0}
	   \frac{ R_0(\widehat{\theta},\delta)-R_0(\theta^*,\delta)}{\|\theta_0\|_{\Sigma}^2+\sigma^2}=o(1),
	\end{equation}
which implies the adversarial excess risk of $\widehat{\theta}$ converges to zero as long as $(p\log n)/n\rightarrow 0.$	

The proof of Theorem~\ref{thm:lsm} is postponed to Appendix \ref{sec:appendix:low}. We also postpone an analog of Theorem \ref{thm:lsm} for the adversarial prediction risk $R$ to Appendix \ref{sec:appendix:more} (for the statement) and \ref{sec:appendix:low} (for the proof). Note that 
the upper bound in (\ref{eqn:thm3}) is not tight, but enough to justify the adversarial risk consistency of $\widehat{\theta}(\delta)$.

We next use an example to illustrate how sparsity information can be utilized in the proposed framework.

	\begin{example}[Sparse Standard Estimates]\label{example:sparse}
	Assume  matrix belongs to the family $\mathcal{F}_{\alpha}$, then using the sparse estimator $\widehat\Sigma$ in \cite{cai2016optimal}, we have
	\begin{eqnarray*}
	\mathbb{E}\|\widehat{\Sigma}-\Sigma\|^2= O\left(n^{-\frac{2\alpha}{2\alpha+1}}+\frac{\log p}{n}\right).
	\end{eqnarray*}
Assume $\widehat\theta_0$ is the LASSO estimate obtained under proper penalization. Denote $s<n$ as the number of nonzero coefficients in $\theta_0$. When $\x$ follows Gaussian and the noise $\epsilon$ satisfies $\mathbb{E}\exp\{ t\epsilon^2 \}<\infty$ for some $t>0$, based on \cite{bickel2009simultaneous,jeng2018post}, we have with probability tending to 1,
\begin{eqnarray*}
	\|\widehat{\theta}_0-\theta_0\|=O\left(\sigma\sqrt{\frac{s\log p}{n}}\right).
	\end{eqnarray*}
Therefore, (\ref{eg0}) holds under weaker conditions, say  $(\sigma s\log p)/n\rightarrow 0$ and $(\log p)/n\rightarrow 0$. On the other hand, we point out that $\widehat\theta$ $(\theta^*)$ does not inherit the sparsity of $\widehat\theta_0$ $(\theta_0)$ according to (\ref{eqn:our_estimator}) and (\ref{thetas}).	
\end{example}

\subsection{Bahadur representation and convergence rate}
We next study statistical properties of the adversarially robust estimator $\widehat\theta$ by establishing its Bahadur representation \cite{he1996general} that implies asymptotic normality in some cases. 
	
	\begin{theorem}\label{thm:bah}
		 Assume both $\|\widehat{\theta}_0-\theta_0\|/\|\theta_0\|$ and $\|\widehat{\Sigma}-\Sigma\|$ converge to zero in probability.
		
		(1) If $\delta\in(\delta_1,\delta_2)$, then $\widehat{\theta}-\theta^*$ is a linear combination of $\widehat{\theta}_0-\theta_0$ and $\widehat{\Sigma}-\Sigma$ in the main term:
		\begin{eqnarray*}
			\widehat{\theta}-\theta^*
			&=&\vv M_1(\theta^*,\theta_0,\Sigma)(\widehat{\theta}_0-\theta_0)+(\theta^*-\theta_0)^{\top}(\widehat{\Sigma}-\Sigma)(\theta^*-\theta_0)\vv M_2(\theta^*,\theta_0,\Sigma) \\&&+\vv M_3(\theta^*,\theta_0,\Sigma) (\widehat{\Sigma}-\Sigma)(\theta^*-\theta_0)+o_p(\|\widehat{\theta}-\theta^*\|),
		\end{eqnarray*}
		where $\vv M_1$, $\vv M_2$, and $\vv M_3$ are functions of $(\delta,\theta_0,\Sigma,\theta^*)$, and detailed formulas are postponed to Appendix \ref{sec:appendix:more}.
	
		(2) If $\delta <\delta_1$, then
		$
		\widehat{\theta}-\theta^*=\widehat{\theta}_0-\theta_0+o_p(\|\widehat{\Sigma}-\Sigma\|)+o_p(\|\widehat{\theta}_0-\theta_0\|).
		$
		
		(3) If $\delta>\delta_2$, we have
		$
		\widehat{\theta}-\theta^*=o_p(\|\widehat{\Sigma}-\Sigma\|)+o_p(\|\widehat{\theta}_0-\theta_0\|).
		$
		
	\end{theorem}
	
The proof for Theorem \ref{thm:bah} is postponed to Appendix \ref{sec:appendix:low}. We next illustrate how the Bahadur representation can be used to infer the asymptotic normality of $\widehat\theta$. 

\begin{example}[Least Square Estimate]\label{example:1} Consider the least square estimate (OLS)
		\begin{eqnarray*}
	\widehat{\theta}_0=(\X^{\top}\X)^{-1}\X^{\top}\y,\qquad \widehat{\Sigma}=\frac{1}{n}\X^{\top}\X.
		\end{eqnarray*}
		It is trivial to see that $\widehat\theta=\vv 0$ in probability when $\delta>\delta_2$ based on Theorems \ref{thm:opt} and \ref{thm:bah}. When $\delta\in[0,\delta_1)$, the asymptotic normality of $\sqrt{n/p}(\widehat\theta-\theta^*)$ trivially follows the fact that $\widehat\theta=\widehat\theta_0$ in probability and $\theta^*=\theta_0$. When $\delta\in(\delta_1,\delta_2)$,
		\begin{eqnarray*}
			\widehat{\theta}-\theta^*=\vv m_1+\vv m_2 +\vv m_3+o_p(\|\widehat{\theta}-\theta^*\|),
		\end{eqnarray*}
		where
		\begin{eqnarray*}
		   \vv m_1&=&\vv M_1\left[\frac{\Sigma^{-1}}{n}\sum_{i=1}^nx_i \epsilon_i\right], \\
		   \vv m_2&=&\vv M_2\left[\frac{1}{n}\sum_{i=1}^n(\theta^*-\theta_0)^{\top}(x_ix_i^{\top}-\Sigma)(\theta^*-\theta_0)\right],\\
		   \vv m_3&=&\vv M_3\left[\frac{1}{n}\sum_{i=1}^n(x_ix_i^{\top}-\Sigma)(\theta^*-\theta_0)\right] .
		\end{eqnarray*}
		If $p$ is fixed and $\delta\in(\delta_1,\delta_2)$, then $\sqrt{n}(\widehat{\theta}-\theta^*)$ asymptotically converges to a zero-mean Gaussian. For inference purpose, we need to estimate $Var(\widehat\theta)$. Since $x_i\epsilon_i$ and $(x_ix_i^{\top}-\Sigma)$ in $\vv m_1,\vv m_2,\vv m_3$ are both i.i.d. random variables, and  $x_i$ follows Gaussian distribution, one can figure out the variance of $\widehat\theta$. As a result, replacing $(\theta^*,\theta_0,\Sigma,\delta)$ with $(\widehat{\theta},\widehat\theta_0,\widehat\Sigma,\delta)$, one can obtain an estimate of $Var(\widehat\theta)$. As a side remark, if $p$ diverges in $n$, we have $\|\widehat{\theta}-\theta^*\|/\sqrt{\|\theta_0\|_{\Sigma}^2+\sigma^2}=O_p(\sqrt{p/n})$.
	\end{example}

Furthermore, when using dense/sparse estimators of $(\theta_0,\Sigma)$, our proposed two-stage estimator achieves minimax rate optimal in dense/sparse models respectively. The upper bound of  $\mathbb{E}\|\widehat{\theta}-\theta^*\|^2$
can be developed from Theorem \ref{thm:bah}:

\begin{corollary}\label{prop:dense}
Denote $v^2=\|\theta_0\|^2_{\Sigma}+\sigma^2$. When $(p\log n)/n\rightarrow 0$, $\widehat{\theta}_0$ is the OLS estimate, and $\widehat\Sigma$ is the sample  matrix, we have
\begin{eqnarray*}
    &&\mathbb{E}\|\widehat{\theta}-\theta^*\|^2=\Theta\left( \frac{v^2 p}{n} \right).
\end{eqnarray*}
\end{corollary}

Combining upper bound result in the above corollary and lower bound in Theorem \ref{thm:minimax:dense} together, one can see that using OLS estimate as $\widehat\theta_0$ and sample covariance matrix as $\widehat\Sigma$ in the two-stage method reaches minimax optimal in dense models. In addition, as stated in the following result, using the sparse estimators in Example \ref{example:sparse}, our proposed two-stage estimator reaches the minimax rate as in Theorem \ref{thm:minimax:sparse}:
	\begin{corollary}\label{prop:sparse}
	For sparse models, when  $(\log p)/n\rightarrow 0$, $\sigma^2(s\log p)/(n\|\theta_0\|^2)\rightarrow 0$, $\widehat{\theta}_0$ is the LASSO estimate and $\widehat{\Sigma}$ is the sparse covariance estimator in \cite{cai2016optimal}, it satisfies that 
	\begin{eqnarray*}
	    &&\mathbb{E}\|\widehat{\theta}-\theta^*\|^2=O\left(\frac{s\sigma^2\log p}{n}+ v^2n^{-\frac{2\alpha}{2\alpha+1}} \right).
	\end{eqnarray*}
If $\log_s(p)>1+c_s$ for some constant $c_s>0$, the above results are minimax-optimal.
	\end{corollary}

\section{Properties of the Proposed Method}\label{sec:two}
This section provides additional properties of the proposed method beyond the consistency and convergence rate. In particular, we use theorems associated with our method to verify two arguments in existing literature: (1) generalization of adversarially robust learning is worse than standard learning; (2) one can improve the generalization of adversarially robust learning through utilizing extra unlabeled data.
 	\subsection{Adversarial learning hurts generalization}
    We study the generalization of our proposed estimator. From the minimax lower bound theorems in Section \ref{sec:lower_bound}, it is easy to see that the excess risk when $\delta>0$ may converge in a slower rate than the one when $\epsilon=0$. Besides this, we work on the multiplicative constants of excess risk and generalization error and reveal that those constants are larger when $\epsilon>0$ as well.
    
    Based on Theorem \ref{thm:bah}, 
	the generalization error (\ref{estris}) and the estimation error of minimal adversarial risk (\ref{generr}) can be decomposed as follows: 
	\begin{eqnarray}
	\label{estris}R_0(\widehat{\theta},\delta)-\widehat{R}_0(\widehat{\theta},\delta)&=&e_{1,\Sigma}(\widehat{\Sigma},\delta)+e_{1,\theta_0}(\widehat{\theta}_0,\delta)+o_p(R_0(\widehat{\theta},\delta)-\widehat{R}_0(\widehat{\theta},\delta)),\nonumber\\
\label{generr}	R_0(\theta^*,\delta)-\widehat{R}_0(\widehat{\theta},\delta)&=&e_{2,\Sigma}(\widehat{\Sigma},\delta)+e_{2,\theta_0}(\widehat{\theta}_0,\delta)+o_p(R_0(\theta^*,\delta)-\widehat{R}_0(\widehat{\theta},\delta)).\nonumber
	\end{eqnarray}
	The term $e_{j,\theta_0}$ $(e_{j,\Sigma})$ represents the error component that is {\em only} caused by the estimation error of $\widehat{\theta}_0$ $(\widehat\Sigma)$. We next characterizes the forms of $e_{j,\Sigma}$ and $e_{j,\theta_0}$ with precise multiplicative constants.
	
	\begin{theorem}\label{coro:generalization}
		Under the same conditions as in Proposition \ref{thm:opt}, if $\|\widehat{\Sigma}-\Sigma\|\rightarrow 0$ and $\|\widehat{\theta}_0-\theta_0\|/\|\theta_0\|\rightarrow 0$, then when $\delta< \delta_1$, 
		\begin{eqnarray*}
			e_{1,\Sigma}(\widehat{\Sigma},\delta)&=& o_p(\|\widehat\Sigma-\Sigma\|\|\theta_0\|^2),\\
			e_{1,\theta_0}(\widehat{\theta}_0,\delta)&=& \|\widehat{\theta}_0-\theta_0\|_{\Sigma}^2+2c_0\delta\|\theta_0\|\|\widehat{\theta}_0-\theta_0\|_{\Sigma}\\&&+o_p(\|\widehat\theta_0-\theta_0\|\|\theta_0\|),\\
			e_{2,\Sigma}(\widehat{\Sigma},\delta)&=& o_p(\|\widehat\Sigma-\Sigma\|\|\theta_0\|^2),\\
			e_{2,\theta_0}(\widehat{\theta}_0,\delta)&=&-2\delta^2\theta_0^{\top}(\widehat{\theta}_0-\theta_0)\\&&+o_p(\|\widehat\theta_0-\theta_0\|\|\theta_0\|).
		\end{eqnarray*}
		If $\delta >\delta_1$, we have
		\begin{eqnarray*}
			e_{1,\Sigma}(\widehat{\Sigma},\delta)&=& -c_{\Sigma}(\delta)\frac{(\theta^*-\theta_0)^{\top}(\widehat{\Sigma}-\Sigma)(\theta^*-\theta_0)}{\|\theta^*-\theta_0\|_{\Sigma}^2}\\&&+o_p(\|\widehat\Sigma-\Sigma\|\|\theta_0\|^2),\\
			e_{1,\theta_0}(\widehat{\theta}_0,\delta)&=& 2c_{\theta_0}(\delta) \frac{(\widehat{\theta}_0-\theta_0)^{\top}\Sigma(\theta^*-\theta_0)}{\|\theta^*-\theta_0\|_{\Sigma}}\\&&+o_p(\|\widehat\theta_0-\theta_0\|\|\theta_0\|),\\
			e_{2,\Sigma}(\widehat{\Sigma},\delta)&=&  e_{1,\Sigma}(\widehat{\Sigma},\delta)+o_p(\|\widehat\Sigma-\Sigma\|\|\theta_0\|^2),\\
			e_{2,\theta_0}(\widehat{\theta}_0,\delta)&=& e_{1,\theta_0}(\widehat{\theta}_0,\delta) + +o_p(\|\widehat\theta_0-\theta_0\|\|\theta_0\|).
		\end{eqnarray*}
		where the multiplicative constants $c_{\Sigma}(\delta):=\|\theta^*-\theta_0\|_{\Sigma}^2+\delta c_0 \|\theta^*\|\|\theta^*-\theta_0\|_{\Sigma}$ and $c_{\theta_0}(\delta):= \|\theta^*-\theta_0\|_{\Sigma}+\delta c_0\|\theta^*\|$ are monotone increasing functions in $\delta$. Recall that $\theta^*$ is a function of $\delta$.
		
	\end{theorem}
	The proof of Theorem \ref{coro:generalization} is postponed to Appendix \ref{sec:appendix:low}.
	
	To better understand Theorem \ref{coro:generalization}, we plot the changes of $|e_{1,\theta_0}|$, $|e_{1,\Sigma}|$, and $|e_{2,\theta_0}|$ w.r.t. $\delta$ by assuming $\Sigma=\vv I_p$ in Figure \ref{fig:diff}. In the left plot, $|e_{1,\theta_0}|$ firstly increases in $\delta$ linearly until $\delta=\delta_1$, then jumps to the second regime and grows until it converges to $2|(\widehat{\theta}_0-\theta_0)^{\top}\Sigma\theta_0|$ after $\delta>\delta_2$. In the middle plot, $|e_{1,\Sigma}|$ is almost zero when $\delta<\delta_1$, then increases when $\delta\in(\delta_1,\delta_2)$ and finally converges when $\delta>\delta_2$. And, $|e_{2,\Sigma}|$ shares a similar pattern. The pattern of $|e_{2,\theta_0}|$ is similar as $|e_{1,\theta_0}|$ except that it smoothly transits into the second regime, as shown in the right plot. The empirical and theoretical curves match very well in Figure \ref{fig:diff}.
	
	\begin{figure*}[!ht]
		\centering
		\includegraphics[scale=0.43]{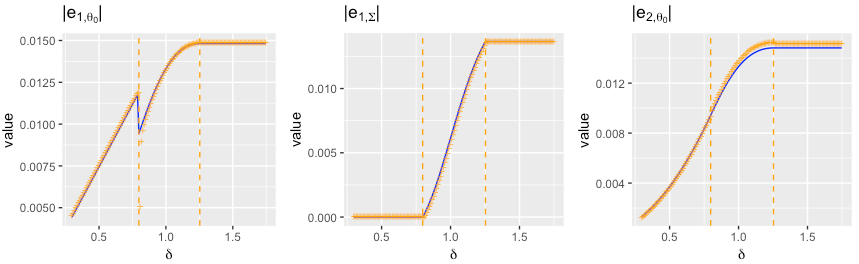}\vspace{-0.1in}
		\caption{The value of $|e_{1,\theta_0}|$, $|e_{1,\Sigma}|$, and $|e_{2,\theta_0}|$ as functions of $\delta$. Assume $\|\theta_0\|=1$, $\Sigma=\vv I_p$. Blue curve is obtained from Theorem \ref{coro:generalization} given $(\widehat{\theta}_0,\widehat{\Sigma})$. Orange points are obtained from simulation. $n=1000$, $\sigma^2=1$. The two vertical dashed lines in each figure represent $\delta_1$ and $\delta_2$.}
		\label{fig:diff}
	\end{figure*}

\subsection{Reducing estimation error through additional unlabeled data}

Unlabeled data is commonly used in semi-supervised learning, e.g. locally-weighted nearest neighbors algorithm \cite{cannings2020local}. Besides, in the context of adversarially robust learning, some studies also observed the benefits of using extra unlabeled data \cite{raghunathan2019adversarial}.

We study the effect of extra unlabeled data on the minimax lower bounds and the upper bounds of our proposed method under different scenarios. With the existence of extra unlabeled data, the minimiax lower bounds become smaller. In addition, these data also help reduce the upper bounds through improving the accuracy of $\widehat\Sigma$.

If we move one more step from Theorem \ref{thm:minimax:dense} and \ref{thm:minimax:sparse} and allow the usage of extra unlabeled data, the lower bounds can be reduced:
\begin{theorem}\label{thm:minimax:extra}
    Under the conditions in Theorem \ref{thm:minimax:dense}, if there are extra $n_1$ samples of unlabeled data, the lower bound becomes $\Omega( (p\sigma^2/n) \vee (pR^2/(n+n_1)) )$.
    
    Under the conditions in Theorem \ref{thm:minimax:sparse}, if there are extra $n_1$ samples of unlabeled data, the lower bound becomes $$\Omega\left( s\sigma^2\frac{\log (p/s)}{n}\vee R^2(n+n_1)^{-\frac{2\alpha}{2\alpha+1}} \right).$$
\end{theorem}

In terms of the upper bounds, since the estimation of $\widehat\Sigma$ is only related to $\x$, one can directly utilize these extra unlabeled data into the two-stage framework. The following result is extended from Theorem \ref{thm:bah}:
\begin{corollary}
    Under the conditions in Corollary \ref{prop:dense}, if there are extra $n_1$ samples of unlabeled data, the upper bound becomes $O( (p\sigma^2/n) \vee (pR^2/(n+n_1)) )$.
    
    Under the conditions in Corollary \ref{prop:sparse}, if there are extra $n_1$ samples of unlabeled data, the bound becomes $$O\left( s\sigma^2\frac{\log (p/s)}{n}\vee R^2(n+n_1)^{-\frac{2\alpha}{2\alpha+1}} \right).$$
\end{corollary}

To summarize, as both lower bounds and upper bounds are reduced, it is essential and effective to utilize extra unlabeled data for adversarially robust learning. A numerical illustration is also given in the next section of experiments.
	
\section{Numerical Experiments}\label{sec:numerical}
In numerical experiments, we consider Example \ref{example:1}, and adopt LASSO/sparse  estimators in the first stage to improve adversarial robustness. 
	
We consider the following specifications of $(\theta_0,\Sigma)$: $\theta_0$ is randomly generated from $\partial B(0,1)$, the sphere of a $\mathcal{L}_2$ ball; the diagonal elements in $\Sigma$ are $\Sigma_{ii}=2r+|\tau_i|$, where $\tau_i$'s follow i.i.d. standard Gaussian, and the other elements in $\Sigma$ are $r$. Under this design of $\Sigma$, coordinates of $\x$ are correlated with each other, and the smallest and largest eigenvalues are within a reasonable range as $p$ increases. Each experiment was repeated 500 times with $\sigma^2=1$. Define $\widehat{\Sigma}=\X^{\top}\X/n$ for non-sparse $\Sigma$. 
	
	\paragraph{Empirical coverage when $p$ is fixed.}
	As mentioned in Example \ref{example:1}, $\sqrt{n}(\widehat{\theta}-\theta^*)$ asymptotically converges to a zero-mean Gaussian when $\delta<\delta_2$. We use empirical coverage to verify this statement. In this experiment, $\theta_0=(1,2)^{\top}$ and $\Sigma_{ii}=i$ for $i=1,2$ with $\Sigma_{12}=0.5$. For each $\delta$, we repeat the experiment of estimating $\theta^*$ for 1000 times using 1000 samples, and calculate the 95\% empirical coverage for $\theta^*_1$ and $\theta^*_2$. In Figure \ref{fig:inference}, when $\delta<1.9$, the magnitude of $\theta^*_i$'s are away from zero, and the empirical coverage for both $\theta^*_i$'s are close to 0.95. When $\delta>1.9$, $\theta^*_i$'s are almost zero, and the corresponding empirical coverages are a little bit away from 0.95.
	\begin{figure}[!ht]
		\centering\vspace{-0.15in}
		\includegraphics[scale=0.75]{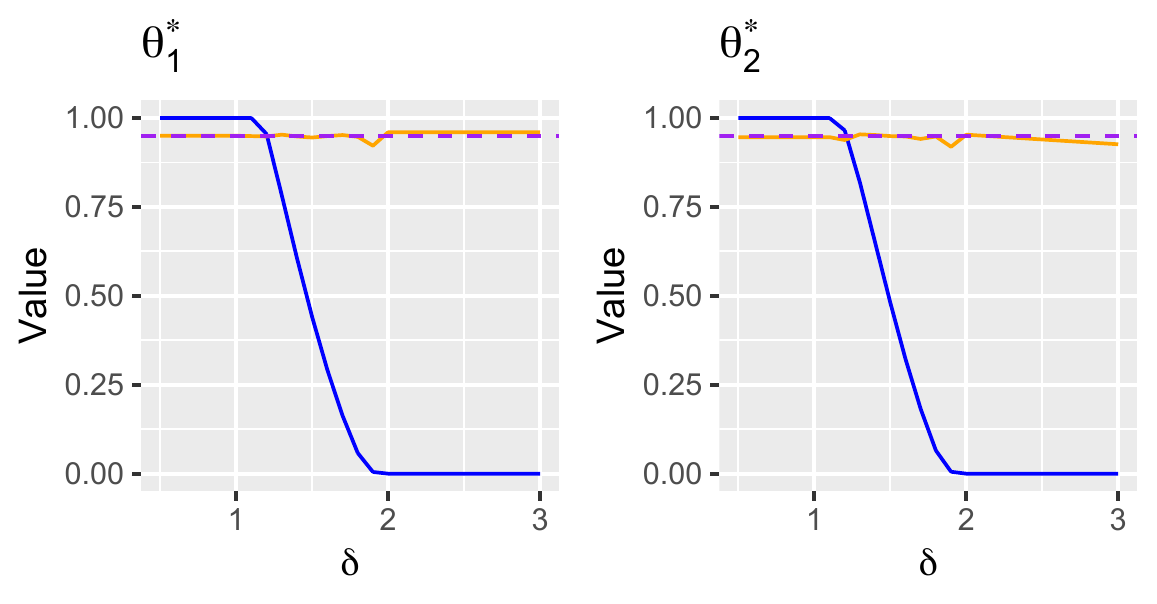}\vspace{-0.15in}
		\caption{Value of $\theta^*_i$ and the 95\% Empirical Coverage. Blue line: $\theta^*_i/\theta_{0i}$. Orange line: 95\% Empirical Coverage. Purple dased line: 0.95. The 95\% coverage for both $\theta^*_1$ and $\theta^*_2$ are close to 0.95 when $\delta<1.9$. }
		\label{fig:inference}
	\end{figure}
	\paragraph{Sparse coefficients.}
	In this experiment, we verify that LASSO helps obtaining better adversarially robust estimate. We take $p=50$, $n=300$, and assume $\Sigma$ is known. Cross validation was applied to choose the penalty that minimizes the (standard) prediction risk. Such a task was implemented by function \texttt{cv.glmnet} in library \texttt{glmnet} in R.  
	
We consider both lower-dimensional dense (Table \ref{tab:lasso_dense}) case with $(p,n)=(50,300)$ and high-dimensional sparse scenario with $(p,n)=(300,200)$. For high-dimensional sparse model, to make it clear on the difference between $\widehat\theta_{OLS}$ and $\widehat\theta_{LASSO}$, we present the results given $\Sigma$ is known/unknown. In the dense coefficient model, although we can select a $\lambda$ such that the LASSO estimator leads to a smaller standard risk than OLS estimator, its corresponding adversarial risk gets worse with an increasing $\delta$. For the sparse model, for all choices of $\delta$, LASSO has a smaller adversarial risk than OLS. The results for unknown $\Sigma$ are similar with the case when $\Sigma$ is known in the sense that the performance of LASSO is also better than OLS. 

In addition, $R_0(\widehat{\theta}_{LASSO},\delta)$ is always smaller when $\Sigma$ is known than when $\Sigma$ is unknown. This also verifies that unlabeled data helps improve the adversarial robustness (the comparison is not applicable to $R_0(\widehat{\theta}_{OLS},\delta)$ since $\widehat\theta_{OLS}$ is not consistent).

	\begin{table*}[!ht]
		\centering
		\caption{Comparison between OLS and LASSO for dense $\theta_0$ with known $\Sigma$.  $p=50$, $n=300$, $r=0.1$, $\sigma^2=1$. $\Sigma$ is known. Standard deviation is provided for $R_0(\theta^*,\delta)-R_0(\widehat{\theta}_{OLS},\delta)$ and  $R_0(\theta^*,\delta)-R_0(\widehat{\theta}_{LASSO},\delta)$. }\label{tab:lasso_dense}
		\begin{tabular}{llllll}
			\hline  $\delta$ & $R_0(\theta^*,\delta)$ & $R_0(\widehat{\theta}_{OLS},\delta)$ & $R_0(\widehat{\theta}_{LASSO},\delta)$ & $R_0(\theta_0,\delta)$ & $R_0(0,\delta)$ \\\hline
			0.5 & 0.2489 & 0.8545(0.1413)& \textbf{0.633}(0.0795) &0.25 & 0.9997 \\
			0.8 & 0.5847 & \textbf{0.8436}(0.0867) & 0.8516(0.0858) & 0.64& 0.9997 \\
			0.9 & 0.6862 & \textbf{0.8715}(0.65) & 0.8888(0.0762) & 0.81 & 0.9997\\\hline
		\end{tabular}
	\end{table*}

	\begin{table*}[!ht]
	\centering
	\caption{Comparison between OLS and LASSO for sparse $\theta_0$ . The first 10 elements of $\theta_0$ are $1/\sqrt{10}$.  $p=300$, $n=200$, $r=0.1$, $\sigma^2=1$.   }\label{tab:lasso_sparse_large}
	\begin{tabular}{lllllll}
		\hline $\Sigma$  &$\delta$ & $R_0(\theta^*,\delta)$ & $R_0(\widehat{\theta}_{OLS},\delta)$ & $R_0(\widehat{\theta}_{LASSO},\delta)$ & $R_0(\theta_0,\delta)$ & $R_0(0,\delta)$ \\\hline
	\multirow{4}{*}{known} &	0.5 & 0.25   & 6.1134(1.0171) & \textbf{0.7486}(0.1200) & 0.25 & 1.8943 \\
	&	1   & 0.7847 & 2.7114(0.4124) & \textbf{0.9941}(0.0752) & 1 & 1.8943\\
	&	2   & 1.3088 & 1.4912(0.0431) & \textbf{1.3684}(0.0453) & 4 & 1.8943 \\
	&	3   & 1.6088 & 1.7522(0.0641) & \textbf{1.6435}(0.1033) & 9 & 1.8943\\\hline
	\multirow{4}{*}{unknown} &0.5 & 0.25   & 2.3533(0.2551) & \textbf{0.8212}(0.0984) & 0.25 & 1.8943 \\
	&1   & 0.7847 & 1.5830(0.1368) & \textbf{1.1414}(0.0732) & 1 & 1.8943\\
	&2   & 1.3088 & 1.5023(0.0341) & \textbf{1.4716}(0.0358) & 4 & 1.8943 \\
	&3   & 1.6088 & 1.7040(0.0250) & \textbf{1.6930}(0.0494) & 9 & 1.8943\\\hline
	\end{tabular}
\end{table*}

		\begin{table*}[!ht]
		\centering
		\caption{Comparison between $\widehat\Sigma$ and $\widehat\Sigma_{sparse}$. $p=300$, $n=200$, $\sigma^2=1$. $\theta_0$ is known. $\widehat\Sigma_{sparse}$ performs slightly better when $\Sigma$ is sparse.}\label{tab:matrix_large}
		\begin{tabular}{lllll}
			\hline $\delta=2$ & $R_0(\theta^*,\delta)$ & $R_0(\theta^*_{ \widehat\Sigma},\delta)$ & $R_0(\theta^*_{ \widehat{\Sigma}_{sparse}},\delta)$ & $R_0(\theta_0,\delta)$  \\\hline
			Dense  &  1.8865 & \textbf{2.0576}(0.1841) & 4.8769(0.1044)  & 4.0000 \\
			Sparse &  2.9807 & 3.0652(0.0279) & \textbf{3.0293}(0.0279) & 4.0000\\\hline
		\end{tabular}
	\end{table*}
\paragraph{Sparse  matrix.} We use sparse matrix estimator to verify that it helps enhancing adversarial robustness. To generate sparse  matrix, we consider $\Sigma$ such that $\Sigma_{ii}=1$, and $\Sigma_{ij}=r|i-j|^{-\alpha-1}$ when $j\neq i$, where $r=0.6$ and $\alpha=0.2$. This choice of $(r,\alpha)$ ensures that all eigenvalues of $\Sigma$ are positive. We take $p=300$, $n=200$ so that the difference between $\widehat{\Sigma}$ and $\widehat{\Sigma}_{sparse}$ is obvious. The attack level $\delta$ is set to be 2 in this comparison. For simplicity, we assume $\theta_0$ is known in the comparison of  matrix estimators. The sparse covariance estimator $\widehat{\Sigma}_{sparse}$ was obtained based on the method in \cite{cai2016optimal}. 
In Table \ref{tab:matrix_large}, the adversarial excess risk is reduced from 0.0845  ($R_0(\theta^*_{\widehat{\Sigma}},\delta)-R_0(\theta^*,\delta)$) to 0.0486 ($R_0(\theta^*_{\widehat{\Sigma}_{sparse}},\delta)-R_0(\theta^*,\delta)$), which shows the effectiveness of $\widehat{\Sigma}_{sparse}$.  
In addition to sparse matrix, we also consider dense covariance matrix generated in the same way as previous experiments by taking $r=0.6$. When the true  martix is dense, using sparse estimate is not appropriate, thus the corresponding adversarial risk is much higher.

\section{Conclusions and Future Directions}\label{sec:conclusion}
In this paper, we figure out the minimax lower bound of estimation error of adversarially robust model in linear regression setup, which indicates the it is important to incorporate model information in adversarially robust learning. In addition, we propose a two-stage adversarially robust learning method based on an explicit relation between adversarially robust estimator and standard estimator. The proposed two-stage estimator can encode model information (e.g. sparsity) into standard estimators, through which the robustness of adversarially robust estimator could be improved and reach minimax optimal convergence rate. Our investigation in the generalization error also verifies that adversarial robustness hurts generalization. 
	
One future direction is to relax the distributional assumption on $(\x,y)$, say $\x$ follows non-Gaussian distribution. Although there is a wide range of data may follow Gaussian assumption, e.g. abalone data and other biological data, many other data may not follow Gaussian, e.g. image data. The constant $c_0$ in our framework currently depends on the Gaussian assumption, and there is potential to relax it. Another one is concerned with sparse adversarially robust learning, say sparse $\widehat\theta$, which could be useful in both compressing and robustifying deep neural networks \cite{guo2018sparse}. The first step is to understand how the sparsity of $\theta_0$ (together with other model assumptions) implies that of $\theta^*$, which in turn determines the sparsity of $\widehat\theta$. An example can be found in \cite{allen2020feature} for linear sparse coding model. However, more careful studies would be needed in the future.

	\bibliographystyle{plain}
	\bibliography{main}
	
	\newpage
	\onecolumn
	\appendix
	\section{More theoretical results}\label{sec:appendix:more}
	This section provides some supplementary theorems.
	
	\subsection{Results regarding to prediction risk (an analog of Theorem \ref{thm:lsm})}
For consistent estimates $(\widehat\theta_0,\widehat\Sigma,\widehat\sigma^2)$, with probability tending to 1, we have\begin{eqnarray*}
	\sup\limits_{\delta\geq 0}\left| R(\theta^*(\delta),\delta)-R(\widetilde\theta(\delta),\delta) \right|&=&O\left(\|\widehat\theta_0-\theta_0\|\|\theta_0\|\right)+O\left(\|\theta_0\|^2\sqrt{\|\widehat\Sigma-\Sigma\|}\right)\\&&+O\left(\widehat\sigma^2-\sigma^2\right)+O\left(\|\theta_0\|(\widehat\sigma-\sigma)\right).
\end{eqnarray*}

	\subsection{Definitions in Theorem \ref{thm:bah}}
	\begin{eqnarray*}
		\vv M_1(\theta^*,\theta_0,\Sigma)&=&\vv H(\theta^*,\theta_0,\Sigma)^{-1}\vv M(\theta^*,\theta_0,\Sigma),\\
		\vv M_2(\theta^*,\theta_0,\Sigma)&=&-\delta c_0\vv H(\theta^*,\theta_0,\Sigma)^{-1}\left( \frac{\Sigma(\theta^*-\theta_0)}{2\|\theta^*-\theta_0\|_{\Sigma}\|\theta^*\|}-\frac{\|\theta^*\|\theta^*}{2\|\theta^*-\theta_0\|_{\Sigma}^3} \right),
		\end{eqnarray*}
		\begin{eqnarray*}
		\vv M_3(\theta^*,\theta_0,\Sigma)&=&-(1+\delta c_0 A(\theta^*,\theta_0,\Sigma)),\\
		\vv M(\theta^*,\theta_0,\Sigma)&=&\Sigma+\delta c_0 A(\theta^*,\theta_0,\Sigma)\Sigma +\frac{\delta c_0 A(\theta^*,\theta_0,\Sigma)}{\|\theta^*-\theta_0\|^2_{\Sigma}}\Sigma(\theta^*-\theta_0)(\theta^*-\theta_0)^{\top}\Sigma \\
		&&+\delta \frac{c_0}{A(\theta^*,\theta_0,\Sigma)\|\theta^*\|^2_2}\theta^*(\theta^*)^{\top},
	\end{eqnarray*}
	where $A(\theta^*,\theta_0,\Sigma)=\|\theta^*\|/\|\theta^*-\theta_0\|_{\Sigma}$. The matrix $\vv H(\theta^*,\theta_0,\Sigma)$ is the Hessian matrix of $R_0$.

	\section{Proofs in Section \ref{sec:population}}\label{sec:appendix:population}
	There are two parts of Proposition \ref{thm:opt}: the statement about Hessian, and the optimal solution $\theta^*$. We prove them separately.
	\subsection{Positive definite Hessian}
	\begin{proof}[Proof of Proposition \ref{thm:opt} for Positive Definite Hessian]
		Expand the adversarial risk at $\x$ as
		\begin{eqnarray*}
			\underset{\|x^*-x\|\le \delta}{\max}\left[(\vv(x^*)^{\top}\theta-\vv x^{\top}\theta_0)^2\right]&=&\underset{\|x^*-x\|\le \delta}{\max}\left[(\vv x^*-\vv x)^{\top}\theta+\vv x^{\top}(\theta-\theta_0)\right]^2\\
			&=&  \left(\delta\|\theta\|+ |\vv x^{\top}(\theta-\theta_0)|\right)^2.
		\end{eqnarray*}
		Since $\x$ follows Gaussian, for any fixed $\theta-\theta_0$, $\x^{\top}(\theta-\theta_0)$ follows Gaussian as well. Let $Z=\vv x^{\top}(\theta-\theta_0).$ Note that $\x\sim N(\vv 0,\Sigma),$ we have $Z\sim N(0,\|\theta-\theta_0\|^2_{\Sigma})$ and
		\begin{eqnarray*}
			R_0(\theta,\delta)&=&\E_{Z}(|Z|+\delta\|\theta\|)^2\\
			&=&\|\theta-\theta_0\|^2_{\Sigma}+2\delta  c_0 \|\theta\|\|\theta-\theta_0\|_{\Sigma}+\delta^2\|\theta\|^2_2.
		\end{eqnarray*}
		Taking the gradient of $R_0(\theta,\delta)$ with respect to $\theta$ yields
		\begin{eqnarray*}
			\nabla_{\theta} R_0(\theta,\delta)&=&2 \left[\Sigma (\theta-\theta_0)+\delta c_0 \frac{\|\theta-\theta_0\|_{\Sigma}}{\|\theta\|}\theta+\delta c_0 \frac{\|\theta\|}{\|\theta-\theta_0\|_{\Sigma}}\Sigma (\theta-\theta_0)+\delta^2 \theta\right]\\
			&=& 2 \left[\left(\vv  I_p+\delta c_0 \frac{\|\theta\|}{\|\theta-\theta_0\|_{\Sigma}}\right)\Sigma (\theta-\theta_0)+\left(\delta c_0 \frac{\|\theta-\theta_0\|_{\Sigma}}{\|\theta\|}+\delta^2\right)\theta\right].
		\end{eqnarray*}
		
		Denote $A=\|\theta\|/\|\theta-\theta_0\|_{\Sigma}$, then $\frac{\partial^2 R_0}{2\partial \theta^2}$ becomes
		\begin{eqnarray}\label{eq:hessian}
		\vv H:=\frac{\partial^2 R_0}{2\partial \theta^2}=
		\Sigma + \delta c_0 A\Sigma +\left( \frac{\delta c_0}{A}+\delta^2 \right)\vv  I_p+\delta c_0\Sigma(\theta-\theta_0)\left(\frac{\partial A}{\partial \theta}\right)^{\top} +\delta c_0\theta\left(\frac{\partial 1/A}{\partial \theta}\right)^{\top}.
		\end{eqnarray}
		To show that $R_0$ is convex, it sufficies to show that $H$ is positive definite.
		
		In $\vv H$, the two partial derivatives are
		\begin{eqnarray*}
			\frac{\partial A}{\partial \theta}&=&\frac{\partial}{\partial \theta}\frac{\|\theta\|}{\|\theta-\theta_0\|_{\Sigma}}=\frac{\theta}{\|\theta\|\|\theta-\theta_0\|_{\Sigma}} -\frac{\|\theta\| \Sigma(\theta-\theta_0)}{\|\theta-\theta_0\|_{\Sigma}^3 },
		\end{eqnarray*}
		and
		\begin{eqnarray*}
			\frac{\partial 1/A}{\partial \theta}=\frac{\Sigma(\theta-\theta_0)}{\|\theta\|\|\theta-\theta_0\|_{\Sigma}} - \frac{\|\theta-\theta_0\|_{\Sigma}\theta}{\|\theta\|^3}.
		\end{eqnarray*}
		Thus
		\begin{eqnarray*}
			\Sigma(\theta-\theta_0)\left(\frac{\partial A}{\partial \theta} \right)^{\top}&=& \frac{1}{\|\theta\|\|\theta-\theta_0\|_{\Sigma}}\Sigma(\theta-\theta_0)\theta^{\top}-\frac{\|\theta\|}{  \|\theta-\theta_0\|_{\Sigma}^3 }\Sigma(\theta-\theta_0)(\theta-\theta_0)^{\top}\Sigma,\\
			&=&  \frac{1}{\|\theta\|\|\theta-\theta_0\|_{\Sigma}}\Sigma(\theta-\theta_0)\theta^{\top}-\frac{A}{  \|\theta-\theta_0\|_{\Sigma}^2 }\Sigma(\theta-\theta_0)(\theta-\theta_0)^{\top}\Sigma,\\
			\theta \left(\frac{\partial 1/A}{\partial \theta} \right)^{\top}&=&\frac{ 1 }{\|\theta\|\|\theta-\theta_0\|_{\Sigma}} \theta(\theta-\theta_0)^{\top}\Sigma-\frac{\|\theta-\theta_0\|_{\Sigma}}{\|\theta\|^3}\theta\theta^{\top},\\
			&=&\frac{ 1 }{\|\theta\|\|\theta-\theta_0\|_{\Sigma}} \theta(\theta-\theta_0)^{\top}\Sigma-\frac{1}{A\|\theta\|^2}\theta\theta^{\top}.
		\end{eqnarray*}
		Then $\vv H$ can be represented  as
		\begin{eqnarray*}
			&&\Sigma + \delta c_0 A\Sigma +\left( \frac{\delta c_0}{A}+\delta^2 \right)\vv  I_p+\delta c_0\Sigma(\theta-\theta_0)\left(\frac{\partial A}{\partial \theta}\right)^{\top} +\delta c_0\theta\left(\frac{\partial 1/A}{\partial \theta}\right)^{\top}\\
			&=&\left(\Sigma -\frac{1}{\|\theta-\theta_0\|^2_{\Sigma}}\Sigma(\theta-\theta_0)(\theta-\theta_0)^{\top}\Sigma\right)A c_0 \delta\\
			&&+\left(\vv I_p-\frac{1}{\|\theta\|^2}\theta\theta^{\top}\right)\frac{c_0\delta}{A}\\
			&&+\left(\Sigma+\delta^2\vv  I_p+\frac{ \delta c_0 }{\|\theta\|\|\theta-\theta_0\|_{\Sigma}} \theta(\theta-\theta_0)^{\top}\Sigma+\frac{\delta c_0}{\|\theta\|\|\theta-\theta_0\|_{\Sigma}}\Sigma(\theta-\theta_0)\theta^{\top}\right)\\
			&:=&\vv M_1 A c_0\delta +\vv M_2 \frac{c_0\delta}{A}+\vv M_3.
		\end{eqnarray*}
		Since $\Sigma$ is positive definite, for any vector $\vv a\neq \vv 0,$ and $\theta\neq \theta_0, \theta\neq \vv 0,$  by Cauchy inequality,
		\begin{eqnarray*}
			\vv a^{\top} \vv M_1\vv a&=&	\vv a^{\top}\Sigma \vv a-\frac{ \left( \vv a^{\top}\Sigma(\theta-\theta_0) \right)^2. }{\|\theta-\theta_0\|_{\Sigma}^2}\ge 0,\\
			\vv a^{\top} \vv M_2\vv a&=&  \vv a^{\top}\vv a-\frac{(\vv a^{\top}\theta)^2}{\|\theta\|^2}\ge0,
		\end{eqnarray*}
		which imply $\vv M_1$ and $\vv M_2$ are positive semi-definite.
		
		For $\vv M_3$, we have
		\begin{eqnarray*}
			\vv M_3&=&\left[ \delta \vv I_p+\frac{ c_0}{\|\theta\|\|\theta-\theta_0\|_{\Sigma}}\Sigma(\theta-\theta_0)\theta^{\top} \right]\left[ \delta \vv I_p+\frac{ c_0}{\|\theta\|\|\theta-\theta_0\|_{\Sigma}}\Sigma(\theta-\theta_0)\theta^{\top} \right]^{\top}\\
			&&+\Sigma - \frac{ c_0^2 }{\|\theta-\theta_0\|_{\Sigma}^2}\Sigma(\theta-\theta_0)(\theta-\theta_0)^{\top}\Sigma.
		\end{eqnarray*}
		Since $c_0=\sqrt{2/\pi}<1, $ for any vector $\vv a\neq \vv 0,$ and $\theta\neq \theta_0, \theta\neq \vv 0,$
		\begin{eqnarray*}
			\vv a^{\top}\left(\Sigma - \frac{ c_0^2 }{\|\theta-\theta_0\|_{\Sigma}^2}\Sigma(\theta-\theta_0)(\theta-\theta_0)^{\top}\Sigma\right)\vv a
			&>&\vv a^{\top} \vv M_1\vv a\ge 0.
		\end{eqnarray*}
	\end{proof}
	\subsection{Optimal solution}
	\begin{proof}[Proof of Proposition \ref{thm:opt} for $\theta^*$]
		We first consider the case where $\Sigma$ is a diagonal matrix. Recall that the gradient of $R_0(\theta,\delta)$ is
		\begin{eqnarray*}
			\nabla_{\theta} R_0(\theta,\delta)&=&2 \left[\Sigma (\theta-\theta_0)+\delta c_0 \frac{\|\theta-\theta_0\|_{\Sigma}}{\|\theta\|}\theta+\delta c_0 \frac{\|\theta\|}{\|\theta-\theta_0\|_{\Sigma}}\Sigma (\theta-\theta_0)+\delta^2 \theta\right]\\
			&=& 2 \left[\left(\vv I_p+\delta c_0 \frac{\|\theta\|}{\|\theta-\theta_0\|_{\Sigma}}\right)\Sigma (\theta-\theta_0)+\left(\delta c_0 \frac{\|\theta-\theta_0\|_{\Sigma}}{\|\theta\|}+\delta^2\right)\theta\right].
		\end{eqnarray*}
		Note the gradient $\nabla_{\theta} R_0(\theta,\delta)$ is well-defined in $\mathbb{R}^p/\{\vv 0, \theta_0\}$ and $R_0(\theta,\delta)\to +\infty$ as $\|\theta\|\to +\infty.$ Thus, the global minimizer $\theta^*$ of $R_0(\theta,\delta)$ should only be $\vv 0, \theta_0$ or the stationary point of $R_0(\theta,\delta).$
		Note if $\nabla_{\theta} R_0(\theta,\delta)=\bf 0,$ we have  $\eta\Sigma (\theta-\theta_0)=- \theta$ for some $\eta> 0,$ or equivalently,
		\begin{eqnarray*}
		\theta_{\eta}=(\eta\Sigma+\vv I_p)^{-1}\eta\Sigma\theta_0=\left[ \vv I_p-(\eta \Sigma+\vv I_p)^{-1}\right]\theta_0.
		\end{eqnarray*}
		Since when $\eta\to 0,$ $\theta_{\eta}\to \vv 0$ and when $\eta\to +\infty,$ $\theta_{\eta}\to \theta_0,$ the global minimizer of $R_0(\theta,\delta)$ should has the form as $\left[ I-(\eta \Sigma+\vv I_p)^{-1}\right]\theta_0$ for some $\eta\in[0,\infty].$
		Define
		\begin{eqnarray}
		r(\eta)&=&R_0(\theta_{\eta},\delta),\label{eqn:r}\\
		\theta_{\eta}&=&(\vv I_p-(\eta\Sigma+\vv I_p)^{-1})\theta_0,\label{eqn:eta}\\
		H(\eta)&=&\frac{\sqrt{\theta^{\top}_0( \frac{\Sigma}{\eta \Sigma+\vv I_p} )^2\theta_0}}{\sqrt{\theta^{\top}_0\frac{\Sigma}{(\eta \Sigma+\vv I_p)^2} \theta_0}},\label{eqn:H}\\
		g(\eta)&=&1-\frac{\delta c_0}{H(\eta)}+\eta(\delta c_0 H(\eta)-\delta^2).\label{eqn:g}
		\end{eqnarray}
		We have
		\begin{eqnarray}\label{eqn:partial_r}
		r'(\eta)=\frac{\partial}{\partial \eta}R_0(\theta_{\eta},\delta)&=&(\nabla_{\theta} R_0(\theta_{\eta}))^{\top} \frac{\partial}{\partial \eta}\theta_{\eta}\nonumber\\
		&=&-2 g(\eta) (\theta-\theta_0)^{\top}\Sigma(\eta\Sigma+\vv I_p)^{-1}\Sigma(\theta-\theta_0).
		\end{eqnarray}
		By Lemma \ref{lem:opt} below, if $\delta\le\delta_1,$ $g(\eta)>0.$ Thus, $r(\eta)$ is decreasing and the global minimizer of $R_0(\theta,\delta)$ is $\theta_{\eta=+\infty}=\theta_0$. If $\delta\ge\delta_2,$ $g(\eta)<0.$ Thus,  $r(\eta)$ is increasing and the global minimizer of $R_0(\theta,\delta)$ is $\theta_{\eta=0}=\vv 0$. If $\delta_1<\delta<\delta_2,$ there exists a unique positive number $\eta^*$ (as denoted as $\eta^*(\delta)$)  such that $g(\eta^*)=0$. Moreover, note
		\begin{eqnarray*}
			g(\eta)=\left(1+\delta c_0 \frac{\|\theta_{\eta}\|}{\|\theta_{\eta}-\theta_0\|_{\Sigma}}\right)-\eta\left(\delta c_0 \frac{\|\theta_{\eta}-\theta_0\|_{\Sigma}}{\|\theta_{\eta}\|}+\delta^2\right),
		\end{eqnarray*}
		thus, $\eta^*$ is the unique solution to (\ref{eqn:cond}). Finally, $g(\eta)>0$ when $\eta\in[0,\eta^*),$ and $g(\eta)<0$ when $\eta\in(\eta^*,\infty).$ Thus, by (\ref{eqn:partial_r}),  $r(\eta)$ is decreasing when $\eta\in[0,\eta^*),$ is increasing when $\eta\in (\eta^*,\infty).$ Thus, $R_0(\theta,\delta)$  gets the global minimum when $\theta=\theta_{\eta=\eta^*}.$

		For the general positive definite  matrix $\Sigma,$
		we consider the orthogonal decomposition of $\Sigma$ and let $\Sigma=\vv U^{\top} \vv D \vv U$ where $\vv D$ is a $p\times p$ diagonal matrix and $\vv U$ is an orthogonal matrix. Let $\theta=\vv U\theta, \theta_0=\vv U\theta_0.$ Then the adversarial prediction risk $R_0(\theta,\delta)$ in (\ref{eqn:risk}) becomes
		\begin{eqnarray}\label{eq:11}
		R_0(\theta,\delta)=R_0(U^{\top}\theta,\delta)=\|\theta-\theta_0\|^2_{\vv D} +2\delta  c_0 \|\theta\|\|\theta-\theta_0\|_{\vv D}+\delta^2\|\theta\|^2_2.
		\end{eqnarray}
		Note $\vv D$ is a diagonal matrix.  Applying the results from Proposition \ref{thm:opt} yields $\theta^*=(\vv I_p-(\eta^*\vv D+\vv I_p)^{-1})\theta_0.$  Therefore, since $\theta=\vv U\theta, \theta_0=\vv U\theta_0,$ $\theta^*=(\vv I_p-(\eta^*\Sigma+\vv I_p)^{-1})\theta_0,$
		which completes the proof.
	\end{proof}
	\begin{lemma}\label{lem:opt}
		Suppose $\Sigma$ is a $p$ by $p$ diagonal  matrix. Define functions $H(\eta)$ and $g(\eta)$ as in (\ref{eqn:H}) and (\ref{eqn:g}), then
		\begin{enumerate}
			\item If $\delta\ge \delta_2,$  $g(\eta)<0$ for all $\eta> 0.$
			\item If $\delta\le \delta_1,$  $g(\eta)>0$ for all $\eta\ge 0.$
			\item If $\delta_1<\delta< \delta_2,$ there exists a unique positive number $\eta^*$ such that $g(\eta)=0.$ Moreover, $g(\eta)>0$ when $\eta\in[0,\eta^*),$ and $g(\eta)<0$ when $\eta\in(\eta^*,\infty).$
		\end{enumerate}
		Here $\delta_1=\frac{ c_0 \|\theta_0\|}{\|\theta_0\|_{\Sigma^{-1}}}$ and $\delta_2=\frac{\|\theta_0\|_{\Sigma^2}}{ c_0  \|\theta_0\|_{\Sigma}}.$
	\end{lemma}
	\begin{proof}[Proof of Lemma \ref{lem:opt}]
		By Lemma \ref{lemma2} below, we have for any $\eta\ge 0,$ $\delta_1/c_0\le H(\eta)\le \delta_2 c_0.$
		Therefore, $g(\eta)>0$ if  $\delta\le\delta_1$ and $g(\eta)<0$ if  $\delta\ge\delta_2.$

		Moreover, note $H(\eta=0)=\delta_2 c_0, H(\eta=\infty)=\delta_1/c_0.$  When $\delta_1<\delta<\delta_2,$ $g(\eta=0)=1-\delta/\delta_2>0$ and $g(\eta=\infty)=-\infty.$ There must exist a positive solution to $g(\eta)=0.$ Next, we will show the solution is unique. Assume $\eta^*$ is the smallest $\eta$ such that $g(\eta)=0.$ Then we claim $g(\eta)$ is decreasing as $\eta\ge \eta^*.$ In fact, if $g(\eta^*)=0,$ we have $1-\frac{\delta c_0}{H(\eta^*)}\le 0$ and $\delta c_0 H(\eta^*)-\delta^2<0.$ By Lemma \ref{lemma4}, $H(\eta)$ is a decreasing function when $\eta\ge 0$. Thus, $g(\eta)$ is decreasing as $\eta\ge \eta^*$ and $g(\eta)<g(\eta^*)=0$ for $\eta>\eta^*.$ Therefore, there is one unique  $\eta^*$ such that $g(\eta^*)=0.$ Moreover, $g(\eta)>0$ when $\eta\in[0,\eta^*),$ and $g(\eta)<0$ when $\eta\in(\eta^*,\infty).$
	\end{proof}

	\begin{lemma}\label{lemma2}
		If $\Sigma=$diag$(d_1,d_2,\cdots, d_p)$ is a diagonal matrix, where all $d_i>0,$ then for any $\eta\ge 0,$
		\begin{eqnarray}\label{eq:6}
		\left(\theta^{\top}_0\left( \frac{1}{ \Sigma}\right)\theta_0\right)\left(\theta^{\top}_0\left( \frac{\Sigma}{\eta \Sigma+\vv I_p} \right)^2\theta_0\right)&\ge& \left(\theta^{\top}_0 \frac{\Sigma}{(\eta \Sigma+\vv I_p)^2} \theta_0\right)\left(\theta^{\top}_0\theta_0\right)\\
		\left(\theta^{\top}_0\Sigma^2\theta_0\right)\left(\theta^{\top}_0 \frac{\Sigma}{(\eta \Sigma+\vv I_p)^2} \theta_0\right)&\ge& \left(\theta^{\top}_0 \Sigma\theta_0\right)\left(\theta^{\top}_0\left( \frac{\Sigma}{\eta \Sigma+\vv I_p} \right)^2\theta_0\right)\label{eq:6.2}
		\end{eqnarray}
	\end{lemma}
	\begin{proof}[Proof of Lemma \ref{lemma2}]
		To prove Lemma \ref{lemma2}, we expand all terms in 
		\begin{eqnarray*}
			&&\left(\theta^{\top}_0\left( \frac{1}{ \Sigma}\right)\theta_0\right)\left(\theta^{\top}_0\left( \frac{\Sigma}{\eta \Sigma+\vv I_p} \right)^2\theta_0\right)=\left(\sum_{i=1}^p\frac{1}{d_i}(\theta^i_0)^2\right)\left(\sum_{i=1}^p(\frac{d_i}{\eta d_i+1})^2(\theta^i_0)^2\right)\\
			&=& -\sum_{i=1}^p \frac{d_i}{(\eta d_i+1)^2}(\theta^i_0)^4+ \underset{1\le i\le j\le p}{\sum\sum}\left[ \frac{1}{d_i}(\theta^i_0)^2 (\frac{d_j}{\eta d_j+1})^2(\theta^j_0)^2+ \frac{1}{d_j}(\theta^j_0)^2 (\frac{d_i}{\eta d_i+1})^2(\theta^i_0)^2\right]\nonumber\\
			&=& -\sum_{i=1}^p \frac{d_i}{(\eta d_i+1)^2}(\theta^i_0)^4+ \underset{1\le i\le j\le p}{\sum\sum}\left[ \left(\frac{1}{d_i} (\frac{d_j}{\eta d_j+1})^2+\frac{1}{d_j} (\frac{d_i}{\eta d_i+1})^2\right)(\theta^i_0\theta^j_0)^2\right]\nonumber
		\end{eqnarray*}
		
		\begin{eqnarray*}
			&&\left(\theta^{\top}_0 \frac{\Sigma}{(\eta \Sigma+\vv I_p)^2} \theta_0\right)\left(\theta^{\top}_0\theta_0\right)=\left(\sum_{i=1}^p\frac{d_i}{(\eta d_i+1)^2}(\theta^i_0)^2\right)\left(\sum_{i=1}^p(\theta^i_0)^2\right)\\
			&=& -\sum_{i=1}^p \frac{d_i}{(\eta d_i+1)^2}(\theta^i_0)^4+\underset{1\le i\le j\le p}{\sum\sum}\left[ \frac{d_i}{(\eta d_i+1)^2}(\theta^i_0)^2 (\theta^j_0)^2+ \frac{d_j}{(\eta d_j+1)^2}(\theta^j_0)^2 (\theta^i_0)^2\right]\nonumber\\
			&=& -\sum_{i=1}^p \frac{d_i}{(\eta d_i+1)^2}(\theta^i_0)^4+ \underset{1\le i\le j\le p}{\sum\sum}\left[ \left( \frac{d_i}{(\eta d_i+1)^2}+\frac{d_j}{(\eta d_j+1)^2}\right)(\theta^i_0\theta^j_0)^2\right]\nonumber
		\end{eqnarray*}
		
		By  rearrangement inequality, for any $i$ and $j,$ we have
		\begin{eqnarray*}
			\left(\frac{1}{d_i} (\frac{d_j}{\eta d_j+1})^2+\frac{1}{d_j} (\frac{d_i}{\eta d_i+1})^2\right)\ge \left( \frac{d_i}{(\eta d_i+1)^2}+\frac{d_j}{(\eta d_j+1)^2}\right),
		\end{eqnarray*}
		which yields the inequality in (\ref{eq:6}). Similarly we can show (\ref{eq:6.2}).
	\end{proof}

	\begin{lemma}\label{lemma4}
		If $\Sigma=diag(d_1,d_2,\cdots, d_p)$ is a diagonal matrix, where all $d_i>0,$ then for any $\eta_1>\eta_2,$
		\begin{eqnarray*}
			\left(\theta^{\top}_0( \frac{\Sigma}{\eta_1 \Sigma+\vv I_p} )^2\theta_0\right)\left(\theta^{\top}_0( \frac{\Sigma}{(\eta_2 \Sigma+\vv I_p)^2} )\theta_0\right)<\left(\theta^{\top}_0( \frac{\Sigma}{\eta_2 \Sigma+\vv I_p} )^2\theta_0\right)\left(\theta^{\top}_0( \frac{\Sigma}{(\eta_1 \Sigma+\vv I_p)^2}. )\theta_0\right).
		\end{eqnarray*}
	\end{lemma}
	\begin{proof}[Proof of Lemma \ref{lemma4}]
		Using the same techniques as in the proof of Lemma \ref{lemma2}, it suffices to show that for any $d_i\neq d_j,$
		\begin{eqnarray*}
			&&\left(\frac{d_i}{\eta_1 d_i+1}\right)^2 \frac{d_j}{(\eta_2d_j+1)^2}+\left(\frac{d_j}{\eta_1 d_j+1}\right)^2 \frac{d_i}{(\eta_2d_i+1)^2}\\
			&<&\left(\frac{d_i}{\eta_2 d_i+1}\right)^2 \frac{d_j}{(\eta_1d_j+1)^2}+\left(\frac{d_j}{\eta_2 d_j+1}\right)^2 \frac{d_i}{(\eta_1d_i+1)^2},
		\end{eqnarray*}
		which is equivalent to
		\begin{eqnarray}\label{eq:10}
		\frac{d_i-d_j}{(\eta_1 d_i+1)^2(\eta_2d_j+1)^2}<\frac{d_i-d_j}{(\eta_1 d_j+1)^2(\eta_2d_i+1)^2}.
		\end{eqnarray}
		The last inequality (\ref{eq:10}) always hold no matter $d_i>d_j$ or $d_i<d_j$ by the rearrangement inequality, which completes the proof.
	\end{proof}
	
	\section{Proofs in Section \ref{sec:lower_bound} and \ref{sec:low}}\label{sec:appendix:low}
		\subsection{Theorem \ref{thm:minimax:dense}}
\begin{lemma}\label{lem:minimax:dense:theta}
Assume $R>c_1\sigma$ for some constant $c_1$. Also Assume $\lambda_{\max}(\Sigma)$ and $\lambda_{\min}(\Sigma)$ are bounded and bounded away from zero. When $(p\log^2 n)/n\rightarrow 0$,
	\begin{eqnarray}
	\inf_{\widehat{\theta}} \sup_{\delta, \sigma<\|\theta_0\|\leq \sqrt{R^2+\sigma^2}, \Sigma}\mathbb{E}\|\widehat{\theta}-\theta^*(\delta)\|^2=\Omega\left(\frac{\sigma^2p}{n} \right).
	\end{eqnarray}
\end{lemma}
\begin{proof}[Proof of Lemma \ref{lem:minimax:dense:theta}]
We first consider a relaxation where $\|\theta_0\|$ is unbounded, then add back the condition on $\|\theta_0\|$ into the bound to show that these conditions does not change the rate of the lower bound.

Assume $\theta_0$ follows $N(0,\sigma^2/(\alpha n)\vv I_p)$ and $\alpha=o(1)$. Denote $\widehat{\Sigma}_n=\vv X^{\top}\vv X/n$, and $\widehat{\theta}_{n,\alpha}=(\widehat{\Sigma}_n+\alpha \vv I_p)^{-1}\vv X^{\top}\vv y/n$.  Given $(\vv X,\vv y)$, it follows that $\theta_0|(\vv X,\vv y)\sim N(\widehat{\theta}_{n,\alpha},(\sigma^2/n)(\widehat{\Sigma}_n+\alpha \vv I_p)^{-1})$,  and
\begin{eqnarray*}
	\inf_{\widehat{\theta}} \sup_{\delta,\theta_0,\Sigma}\mathbb{E}\|\widehat{\theta}-\theta^*(\delta)\|^2&\geq&
	\inf_{\widehat{\theta}} \sup_{\delta,\theta_0,\Sigma=\vv I_p}\mathbb{E}\|\widehat{\theta}-\theta^*(\delta)\|^2\\&\geq& \inf_{\widehat{\theta}}\sup_{\delta}\mathbb{E}\left[\mathbb{E}_{\theta_0|\vv X, \vv y,\Sigma=\vv I_p}\|\widehat{\theta}-\theta^*(\delta)\|^2\right].
\end{eqnarray*}

Observe that
\begin{eqnarray*}
	&&\inf_{\widehat{\theta}}\sup_{\delta}\mathbb{E}\left[\mathbb{E}_{\theta_0|\vv X, \vv y,\Sigma=\vv I_p}\|\widehat{\theta}-\theta^*(\delta)\|^2\right]\\
	&=&\inf_{\widehat{\theta}}\sup_{\delta}\mathbb{E}\left(\|\widehat{\theta}-\mathbb{E}_{\theta_0|\vv X, \vv y,\Sigma=\vv I_p}[\theta^*(\delta)]\|^2+\mathbb{E}_{\theta_0|\vv X, \vv y,\Sigma=\vv I_p}\|\mathbb{E}_{\theta_0|\vv X, \vv y,\Sigma=\vv I_p}[\theta^*(\delta)]-\theta^*(\delta)\|^2\right)\\
	&\geq& \sup_{\delta}\mathbb{E}\left[\mathbb{E}_{\theta_0|\vv X, \vv y,\Sigma=\vv I_p}\|\mathbb{E}_{\theta_0|\vv X, \vv y,\Sigma=\vv I_p}[\theta^*(\delta)]-\theta^*(\delta)\|^2\right].
\end{eqnarray*}
When $\Sigma=\vv I_p$, by Proposition \ref{thm:opt}, we know that $\theta^*(\delta)=(1-\kappa(\delta))\theta_0$ for some function $\kappa$ that only depends on $\delta$. In addition, based on equation (A.5) in Lemma A.6 of \cite{ing2011stepwise}, we have $\|\widehat{\Sigma}_n-\vv I_p\|=o(1)$ and $tr(\widehat{\Sigma}^{-1})=\Theta(p)$ with probability tending to 1. Thus for $\alpha=o(1)$,
\begin{eqnarray*}
\mathbb{E}_{\theta_0|\vv X,\vv y,\Sigma=\vv I_p}\|\theta^*(\delta)-\mathbb{E}_{\theta_0|\vv X,\vv y,\Sigma=\vv I_p}[\theta^*(\delta)]\|^2
&=& (1-\kappa(\delta))^2 \mathbb{E}_{\theta_0|\vv X,\vv y,\Sigma=\vv I_p}\|\theta_0-\widehat{\theta}_{n,\alpha}\|^2\\
&=&(1-\kappa(\delta))^2\frac{\sigma^2}{n} tr\left((\widehat{\Sigma}_n+\alpha \vv I_p)^{-1} \right)\\
&=&(1-\kappa(\delta))^2\frac{\sigma^2}{n} tr\left(\widehat{\Sigma}_n^{-1}-\alpha(\widehat{\Sigma}_n+\alpha \vv I_p)^{-1}\widehat{\Sigma}_{n}^{-1} \right)\\
&=&\left(1+ O\left(\frac{1}{ \lambda_{\max}(\hat{\Sigma}_n)/\alpha+1}\right)\right)(1-\kappa(\delta))^2\frac{\sigma^2}{n} tr\left(\widehat{\Sigma}_n^{-1}\right)\\
&=&(1+o(1))(1-\kappa(\delta))^2\frac{\sigma^2}{n} tr\left(\widehat{\Sigma}_n^{-1}\right).
\end{eqnarray*}

The above derivation is for unbounded $\theta_0$. Now we show that adding back the constraint $\|\theta_0\|\leq R$ does not change the order of this bound. 

Take $\alpha= (pR^2)/(n) $, and denote $\Pi(c)=\{ (\X,\y)\;|\; \|\widehat{\theta}_{n,\alpha}\|\in( \sigma(1+c),\sqrt{R^2+\sigma^2}(1-c) ],\; \|\widehat{\Sigma}_n-\vv I_p\|=o(1) \}$. Recall that $R\geq c_1\sigma$ for some constant $c_1>0$, thus there exists some small constant $c>0$, such that $P((\X,\y)\in\Pi(c))>c_2$ for some constant $c_2>0$.

For any $(\X,\y)\in \Pi(c)$, from the conditional distribution $\theta_0|\X,\y$ and the assumption that $(p\log^2 n)/n\rightarrow 0$, one can show that
\begin{eqnarray*}
&&\mathbb{E}_{\theta_0|\vv X, \vv y,\Sigma=\vv I_p}\left\|\mathbb{E}_{\theta_0|\vv X, \vv y,\Sigma=\vv I_p}\left[\theta^*(\delta)1_{\{\|\theta_0\|\in(\sigma,\sqrt{R^2+\sigma^2}]\}}\right]-\left[\theta^*(\delta)1_{\{\|\theta_0\|\in(\sigma,\sqrt{R^2+\sigma^2}]\}}\right]\right\|^2\\
&=&(1+o(1)) \mathbb{E}_{\theta_0|\vv X,\vv y,\Sigma=\vv I_p}\|\theta^*(\delta)-\mathbb{E}_{\theta_0|\vv X,\vv y,\Sigma=\vv I_p}[\theta^*(\delta)]\|^2.
\end{eqnarray*}
Consequently,
\begin{eqnarray*}
&&\inf_{\widehat{\theta}} \sup_{\delta,\sigma<\|\theta_0\|\leq \sqrt{R^2+\sigma^2}, \Sigma}\mathbb{E}\|\widehat{\theta}-\theta^*(\delta)\|^2\\
&=& \inf_{\widehat{\theta}} \sup_{\delta, \theta_0, \Sigma}\mathbb{E}\left[\|\widehat{\theta}-\theta^*(\delta)\|^21_{\{\|\theta_0\|\in(\sigma,\sqrt{R^2+\sigma^2}]\}}\right]\\
&\geq&\inf_{\widehat{\theta}}\sup_{\delta}\mathbb{E}\left[1_{\{ (\X,\y)\in \Pi(c) \}}\left(\mathbb{E}_{\theta_0|\vv X, \vv y,\Sigma=\vv I_p}\|\widehat{\theta}-\theta^*(\delta)\|^21_{\{\|\theta_0\|\in(\sigma,\sqrt{R^2+\sigma^2}]\}}\right)\right]\\
&\geq&(1+o(1))\sup_{\delta}\mathbb{E}\left[ 1_{\{ (\X,\y)\in \Pi(c) \}}\left(\mathbb{E}_{\theta_0|\vv X,\vv y,\Sigma=\vv I_p}\|\theta^*-\mathbb{E}_{\theta_0|\vv X,\vv y,\Sigma=\vv I_p}[\theta^*]\|^2\right)\right]\\
&=&\Omega\left( \frac{\sigma^2 p}{n} \right).
\end{eqnarray*}

\end{proof}

\begin{lemma}\label{lem:minimax:dense:Sigma}

	Assume $(p\log^2 n)/n\rightarrow 0$, then for any $\theta_0$, when $\lambda_{\min}(\Sigma)$ and $\lambda_{\max}(\Sigma)$ are both bounded and bounded away from zero, for any nonzero $\theta_0$,
	\begin{eqnarray}
	\inf_{\widehat{\theta}} \sup_{\delta,\Sigma}\mathbb{E}\|\widehat{\theta}-\theta^*(\delta)\|^2=\Omega\left(\frac{p\|\theta_0\|^2}{n}\right).
	\end{eqnarray}

\end{lemma}
\begin{proof}[Proof of Lemma \ref{lem:minimax:dense:Sigma}]

We impose a prior distribution on $\Sigma$. Assume $\Sigma$ follows $IW(\nu,\Lambda)$ with $\Lambda=(\nu-p-1)\vv I_p$ and $\nu=n+p+1$. In this case, we have $\mathbb{E}\Sigma=\frac{\Lambda}{\nu-p-1}=\vv I_p$, and
	\begin{eqnarray*}
		\Sigma|\vv X\sim IW(n+\nu,\Lambda+n\widehat{\Sigma}_n).
	\end{eqnarray*}
	Similar as Lemma \ref{lem:minimax:dense:theta}, we first relax the condition on the eigenvalues on $\Sigma$ to obtain a bound, then add back the conditions back to the bound.
	
	  Based on the distribution of $\Sigma|\X$, we have
	\begin{eqnarray}\label{eqn:minimax:dense:matrix:5}
	\inf_{\widehat{\theta}}\sup_{\delta}\mathbb{E}\left[\mathbb{E}_{\Sigma|\X}\|\widehat{\theta}-\theta^*(\delta)\|^2\right]
	&=&\inf_{\widehat{\theta}}\sup_{\delta}\mathbb{E}\left(\|\widehat{\theta}-\mathbb{E}_{\Sigma|\X}\theta^*(\delta)\|^2+\mathbb{E}_{\Sigma|\vv X}\|\mathbb{E}_{\Sigma|\vv X}(\theta^*)-\theta^*\|^2\right)\\
	&\geq&\sup_{\delta}\mathbb{E}\left[\mathbb{E}_{\Sigma|\vv X}\|\mathbb{E}_{\Sigma|\vv X}(\theta^*)-\theta^*\|^2\right].\nonumber
	\end{eqnarray}
	Denote $\lambda=\lambda^*((n\widehat{\Sigma}_n+\Lambda)/(n+\nu-p-1) ,\theta_0,\delta)$. For any $\delta>0$,
	\begin{eqnarray}
		&&\mathbb{E}_{\Sigma|\vv X}\|\mathbb{E}_{\Sigma|\vv X}(\theta^*)-\theta^*\|^2\nonumber\\
		\nonumber&\geq& \frac{1}{2}\mathbb{E}_{\Sigma|\vv X}\|(\mathbb{E}_{\Sigma|\vv X}(\Sigma+\lambda \vv I_p )^{-1}\Sigma\theta_0)-(\Sigma+\lambda \vv I_p )^{-1}\Sigma\theta_0)\|^2\\\nonumber&&-\mathbb{E}_{\Sigma|\vv X}\left\|\mathbb{E}_{\Sigma|\vv X}(\theta^*)-\mathbb{E}_{\Sigma|\vv X}(\Sigma+\lambda \vv I_p )^{-1}\Sigma\theta_0-\theta^*+(\Sigma+\lambda \vv I_p )^{-1}\Sigma\theta_0\right\|^2\\
		\nonumber&=& \frac{1}{2}\mathbb{E}_{\Sigma|\vv X}\|(\mathbb{E}_{\Sigma|\vv X}(\Sigma+\lambda \vv I_p )^{-1}\Sigma\theta_0)-(\Sigma+\lambda \vv I_p )^{-1}\Sigma\theta_0)\|^2\\\nonumber&&-\mathbb{E}_{\Sigma|\vv X}\left\|\mathbb{E}_{\Sigma|\vv X} (\lambda^*-\lambda)(\Sigma+\lambda^* \vv I_p)^{-1}(\Sigma+\lambda\vv I_p)^{-1}\Sigma\theta_0 - (\lambda^*-\lambda)(\Sigma+\lambda^* \vv I_p)^{-1}(\Sigma+\lambda\vv I_p)^{-1}\Sigma\theta_0 \right\|^2\\
		&=&  \frac{1}{2}\mathbb{E}_{\Sigma|\vv X}\|(\mathbb{E}_{\Sigma|\vv X}(\Sigma+\lambda \vv I_p )^{-1}\Sigma\theta_0)-(\Sigma+\lambda \vv I_p )^{-1}\Sigma\theta_0)\|^2-O\left(\mathbb{E}\| \Sigma-\mathbb{E}_{\Sigma|\vv X}\Sigma \|^4\|\theta_0\|^2\right).\label{eqn:minimax:dense:sigma:1}
	\end{eqnarray}
	When $(p\log^2n)/n\rightarrow 0$, $\mathbb{E}\| \Sigma-\mathbb{E}_{\Sigma|\vv X}\Sigma \|^4=O(((p\log n)/n)^2)=o(p/n)$ based on equation (A.5) in Lemma A.6 of \cite{ing2011stepwise}. As will be shown later, the dominant term is in $\Theta(p\|\theta_0\|^2/n)$, therefore $\mathbb{E}\| \Sigma-\mathbb{E}_{\Sigma|\vv X}\Sigma \|^4\|\theta_0\|^2$ is only a remainder term.  Furthermore,
	\begin{eqnarray*}
		&&\mathbb{E}_{\Sigma|\vv X}\|(\mathbb{E}_{\Sigma|\vv X}(\Sigma+\lambda \vv I_p )^{-1}\Sigma\theta_0)-(\Sigma+\lambda \vv I_p )^{-1}\Sigma\theta_0)\|^2\\
		&=&\lambda^2\mathbb{E}_{\Sigma\vv |\vv X}\|(\mathbb{E}_{\Sigma|\vv X}(\Sigma+\lambda \vv I_p )^{-1}\theta_0)-(\Sigma+\lambda \vv I_p )^{-1}\theta_0\|^2\\
		&:=&\lambda^2\psi(\lambda).
	\end{eqnarray*}
	Based on Lemma \ref{lem:minimax:dense:2} below, when $\lambda\geq 0$,
	\begin{eqnarray*}
		&&\psi(0)=\theta_0^{\top}\vv V_n \theta_0,\\
		&&\frac{\partial \psi(\lambda)}{\partial \lambda}\leq 0,\\
		&&\frac{\partial \psi(\lambda)}{\partial \lambda}\bigg|_{\lambda=0}=-\Theta(\theta_0^{\top}\vv V_n \theta_0),\\
		&&\frac{\partial^2 \psi(\lambda)}{\partial \lambda^2}\geq 0,
	\end{eqnarray*}
	where
	\begin{eqnarray*}
		\vv V_{n,i,j}&=& \sum_{k=1}^p Cov_{\Sigma|X}(\Sigma_{i,k}^{-1},\Sigma_{k,j}^{-1}).
	\end{eqnarray*}
	From the formula in \cite{Nydick2012TheWA},
	\begin{eqnarray*}
		Cov_{\Sigma|X}(\Sigma_{i,k}^{-1},\Sigma_{k,j}^{-1})&=&(n+\nu)\left((\Lambda+n\widehat{\Sigma}_{n})^{-1}_{i,k}(\Lambda+n\widehat{\Sigma}_{n})^{-1}_{k,j}+(\Lambda+n\widehat{\Sigma}_{n})^{-1}_{i,j}(\Lambda+n\widehat{\Sigma}_{n})^{-1}_{k,k}\right).
	\end{eqnarray*}
	Consequently,  there exists some constant $\epsilon>0$ such that, when $\delta$ satisfies $0<\lambda<\epsilon$, 
	\begin{eqnarray*}
		\mathbb{E}_{\Sigma|\X}\|\mathbb{E}_{\theta_0,\Sigma|\X,\y}(\theta^*)-\mathbb{E}_{\theta_0|\vv X,\vv y}[\theta^*|\Sigma]\|^2=\Omega\left( \frac{tr(\widehat{\Sigma}_n^{-1})\lambda_{\min}(\widehat{\Sigma}^{-1})\|\theta_0\|^2}{n} \right).
	\end{eqnarray*}
	
	Note that the above bound does not have restriction on $\Sigma$.
	
	Now we add back the condition where $\lambda_{\min}(\Sigma)$ and $\lambda_{\max}(\Sigma)$ are both bounded and bounded away from zero. When $(p\log^2 n)/n\rightarrow 0$ and $c_1\leq \lambda_{\min}(\widehat\Sigma_n)\leq\lambda_{\max}(\widehat\Sigma_n)\leq c_2$, there exists some constant $c>0$ such that
	\begin{eqnarray}\nonumber
	&&\mathbb{E}_{\Sigma\vv |\vv X}\|(\mathbb{E}_{\Sigma|\vv X}\Sigma^{-1}\theta_01_{\{ \lambda_{\max}(\Sigma),\lambda_{\min}(\Sigma)\in(c_1-c,c_2+c) \}})-\Sigma^{-1}\theta_01_{\{ \lambda_{\max}(\Sigma),\lambda_{\min}(\Sigma)\in(c_1-c,c_2+c) \}}\|^2\\&=&(1+o(1)) \mathbb{E}_{\Sigma\vv |\vv X}\|(\mathbb{E}_{\Sigma|\vv X}\Sigma^{-1}\theta_0)-\Sigma^{-1}\theta_0\|^2.\nonumber
	\end{eqnarray}
	
	Furthermore, since with probability tending to 1, $\|\widehat{\Sigma}_n-\vv I_p\|=O(\sqrt{(p\log n)/n})$, we also have with probability tending to 1, $\lambda^*((n\widehat{\Sigma}_n+\Lambda)/(n+\nu-p-1) ,\theta_0,\delta)=(1+o(1)) \lambda^*(\vv I_p,\theta_0,\delta)$. Therefore, denote $\delta^*_1$ and $\delta^*_2$ be the $\delta$'s such that $\lambda^*(\vv I_p,\theta_0,\delta)=0^+$ and $\lambda^*(\vv I_p,\theta_0,\delta)=\epsilon$ respectively, then when $\delta\in(\delta_1^*+\varepsilon,\delta_2^*-\varepsilon)$ for some small $\varepsilon>0$, with probability tending to 1,
	\begin{eqnarray*}
	\lambda^*((n\widehat{\Sigma}_n+\Lambda)/(n+\nu-p-1) ,\theta_0,\delta)\in(0,\epsilon).
	\end{eqnarray*}
	Recall that the prior distribution of $\Sigma$ is $IW(\nu,\Lambda)$, so there exists some $(C_1,C_2,c)>0$ such that with probability tending to 1, $\lambda_{\min}(\widehat\Sigma)>C_1+c,\lambda_{\max}(\widehat\Sigma)<C_2-c$. Therefore, taking $0<C_1+c<1<C_2-c<\infty$,
	\begin{eqnarray}
	&&\inf_{\widehat{\theta}} \sup_{\delta,\lambda_{\min}(\Sigma)>C_1,\lambda_{\max}(\Sigma)<C_2} \mathbb{E}\|\widehat{\theta}-\theta^*(\delta)\|^2\label{eqn:minimax:dense:matrix:1}\\
	&=& \inf_{\widehat{\theta}} \sup_{\delta,\Sigma} \mathbb{E}\|\widehat{\theta}-\theta^*(\delta)\|^21_{\{\lambda_{\min}(\Sigma)>C_1,\lambda_{\max}(\Sigma)<C_2\}}\label{eqn:minimax:dense:matrix:2}\\
	&\geq& (1+o(1))\sup_{\delta}\mathbb{E}\left[1_{\{\lambda_{\min}(\widehat\Sigma)>C_1+c,\lambda_{\max}(\widehat\Sigma)<C_2-c\}} \mathbb{E}_{\Sigma|\X}\|\mathbb{E}_{\theta_0,\Sigma|\X,\y}(\theta^*)-\mathbb{E}_{\theta_0|\vv X,\vv y}[\theta^*|\Sigma]\|^2\right]\label{eqn:minimax:dense:matrix:3}\\
	&\geq& (1+o(1))\sup_{\delta\in(\delta_1^*+\varepsilon,\delta_2^*-\varepsilon)}\mathbb{E}\bigg[1_{\{  \lambda^*((n\widehat{\Sigma}_n+\Lambda)/(n+\nu-p-1) ,\theta_0,\delta)\in(0,\epsilon)\}}1_{\{\lambda_{\min}(\widehat\Sigma)>C_1+c,\lambda_{\max}(\widehat\Sigma)<C_2-c\}}\nonumber\\&&\qquad\qquad\qquad\qquad\qquad\qquad\qquad\qquad\times \mathbb{E}_{\Sigma|\X}\|\mathbb{E}_{\theta_0,\Sigma|\X,\y}(\theta^*)-\mathbb{E}_{\theta_0|\vv X,\vv y}[\theta^*|\Sigma]\|^2\bigg]\label{eqn:minimax:dense:matrix:4}\\
	&=&\Omega\left( \frac{p\|\theta_0\|^2}{n} \right).\nonumber
	\end{eqnarray}
	From (\ref{eqn:minimax:dense:matrix:1}) to (\ref{eqn:minimax:dense:matrix:2}), we use the fact that the exact choice of $\Sigma$ in (\ref{eqn:minimax:dense:matrix:1}) will automatically leads to $1_{\{\lambda_{\min}(\Sigma)>C_1,\lambda_{\max}(\Sigma)<C_2\}}=1$, thus moving the eigenvalue conditions from $\sup$ to indicator function does not change the result. 
	
	From (\ref{eqn:minimax:dense:matrix:2}) to (\ref{eqn:minimax:dense:matrix:3}), we change from ``choosing the exact $\Sigma$" to ``$\Sigma$ satisfies a prior distribution", so the equality becomes inequality. Further, since under our choice of prior distribution of $\Sigma$, $\Sigma|\widehat\Sigma\rightarrow\widehat\Sigma$, we replace $1_{\{\lambda_{\min}(\Sigma)>C_1,\lambda_{\max}(\Sigma)<C_2\}}$ to $1_{\{\lambda_{\min}(\widehat\Sigma)>C_1+c,\lambda_{\max}(\widehat\Sigma)<C_2-c\}}$. The estimator $\widehat\theta$ is eliminated due to (\ref{eqn:minimax:dense:matrix:5}).
	
	From (\ref{eqn:minimax:dense:matrix:3}) to (\ref{eqn:minimax:dense:matrix:4}), we restrict the choice of $\delta$ into a certain range.
\end{proof}
\begin{lemma}\label{lem:minimax:dense:2}
	When $(p\log n)/n\rightarrow 0$, and all eigenvalues of $\widehat{\Sigma}_n$ are finite and bounded away from zero,
	\begin{eqnarray*}
		&&\psi(0)=\theta_0^{\top}\vv V_n \theta_0,\\
		&&\frac{\partial \psi(\lambda)}{\partial \lambda}\leq 0,\quad\frac{\partial \psi(\lambda)}{\partial \lambda}\bigg|_{\lambda=0}=-\Theta(\theta_0^{\top}\vv V_n \theta_0),\\
		&&\frac{\partial^2 \psi(\lambda)}{\partial \lambda^2}\geq 0,
	\end{eqnarray*}
\end{lemma}
\begin{proof}[Proof of Lemma \ref{lem:minimax:dense:2}]
Recall that the definition of $\psi$ is 
	\begin{eqnarray*}
\psi(\lambda)=\mathbb{E}_{\Sigma\vv |\vv X}\|(\mathbb{E}_{\Sigma|\vv X}(\Sigma+\lambda \vv I_p )^{-1}\theta_0)-(\Sigma+\lambda \vv I_p )^{-1}\theta_0\|^2,
	\end{eqnarray*}
	thus when $\lambda=0$, we have
	\begin{eqnarray*}
		\psi(0)=\mathbb{E}_{\Sigma|\vv X}\left\|\left[\Sigma^{-1}-\mathbb{E}_{\Sigma|\vv X}(\Sigma^{-1})\right]\theta_0\right\|^2=\theta_0^{\top}\vv V_n \theta_0.
	\end{eqnarray*}
	On the other hand,
	\begin{eqnarray*}
		\frac{\partial \psi(\lambda)}{\partial \lambda}&=&\frac{\partial }{\partial \lambda}\mathbb{E}_{\Sigma\vv |\vv X}\|(\mathbb{E}_{\Sigma|\vv X}(\Sigma+\lambda \vv I_p )^{-1}\theta_0)-(\Sigma+\lambda \vv I_p )^{-1}\theta_0\|^2\\
		&=&\mathbb{E}_{\Sigma\vv |\vv X}\frac{\partial }{\partial \lambda}\left\|(\mathbb{E}_{\Sigma|\vv X}(\Sigma+\lambda \vv I_p )^{-1}\theta_0)-(\Sigma+\lambda \vv I_p )^{-1}\theta_0\right\|^2\\
		&=&2\mathbb{E}_{\Sigma\vv |\vv X}\left[(\mathbb{E}_{\Sigma|\vv X}(\Sigma+\lambda \vv I_p )^{-1}\theta_0)-(\Sigma+\lambda \vv I_p )^{-1}\theta_0\right]^{\top}\frac{\partial }{\partial \lambda} \left[(\mathbb{E}_{\Sigma|\vv X}(\Sigma+\lambda \vv I_p )^{-1}\theta_0)-(\Sigma+\lambda \vv I_p )^{-1}\theta_0\right] \\
		&=&-2\mathbb{E}_{\Sigma\vv |\vv X}\left[(\mathbb{E}_{\Sigma|\vv X}(\Sigma+\lambda \vv I_p )^{-1}\theta_0)-(\Sigma+\lambda \vv I_p )^{-1}\theta_0\right]^{\top} \left[(\mathbb{E}_{\Sigma|\vv X}(\Sigma+\lambda \vv I_p )^{-2}\theta_0)-(\Sigma+\lambda \vv I_p )^{-2}\theta_0\right] \\
		&=&2(\mathbb{E}_{\Sigma|\vv X}(\Sigma+\lambda \vv I_p )^{-1}\theta_0)^{\top}(\mathbb{E}_{\Sigma|\vv X}(\Sigma+\lambda \vv I_p )^{-2}\theta_0)-2\mathbb{E}_{\Sigma\vv |\vv X}\theta_0^{\top}(\Sigma+\lambda \vv I_p )^{-3}\theta_0\\
		&\leq& 0,
	\end{eqnarray*}
	\begin{eqnarray*}
\frac{\partial^2 \psi(\lambda)}{\partial \lambda^2}&=&\frac{\partial^2 }{\partial \lambda^2}\mathbb{E}_{\Sigma\vv |\vv X}\|(\mathbb{E}_{\Sigma|\vv X}(\Sigma+\lambda \vv I_p )^{-1}\theta_0)-(\Sigma+\lambda \vv I_p )^{-1}\theta_0\|^2\\
		&=&2\mathbb{E}_{\Sigma\vv |\vv X}\left[(\mathbb{E}_{\Sigma|\vv X}(\Sigma+\lambda \vv I_p )^{-2}\theta_0)-(\Sigma+\lambda \vv I_p )^{-2}\theta_0\right] ^{\top} \left[(\mathbb{E}_{\Sigma|\vv X}(\Sigma+\lambda \vv I_p )^{-2}\theta_0)-(\Sigma+\lambda \vv I_p )^{-2}\theta_0\right] \\
		&&+4\mathbb{E}_{\Sigma\vv |\vv X}\left[(\mathbb{E}_{\Sigma|\vv X}(\Sigma+\lambda \vv I_p )^{-1}\theta_0)-(\Sigma+\lambda \vv I_p )^{-1}\theta_0\right]^{\top} \left[(\mathbb{E}_{\Sigma|\vv X}(\Sigma+\lambda \vv I_p )^{-3}\theta_0)-(\Sigma+\lambda \vv I_p )^{-3}\theta_0\right] \\
		&\geq&0.
	\end{eqnarray*}
	When $(p\log n)/n\rightarrow 0$, and all eigenvalues of $\widehat{\Sigma}_n$ are finite and bounded away from zero,
	\begin{eqnarray*}
		&&\theta_0^{\top}\left( \mathbb{E}_{\Sigma|\vv X}\Sigma^{-1}\mathbb{E}_{\Sigma|\vv X}\Sigma^{-2} -\mathbb{E}_{\Sigma|\vv X}\Sigma^{-3}  \right)\theta_0\\
		&=&\theta_0^{\top}\left( -\mathbb{E}_{\Sigma|\vv X}(\Sigma^{-1}-\mathbb{E}_{\Sigma|\vv X}\Sigma^{-1} )^3-2\mathbb{E}_{\Sigma|\vv X}\Sigma^{-1}\mathbb{E}_{\Sigma|\vv X}(\Sigma^{-1}-\mathbb{E}_{\Sigma|\vv X}\Sigma^{-1} )^2  \right)\theta_0\\
		&=&-\Theta(1)\theta_0^{\top}\vv V_n \theta_0.
	\end{eqnarray*}
	
\end{proof}
\begin{proof}[Proof of Theorem \ref{thm:minimax:dense}]
In Lemma \ref{lem:minimax:dense:theta} and \ref{lem:minimax:dense:Sigma}, we obtain two lower bounds for $\mathbb{E}\|\widehat{\theta}-\theta^*\|^2$, therefore the final lower bound just takes the larger one of these two bounds.

\end{proof}

\subsection{Theorem \ref{thm:minimax:sparse}}
\begin{proof}[Proof of Theorem \ref{thm:minimax:sparse}]
Similar as Theorem \ref{thm:minimax:dense}, we have the following decomposition:
\begin{eqnarray}
    &&\inf_{\widehat\theta}\sup_{\delta,\|\theta_0\|\leq R,\|\theta_0\|_0\leq s,\Sigma}\mathbb{E}\|\widehat{\theta}-\theta^*\|^2\nonumber\\&\geq& \left(\inf_{\widehat\theta}\sup_{\delta=0,\|\theta_0\|\leq R,\|\theta_0\|_0\leq s,\Sigma=\vv I_p}\mathbb{E}\|\widehat{\theta}-\theta^*\|^2\right)\vee \left(\inf_{\widehat\theta}\sup_{\delta,\theta_0=(1,0,0,...)^{\top},\Sigma}\mathbb{E}\|\widehat{\theta}-\theta^*\|^2\right).\label{eqn:minimax:sparse:1}
\end{eqnarray}

	For the first part of bound in (\ref{eqn:minimax:sparse:1}), it is directly followed from Proposition 4.3 of \cite{verzelen2010high}: for some constant $L>0$,
	\begin{equation*} \inf_{\widehat\theta}\sup_{\delta=0,\|\theta_0\|\leq R,\|\theta_0\|_0\leq s,\Sigma=\vv I_p}\mathbb{E}\|\widehat{\theta}-\theta^*\|^2\geq (sLR^2)\wedge \frac{sL\sigma^2(1+\log(p/s))}{n}. 
	\end{equation*}
	Since we assume $\|\theta_0\|/\sigma$ to be bounded away from zero, the above result becomes
		\begin{equation*} \inf_{\widehat\theta}\sup_{\delta=0,\|\theta_0\|\leq R,\|\theta_0\|_0\leq s,\Sigma=\vv I_p}\mathbb{E}\|\widehat{\theta}-\theta^*\|^2=\Omega\left( \frac{s\sigma^2(1+\log(p/s))}{n} \right). 
	\end{equation*}
	The above bound also holds for $\delta>0$ since when $\Sigma=\vv I_p$, $\theta^*=(1-\kappa(\delta))\theta_0$.
	
	For the second part of bound in (\ref{eqn:minimax:sparse:1}), we use Assouad's method and modify the proof in \cite{cai2016optimal}. Consider $\Sigma_1=\vv I_p$ and $\Sigma_2=\vv I_p+\vv D$, where $\vv D_{1,j}=\vv D_{j,1}=n^{-(\alpha+1)/(2\alpha+1)}$ for $j=1,...,n^{1/(2\alpha+1)}$. Denote $k=n^{1/(2\alpha+1)}$ and $a=n^{-(\alpha+1)/(2\alpha+1)}$, then $\vv D$ is just a matrix where the first $k$ elements in the first row and first column are $a$.
	
	Denote $P_{\Sigma}$ as the density of $N(0,\Sigma)$. Based on Assouad's Lemma, for any $\delta$ and $\theta_0$, for some constant $C>0$ (which is independent with $(\delta,\theta_0)$),
	\begin{eqnarray*}
	\inf_{\widehat{\theta}}\sup_{\Sigma}\|\widehat{\theta}-\theta^*(\Sigma,\delta)\|^2\geq C \|\theta^*(\Sigma_1,\delta)-\theta^*(\Sigma_2,\delta)\|^2 \|P_{\Sigma_1}\wedge P_{\Sigma_2}\|,
	\end{eqnarray*}
where $\|P_{\Sigma_1}\wedge P_{\Sigma_2}\|=\int P_{\Sigma_1}(x)\wedge P_{\Sigma_2}(x)dx$. The notation $\theta^*(\Sigma,\delta)$ is to emphasize the choice of $\Sigma$.
	
		From Lemma 6 of \cite{cai2016optimal}, we have
	\begin{eqnarray*}
	\|P_{\Sigma_1}\wedge P_{\Sigma_2}\|\geq c.
	\end{eqnarray*}
	
	As a result, our remaining task becomes to quantify $\|\theta^*(\Sigma_1,\delta)-\theta^*(\Sigma_2,\delta)\|^2$. Consider $\theta_0=(1,0,0,...,0)^{\top}$, for a given $\delta$ such that $\lambda_1:=\lambda^*(\theta_0,\Sigma_1,\delta)>0$, we have
	\begin{eqnarray}\label{eqn:thm:minimax:sparse:1}
	\lambda_1-\delta c_0\lambda_1\frac{ \|(\Sigma_1+\lambda_1\vv I_p)^{-1}\theta_0\|_{\Sigma_1} }{\|(\Sigma_1+\lambda_1\vv I_p)^{-1}\Sigma_1\theta_0\| }+\delta c_0 \frac{\|(\Sigma_1+\lambda_1\vv I_p)^{-1}\Sigma_1\theta_0\| }{ \|(\Sigma_1+\lambda_1\vv I_p)^{-1}\theta_0\|_{\Sigma_1} }-\delta^2=0.
	\end{eqnarray}
	Similarly, denote $\lambda_2:=\lambda^*(\theta_0,\Sigma_2,\delta)$, then
	\begin{eqnarray}\label{eqn:thm:minimax:sparse:2}
	\lambda_2-\delta c_0\lambda_2\frac{ \|(\Sigma_2+\lambda_2\vv I_p)^{-1}\theta_0\|_{\Sigma_2} }{\|(\Sigma_2+\lambda_2\vv I_p)^{-1}\Sigma_2\theta_0\| }+\delta c_0 \frac{\|(\Sigma_2+\lambda_2\vv I_p)^{-1}\Sigma_2\theta_0\| }{ \|(\Sigma_2+\lambda_2\vv I_p)^{-1}\theta_0\|_{\Sigma_2} }-\delta^2=0.
	\end{eqnarray}
	It is easy to observe that $\lambda_1-\lambda_2=O(\|\Sigma_1-\Sigma_2\|)$. However, since our aim is to figure out the lower bound of $\|\widehat{\theta}-\theta^*\|$, we want the lower bound of $|\lambda_1-\lambda_2|$. To characterize $\lambda_1-\lambda_2$ in details, observe that
	\begin{eqnarray*}
	&&\|(\Sigma_2+\lambda_2\vv I_p)^{-1}\theta_0\|_{\Sigma_2}-\|(\Sigma_1+\lambda_2\vv I_p)^{-1}\theta_0\|_{\Sigma_1}\\
	&=& \frac{1}{2\|(\Sigma_1+\lambda_2\vv I_p)^{-1}\theta_0\|_{\Sigma_1}}\left[ \theta_0^{\top}(\Sigma_2+\lambda_2\vv I_p)^{-1}\Sigma_2(\Sigma_2+\lambda_2\vv I_p)^{-1}\theta_0-\theta_0^{\top}(\Sigma_1+\lambda_2\vv I_p)^{-1}\Sigma_1(\Sigma_1+\lambda_2\vv I_p)^{-1}\theta_0 \right]+o\\
	&=&\frac{1}{2 \|(\Sigma_1+\lambda_2\vv I_p)^{-1}\theta_0\|_{\Sigma_1}}\theta_0^{\top}\left[ (\Sigma_2+\lambda_2\vv I_p)^{-1}  -(\Sigma_1+\lambda_2\vv I_p)^{-1} -\lambda_2 (\Sigma_2+\lambda_2\vv I_p)^{-2}+\lambda_2(\Sigma_1+\lambda_2\vv I_p)^{-2}      \right]\theta_0+o\\
	&=&\frac{\theta_0^{\top}\left[ -(\Sigma_2+\lambda_2\vv I_p)^{-1}\vv D (\Sigma_1+\lambda_2\vv I_p)^{-1} +\lambda_2(\Sigma_2+\lambda_2\vv I_p)^{-1}\vv D (\Sigma_1+\lambda_2\vv I_p)^{-1}\left( (\Sigma_2+\lambda_2\vv I_p)^{-1}+(\Sigma_1+\lambda_2\vv I_p)^{-1} \right)     \right]\theta_0}{ 2\|(\Sigma_1+\lambda_2\vv I_p)^{-1}\theta_0\|_{\Sigma_1}}+o\\
	&=&\frac{\theta_0^{\top}\left[ -(\Sigma_1+\lambda_2\vv I_p)^{-1}\vv D (\Sigma_1+\lambda_2\vv I_p)^{-1} +2\lambda_2(\Sigma_1+\lambda_2\vv I_p)^{-1}\vv D (\Sigma_1+\lambda_2\vv I_p)^{-2}    \right]\theta_0}{ 2\|(\Sigma_1+\lambda_2\vv I_p)^{-1}\theta_0\|_{\Sigma_1}}+o\\
	&=&\frac{1}{2 \|(\Sigma_1+\lambda_2\vv I_p)^{-1}\theta_0\|_{\Sigma_1}}\frac{\lambda_2-1}{(1+\lambda_2)^3}(2k-1)a+o\\
	&=&\frac{\lambda_2-1}{2(1+\lambda_2)^2}(2k-1)a+o,
	\end{eqnarray*}
	and
	\begin{eqnarray*}
	&&\|(\Sigma_2+\lambda_2\vv I_p)^{-1}\Sigma_2\theta_0\|-\| (\Sigma_1+\lambda_2\vv I_p)^{-1}\Sigma_1\theta_0 \|\\
	&=&\frac{1}{2\| (\Sigma_1+\lambda_2\vv I_p)^{-1}\Sigma_1\theta_0 \|}\theta_0^{\top}\left[  -2\lambda_2\left( (\Sigma_2+\lambda_2\vv I_p)^{-1}- (\Sigma_1+\lambda_2\vv I_p)^{-1} \right)+\lambda_2^2\left(   (\Sigma_2+\lambda_2\vv I_p)^{-2}- (\Sigma_1+\lambda_2\vv I_p)^{-2} \right) \right]\theta_0+o\\
	&=&\frac{1}{2\| (\Sigma_1+\lambda_2\vv I_p)^{-1}\Sigma_1\theta_0 \|}\theta_0^{\top}\left[  2\lambda_2 (\Sigma_1+\lambda_2\vv I_p)^{-1}\vv D (\Sigma_1+\lambda_2\vv I_p)^{-1} -2\lambda_2^2   (\Sigma_1+\lambda_2\vv I_p)^{-1}\vv D (\Sigma_1+\lambda_2\vv I_p)^{-2}  \right]\theta_0+o\\
	&=&\frac{1}{2\| (\Sigma_1+\lambda_2\vv I_p)^{-1}\Sigma_1\theta_0 \|}\frac{2\lambda_2}{(1+\lambda_2)^3}(2k-1)a+o\\
	&=&\frac{\lambda_2}{(1+\lambda_2)^2}(2k-1)a+o.
	\end{eqnarray*}
	Therefore, denote $\Delta_1=A(\theta_0,\Sigma_2,\delta,\lambda_2)-A(\theta_0,\Sigma_1,\delta,\lambda_2)$, with
	\begin{eqnarray*}
	A(\theta_0,\Sigma,\delta,\lambda)=\frac{\|(\Sigma+\lambda \vv I_p)^{-1}\Sigma\theta_0\|}{\|(\Sigma+\lambda\vv I_p)^{-1}\theta_0\|_{\Sigma}},
	\end{eqnarray*}
	then
	\begin{eqnarray*}
	A(\theta_0,\Sigma_2,\delta,\lambda_2)&=&\left(  \|(\Sigma_1+\lambda\vv I_p)^{-1}\Sigma_1\theta_0\|+\frac{\lambda_2}{(1+\lambda_2)^2}(2k-a)+o  \right)\\&&\times\left( \frac{1}{ \|(\Sigma_1+\lambda\vv I_p)^{-1}\theta_0\|_{\Sigma_1}}-\frac{1}{ \|(\Sigma_1+\lambda\vv I_p)^{-1}\theta_0\|_{\Sigma_1}^2}\frac{\lambda_2-1}{2(1+\lambda_2)^2}(2k-a)+o  \right)\\
	&=&-\frac{ \|(\Sigma_1+\lambda\vv I_p)^{-1}\Sigma_1\theta_0\|}{ \|(\Sigma_1+\lambda\vv I_p)^{-1}\theta_0\|_{\Sigma_1}^2}\frac{\lambda_2-1}{2(1+\lambda_2)^2}(2k-1)a+\frac{1}{ \|(\Sigma_1+\lambda\vv I_p)^{-1}\theta_0\|_{\Sigma_1}}\frac{\lambda_2}{(1+\lambda_2)^2}(2k-1)a+o\\
	&&+A(\theta_0,\Sigma_1,\delta,\lambda_2)\\
	&=&-\frac{\lambda_2-1}{2(1+\lambda_2)}(2k-1)a+\frac{\lambda_2}{1+\lambda_2}(2k-1)a+A(\theta_0,\Sigma_1,\delta,\lambda_2)+o\\
	&=&\frac{1}{2}(2k-1)a+A(\theta_0,\Sigma_1,\delta,\lambda_2)+o.
	\end{eqnarray*}
	Hence $\Delta_1=(2k-1)a+o$.

	Denote $\varepsilon=\lambda_2-\lambda_1$. Note that $A(\theta_0,\Sigma_1,\delta,\lambda)=1$ for any $\lambda\geq 0$ since $\Sigma_1=\vv I_p$. Therefore, (\ref{eqn:thm:minimax:sparse:1}) minus (\ref{eqn:thm:minimax:sparse:2}) leads to
	\begin{eqnarray*}
	0&=&-\varepsilon-\delta c_0\lambda_1 \frac{1}{A(\theta_0,\Sigma_1,\delta,\lambda_1)}+\delta c_0\lambda_2 \frac{1}{A(\theta_0,\Sigma_2,\delta,\lambda_2)}-\delta c_0 \Delta_1+A(\theta_0,\Sigma_1,\delta,\lambda_1)-A(\theta_0,\Sigma_1,\delta,\lambda_2)\\
	&=&-\varepsilon-\delta c_0\lambda_1 \frac{1}{A(\theta_0,\Sigma_1,\delta,\lambda_1)}+\delta c_0(\lambda_1+\varepsilon)\left[ \frac{1}{A(\theta_0,\Sigma_1,\delta,\lambda_2)}-\frac{\Delta_1}{A^2(\theta_0,\Sigma_1,\delta,\lambda_2)}\right]-\delta c_0 \Delta_1+o\\
	&=&-\varepsilon+\delta c_0\lambda_1\left( \frac{1}{A(\theta_0,\Sigma_1,\delta,\lambda_2)}- \frac{1}{A(\theta_0,\Sigma_1,\delta,\lambda_1)}\right) +\varepsilon\frac{\delta c_0}{A(\theta_0,\Sigma_1,\delta,\lambda_2)}-\frac{\lambda_1\delta c_0\Delta_1}{ A^2(\theta_0,\Sigma_1,\delta,\lambda_2)}-\delta c_0\Delta_1+o\\
	&=&-\varepsilon +\varepsilon\frac{\delta c_0}{A(\theta_0,\Sigma_1,\delta,\lambda_2)}-\frac{\lambda_1\delta c_0\Delta_1}{ A^2(\theta_0,\Sigma_1,\delta,\lambda_2)}-\delta c_0\Delta_1+o\\
	&=&-\varepsilon+\varepsilon
	\delta c_0-
	\lambda_1\delta c_0\Delta_1-
	\delta c_0\Delta_1+o.
	\end{eqnarray*}
	Consequently, 
	\begin{eqnarray*}
	\varepsilon =\delta c_0\frac{ \lambda_1+1 }{ \delta c_0-1 }\Delta_1+o,
	\end{eqnarray*}
	and hence
	\begin{eqnarray*}
	&&(\Sigma_1+\lambda_1\vv I_p)^{-1}\Sigma_1\theta_0-(\Sigma_2+\lambda_2\vv I_p)^{-1}\Sigma_2\theta_0 \\
	&=&(\Sigma_1+\lambda_1\vv I_p)^{-1}\Sigma_1\theta_0-(\Sigma_2+\lambda_2\vv I_p)^{-1}\Sigma_1\theta_0 +(\Sigma_2+\lambda_2\vv I_p)^{-1}\Sigma_1\theta_0 -(\Sigma_2+\lambda_2\vv I_p)^{-1}\Sigma_2\theta_0 \\
	&=&(\Sigma_1+\lambda_1\vv I_p)^{-1} (\vv D+\varepsilon \vv I_p) (\Sigma_2+\lambda_1\vv I_p)^{-1}\Sigma_1\theta_0-  (\Sigma_2+\lambda_2\vv I_p)^{-1}\vv D\theta_0\\
	&=&(\Sigma_1+\lambda_1\vv I_p)^{-1} (\vv D+\varepsilon \vv I_p) (\Sigma_1+\lambda_1\vv I_p)^{-1}\Sigma_1\theta_0-  (\Sigma_1+\lambda_1\vv I_p)^{-1}\vv D\theta_0+o\\
	&=&\frac{1}{(1+\lambda_1)^2} (\vv D+\varepsilon \vv I_p)\theta_0-\frac{1}{1+\lambda_1}\vv D\theta_0+o.
	\end{eqnarray*}
	Since
	\begin{eqnarray*}
	\varepsilon\theta_0+\vv D\theta_0=\begin{bmatrix}
	\varepsilon+a\\
	 a\\
	...\\
 a\\
	0\\
	...
	\end{bmatrix},
	\end{eqnarray*}
	and recall that $k=n^{1/(2\alpha+1)}$ and $a=n^{-(\alpha+1)/(2\alpha+1)}$, when $\delta$ is chosen such that $\varepsilon=\Theta(ka)$, we have
	\begin{eqnarray*}
	\left\|(\Sigma_1+\lambda_1\vv I_p)^{-1}\Sigma_1\theta_0-(\Sigma_2+\lambda_2\vv I_p)^{-1}\Sigma_2\theta_0\right\|^2=\Omega(k^2a^2)=\Omega\left(n^{-\frac{2\alpha}{2\alpha+1}}\right).
	\end{eqnarray*}
	
	As a result, we conclude that
	\begin{eqnarray}
	    \inf_{\widehat{\theta}}\sup_{\delta,\|\theta_0\|\leq R,\|\theta_0\|_0\leq s,\Sigma}\|\widehat\theta-\theta^*\|^2=\Omega \left( R^2n^{-\frac{2\alpha}{2\alpha+1}} \right).
	\end{eqnarray}
\end{proof}

	\subsection{Proof of Theorem \ref{thm:lsm}}
	\begin{proof}[Proof of Theorem \ref{thm:lsm}]
		There exacts a constant $\delta'>\delta_1$ such that as $\delta\ge\delta',$ $\widehat{\theta}=\theta^*=0.$ Thus, we have $R_0(\widehat{\theta},\delta)-R_0(\theta^*,\delta)=0$ when $\delta\ge\delta'.$ Next, we will show for any $\delta\le\delta',$ (\ref{eqn:thm3}) always hold. 
		
		To simplify notations, denote $\widehat{\theta}(\lambda)=\widehat{\theta}_0-(\widehat{\Sigma}/\lambda+\vv I_p)^{-1}\widehat{\theta}_0$ and $\theta(\lambda)=\theta_0-(\Sigma/\lambda+\vv I_p)^{-1}\theta_0$, $\widehat{R}(\theta,\delta)=R_n(\theta,\widehat{\theta}_0,\widehat{\Sigma},\delta)$ as in (\ref{eqn:empirical}). Then
		\begin{eqnarray*}
			R_0(\widehat{\theta}(\lambda),\delta)-\widehat{R}
			_0(\widehat{\theta}(\lambda),\delta)=\| \widehat{\theta}(\lambda)-\theta_0 \|_{\Sigma}^2+2\delta c_0\|\widehat{\theta}(\lambda)\|\|\widehat{\theta}(\lambda)-\theta_0\|_{\Sigma}-\|\widehat{\theta}(\lambda)-\widehat{\theta}_0\|_{\widehat{\Sigma}}^2-2\delta c_0\|\widehat{\theta}(\lambda)\|\|\widehat{\theta}(\lambda)-\widehat{\theta}_0\|_{\widehat{\Sigma}}.
		\end{eqnarray*}
		From the formula of $\widehat{\theta}(\lambda)$, we have $\sup_{\lambda}\|2\widehat{\theta}(\lambda)-\theta_0-\widehat{\theta}_0\|$ and $\|\widehat{\theta}(\lambda)-\widehat{\theta}_0\|$ are always in $O(\|\theta_0\|)$, therefore
		\begin{eqnarray*}
			\left|\| \widehat{\theta}(\lambda)-\theta_0 \|_{\Sigma}^2-\|\widehat{\theta}(\lambda)-\widehat{\theta}_0\|_{\widehat{\Sigma}}^2\right|
			&=&(\widehat{\theta}_0-\theta_0)^{\top}\Sigma(2\widehat{\theta}(\lambda)-\theta_0-\widehat{\theta}_0)-\|\widehat{\theta}(\lambda)-\widehat{\theta}_0\|_{\widehat{\Sigma}-\Sigma}^2\\
			&=&O(\|\widehat{\theta}_0-\theta_0\|\|\theta(\lambda)-\theta_0\|)+O(\|\theta(\lambda)-\theta_0\|^2\|\widehat{\Sigma}-\Sigma\|).
		\end{eqnarray*}
		Based on similar arguments, 
		\begin{eqnarray*}
			\left|\|\widehat{\theta}(\lambda)-\theta_0\|_{\Sigma}-\|\widehat{\theta}(\lambda)-\widehat{\theta}_0\|_{\widehat{\Sigma}}\right|&=&\left|\|\widehat{\theta}(\lambda)-\theta_0\|_{\Sigma}-\|\widehat{\theta}(\lambda)-\widehat\theta_0\|_{\Sigma}+\|\widehat{\theta}(\lambda)-\widehat\theta_0\|_{\Sigma}-\|\widehat{\theta}(\lambda)-\widehat{\theta}_0\|_{\widehat{\Sigma}}\right|\\
			&=&O(\|\widehat{\theta}_0-\theta_0\|)+O\left(\|\theta(\lambda)-\theta_0\|\sqrt{\|\widehat{\Sigma}-\Sigma\|}\right).
		\end{eqnarray*}
		Thus, 
		\begin{eqnarray*}
			\left|\|\widehat{\theta}(\lambda)\|\|\widehat{\theta}(\lambda)-\theta_0\|_{\Sigma}-\|\widehat{\theta}(\lambda)\|\|\widehat{\theta}(\lambda)-\widehat{\theta}_0\|_{\widehat{\Sigma}}\right|= O(\|\theta(\lambda)\|\|\widehat{\theta}_0-\theta_0\|)+O\left(\|{\theta}(\lambda)\|\|\theta(\lambda)-\theta_0\|\sqrt{\|\widehat{\Sigma}-\Sigma\|}\right).
		\end{eqnarray*}
		Therefore, uniformly for all $\lambda$:
		\begin{eqnarray}
		&&R_0(\widehat{\theta}(\lambda),\delta)-\widehat{R}_0(\widehat{\theta}(\lambda),\delta)\\&=&O(\|\widehat{\theta}_0-\theta_0\|\|\theta(\lambda)-\theta_0\|)+O(\|\theta(\lambda)-\theta_0\|^2\|\widehat{\Sigma}-\Sigma\|)+O(\delta\|{\theta}(\lambda)\|\|\widehat{\theta}_0-\theta_0\|)+O\left(\delta\|{\theta}(\lambda)\|\|\theta(\lambda)-\theta_0\|\sqrt{\|\widehat{\Sigma}-\Sigma\|}\right).
		\end{eqnarray}
		When $\delta\rightarrow\infty$, $\delta\|\theta(\lambda)\|/\|\theta_0\|$ converges to some constant. For any $\delta>0$, $\delta\|\theta^*\|/\|\theta_0\|$ is finite. As a result, 
		\begin{eqnarray}\label{eqn:bound1}
		R_0(\widehat{\theta}(\lambda),\delta)-\widehat{R}_0(\widehat{\theta}(\lambda),\delta)=O\left( \|\widehat{\theta}_0-\theta_0\|\|\theta_0\| \right)+O\left(\|\theta_0\|\sqrt{\|\widehat{\Sigma}-\Sigma\|}\right).
		\end{eqnarray}
		By similar derivations, we can get uniformly for all $\lambda$:
		\begin{eqnarray}\label{eqn:bound2}
		R_0(\theta(\lambda),\delta)-\widehat{R}_0(\theta(\lambda),\delta)=O\left( \|\widehat{\theta}_0-\theta_0\|\|\theta_0\| \right)+O\left(\|\theta_0\|\sqrt{\|\widehat{\Sigma}-\Sigma\|}\right).
		\end{eqnarray}
		
		From the definition of $R$ and $\widehat{R}$, we have
		\begin{eqnarray*}
			R_0(\widehat{\theta},\delta)-R_0(\theta^*,\delta)=R_0(\widehat{\theta},\delta)-\widehat{R}_0(\widehat{\theta},\delta)+\widehat{R}_0(\widehat{\theta},\delta)-\widehat{R}_0(\theta^*,\delta)+\widehat{R}_0(\theta^*,\delta)-{R_0}(\theta^*,\delta).
		\end{eqnarray*}
		Since $\widehat{\lambda}$ is the minimizer of $\widehat{R}(\widehat{\theta}(\lambda),\delta)$, it becomes
		\begin{eqnarray*}
			\widehat{R}(\widehat{\theta},\delta)-\widehat{R}(\theta^*,\delta)<0.
		\end{eqnarray*}
		By the universal bounds in (\ref{eqn:bound1}), (\ref{eqn:bound2}), we obtain
		\begin{eqnarray*}
			R_0(\widehat{\theta},\delta)-R_0(\theta^*,\delta)=O\left( \|\widehat{\theta}_0-\theta_0\|\|\theta_0\| \right)+O\left(\|\theta_0\|\sqrt{\|\widehat{\Sigma}-\Sigma\|}\right).
		\end{eqnarray*}
	\end{proof}
	\subsection{Theorem \ref{thm:bah}}
	
	\begin{proof}[Proof of Theorem \ref{thm:bah}]
		In the proof, we first assume $\|\theta_0\|$ and $\sigma^2$ are finite, then extend to unbounded $\|\theta_0\|$ and $\sigma^2$ in the end. For simplification, we define
		\begin{eqnarray*}
			R_0(\theta_1,\theta_2,\Sigma)=\|\theta_1-\theta_2\|^2_{\Sigma}+2\delta  c_0 \|\theta_1\|\|\theta_1-\theta_2\|_{\Sigma}+\delta^2\|\theta_1\|^2_2.
		\end{eqnarray*}
		We will prove the theorem based on different scenarios of $\delta$. Denote $\theta^*$, $\widetilde{\theta}$, and $\widehat{\theta}$ as the minimizers of  $R_0(\cdot,\theta_0,\Sigma)$, $R_0(\cdot,\widehat{\theta}_0,\Sigma)$, and $R_0(\cdot,\widehat{\theta}_0 ,\widehat\Sigma).$ Then we consider the partial derivative of $R_0(\theta_1,\theta_2,\Sigma),$
		\begin{eqnarray*}
			\frac{\partial R_0(\theta_1,\theta_2,\Sigma)}{\partial \theta_1}=2\left[ \left(1+\delta c A(\theta_1,\theta_2,\Sigma)\right) \Sigma(\theta_1-\theta_2)+\left( \delta c \frac{1}{A(\theta_1,\theta_2,\Sigma)}+\delta^2 \right)\theta_1\right],
		\end{eqnarray*}
		where $A(\theta_1,\theta_2,\Sigma)=\|\theta_1\|/\|\theta_1-\theta_2\|_{\Sigma}.$
		\paragraph{Case 1:} When $\delta_1<\delta<\delta_2,$ based on Proposition \ref{thm:opt}, the minimizer $\theta^*$ is neither $\theta_0$ nor $0.$ Thus, for large $n$ (such that the probability of $\widehat{\theta}$ being 0 or $\widehat{\theta}_0$ can be ignored), from the first order optimality condition of $\theta^*$, $\widetilde{\theta}$, and $\widehat{\theta}$, we first have
		\begin{eqnarray}\label{eq:30}
		\vv 0&=&\frac{\partial R_0}{2\partial \theta_1}(\widehat\theta,\widehat\theta_0,\widehat\Sigma)-\frac{\partial R_0}{2\partial {\theta_1}}(\theta^*,\theta_0,\Sigma)\nonumber\\
		&=&\left[ \left(1+\delta c A(\widehat\theta,\widehat\theta_0,\widehat\Sigma)\right) \widehat\Sigma(\widehat\theta-\widehat\theta_0)+\left( \delta c \frac{1}{A(\widehat\theta,\widehat\theta_0,\widehat\Sigma)}+\delta^2 \right)\widehat\theta\right]- \nonumber\\
		&&\left[ \left(1+\delta c A(\theta^*,\theta_0,\Sigma)\right) \Sigma(\theta^*-\theta_0)+\left( \delta c \frac{1}{A(\theta^*,\theta_0,\Sigma)}+\delta^2 \right)\theta^*\right]\nonumber\\
		&=&\left[ \left(1+\delta c A(\widehat\theta,\widehat\theta_0,\widehat\Sigma)\right) \Sigma(\widehat\theta-\widehat\theta_0)+\left( \delta c \frac{1}{A(\widehat\theta,\widehat\theta_0,\widehat\Sigma)}+\delta^2 \right)\widehat\theta\right]- \nonumber\\
		&&\left[ \left(1+\delta c A(\theta^*,\theta_0,\Sigma)\right) \Sigma(\theta^*-\theta_0)+\left( \delta c \frac{1}{A(\theta^*,\theta_0,\Sigma)}+\delta^2 \right)\theta^*\right]\nonumber\\
		&&+\left(1+\delta c A(\widehat\theta,\widehat\theta_0,\widehat\Sigma)\right) (\widehat\Sigma-\Sigma)(\widehat\theta-\widehat\theta_0).
		\end{eqnarray}
		Consider the Taylor expansions of $A(\widehat\theta,\widehat\theta_0,\widehat\Sigma), \frac{1}{A(\widehat\theta,\widehat\theta_0,\widehat\Sigma)}$ at $(\theta^*,\theta_0,\Sigma)$. For both $A$ and $1/A$, we observe that 
		\begin{eqnarray}
		A(\widehat\theta,\widehat\theta_0,\widehat\Sigma)&=&A(\theta^*,\theta_0,\Sigma)+\left(\frac{\partial A}{\partial \theta_1}(\theta^*,\theta_0,\Sigma)\right)^{\top} (\widehat{\theta}-\theta^*)+\left(\frac{\partial A}{\partial \theta_2}(\theta^*,\theta_0,\Sigma)\right)^{\top} (\widehat\theta_0-\theta_0)\nonumber\\
		&&+ \frac{(\theta^*-\theta_0)^{\top}(\widehat{\Sigma}-\Sigma)(\theta^*-\theta_0)}{2\|\theta^*-\theta_0\|_{\Sigma}\|\theta^*\|} +O\left(\frac{\|\widehat{\theta}-\theta^*\|^2}{\|\theta_0\|^2}\right)\nonumber\\
		\frac{1}{A(\widehat\theta,\widehat\theta_0,\widehat\Sigma)}&=&\frac{1}{A(\theta^*,\theta_0,\Sigma)}+\left(\frac{\partial 1/A}{\partial \theta_1}(\theta^*,\theta_0,\Sigma)\right)^{\top} (\widehat{\theta}-\theta^*)+\left(\frac{\partial 1/A}{\partial \theta_2}(\theta^*,\theta_0,\Sigma)\right)^{\top} (\widehat\theta_0-\theta_0)\nonumber\\
		&&-\|\theta^*\|\frac{(\theta^*-\theta_0)^{\top}(\widehat{\Sigma}-\Sigma)(\theta^*-\theta_0)}{2\|\theta^*-\theta_0\|_{\Sigma}^3}+O\left(\frac{\|\widehat{\theta}-\theta^*\|^2}{\|\theta_0\|^2}\right)\nonumber
		\end{eqnarray}
		Moreover,
		\begin{eqnarray*}
			&&\left(1+\delta c A(\widehat\theta,\widehat\theta_0,\widehat\Sigma)\right) (\widehat\Sigma-\Sigma)(\widehat\theta-\widehat\theta_0)\\&=&\left(1+\delta c A(\theta^*,\theta_0,\Sigma)\right) (\widehat\Sigma-\Sigma)(\theta^*-\theta_0)+O(\|\widehat{\theta}-\theta^*\|^2/\|\theta_0\|)+O(\|\widehat{\theta}_0-\theta_0\|^2/\|\theta_0\|)+O(\|\widehat{\Sigma}-\Sigma\|^2\|\theta_0\|)
		\end{eqnarray*}
		Combined with (\ref{eq:30}) yields
		\begin{eqnarray*}
			&&	\Sigma( \widehat{\theta}-\theta^*+\theta_0-\widehat{\theta}_0) +  \delta c A(\theta^*,\theta_0,\Sigma) \Sigma( \widehat{\theta}-\theta^*+\theta_0-\widehat{\theta}_0) +\left( \frac{\delta c}{A(\theta^*,\theta_0,\Sigma)}+\delta^2 \right)(\widehat{\theta}-\theta^*)\\
			&&+\delta c\Sigma(\theta^*-{\theta}_0)\left(\frac{\partial A}{\partial \theta_1}(\theta^*,\theta_0,\Sigma)\right)^{\top} (\widehat{\theta}-\theta^*)+\delta c \theta^* \left(\frac{\partial 1/A}{\partial \theta_1}(\theta^*,\theta_0,\Sigma)\right)^{\top} (\widehat{\theta}-\theta^*)+\\
			&&\delta c \Sigma(\theta^*-{\theta}_0)\left(\frac{\partial A}{\partial \theta_2}(\theta^*,\theta_0,\Sigma)\right)^{\top} (\widehat\theta_0-\theta_0)+\delta c\theta^* \left(\frac{\partial 1/A}{\partial \theta_2}(\theta^*,\theta_0,\Sigma)\right)^{\top} (\widehat\theta_0-\theta_0)\\
			&&+\delta c\frac{(\theta^*-\theta_0)^{\top}(\widehat{\Sigma}-\Sigma)(\theta^*-\theta_0)}{\|\theta^*-\theta_0\|_{\Sigma}\|\theta^*\|}\Sigma(\theta^*-\theta_0)-\delta c\|\theta^*\|\frac{(\theta^*-\theta_0)^{\top}(\widehat{\Sigma}-\Sigma)(\theta^*-\theta_0)}{\|\theta^*-\theta_0\|_{\Sigma}^3}\theta^*\\
			&&+\left(1+\delta c A(\theta^*,\theta_0,\Sigma)\right) (\widehat\Sigma-\Sigma)(\theta^*-\theta_0)\\
			&&+O\left(\frac{\|\widehat{\theta}-\theta^*\|^2}{\|\theta_0\|}\right)=0.
		\end{eqnarray*}
		
		Thus the difference between $\widehat{\theta}$ and $\theta^*$ is dominated by $\widehat{\theta}_0-\theta_0,$ and $\widehat{\Sigma}-\Sigma$:
		\begin{eqnarray}
		\widehat{\theta}-\theta^*&=&\left[\vv H(\theta^*,\theta_0,\Sigma)\right]^{-1} \Bigg[ \vv M(\theta^*,\theta_0,\Sigma)(\widehat{\theta}_0-\theta_0)-\delta c \frac{(\theta^*-\theta_0)^{\top}(\widehat{\Sigma}-\Sigma)(\theta^*-\theta_0)}{2\|\theta^*-\theta_0\|_{\Sigma}\|\theta^*\|}\Sigma(\theta^*-\theta_0)\\
		&&+\delta c\|\theta^*\|\frac{(\theta^*-\theta_0)^{\top}(\widehat{\Sigma}-\Sigma)(\theta^*-\theta_0)}{2\|\theta^*-\theta_0\|_{\Sigma}^3}\theta^*-\left(1+\delta c A(\theta^*,\theta_0,\Sigma)\right) (\widehat\Sigma-\Sigma)(\theta^*-\theta_0)\Bigg]\nonumber\\
		&&+O\left(\frac{\|\widehat{\theta}-\theta^*\|^2}{\|\theta_0\|}\right),\nonumber
		\end{eqnarray}
		where
		\begin{eqnarray*} 
			\vv H(\theta^*,\theta_0,\Sigma)&=&\Sigma+\delta c A(\theta^*,\theta_0,\Sigma)\Sigma+\left( \frac{\delta c}{A(\theta^*,\theta_0,\Sigma)}+\delta^2 \right) \vv I_p +\delta c\Sigma(\theta^*-\theta_0)\left(\frac{\partial A}{\partial \theta_1}(\theta^*,\theta_0,\Sigma)\right)^{\top} \\
			&&+\delta c\theta^*\left(\frac{\partial 1/A}{\partial \theta_1}(\theta^*,\theta_0,\Sigma)\right)^{\top} 
		\end{eqnarray*}
		is the Hessian matrix of the population risk defined in (\ref{eq:hessian}) at point $\theta=\theta^*,$ which is positive-definite by Proposition \ref{thm:opt}.
		\begin{eqnarray*}
			\vv M(\theta^*,\theta_0,\Sigma)&=& \Sigma+\delta c A(\theta^*,\theta_0,\Sigma)\Sigma -\delta c\Sigma(\theta^*-\theta_0)\left(\frac{\partial A}{\partial \theta_2}(\theta^*,\theta_0,\Sigma)\right)^{\top} -\delta c\theta^*\left(\frac{\partial 1/A}{\partial \theta_2}(\theta^*,\theta_0,\Sigma)\right)^{\top} \\
			&=&\Sigma+\delta c A(\theta^*,\theta_0,\Sigma)\Sigma +\frac{\delta c A(\theta^*,\theta_0,\Sigma)}{\|\theta^*-\theta_0\|^2_{\Sigma}}\Sigma(\theta^*-\theta_0)(\theta^*-\theta_0)^{\top}\Sigma \\
			&&+\delta \frac{c}{A(\theta^*,\theta_0,\Sigma)\|\theta^*\|^2_2}\theta^*(\theta^*)^{\top},
		\end{eqnarray*}
		which is also positive definite.
		


		\paragraph{Case 2:} $\delta$ is either smaller than $\delta_1$ or larger than $\delta_2$. Recall that $\delta_1=\frac{ c_0 \|\theta_0\|}{\sqrt{\theta^{\top}_0\left( \Sigma^{-1} \right)\theta_0}}$ and $\delta_2=\frac{\sqrt{\theta_0^{\top} \Sigma^2\theta_0}}{ c_0 \sqrt{\theta_0^{\top}\Sigma \theta_0}}$.
		
		\begin{eqnarray*}
			\widehat{\delta}_1-\delta_1&=&\frac{\partial \delta_1}{\partial \theta_0}(\widehat{\theta}_0-\theta_0)+\left\langle\frac{\partial \delta_1}{\partial \Sigma},\widehat{\Sigma}-\Sigma\right\rangle_F+O\left(\frac{\|\widehat{\theta}_0-\theta_0\|^2}{\|\theta_0\|^2}\right)+O(\|\widehat{\Sigma}-\Sigma\|^2),\\
			\widehat{\delta}_2-\delta_2&=&\frac{\partial \delta_2}{\partial \theta_0}(\widehat{\theta}_0-\theta_0)+\left\langle\frac{\partial \delta_2}{\partial \Sigma},\widehat{\Sigma}-\Sigma\right\rangle_F+O\left(\frac{\|\widehat{\theta}_0-\theta_0\|^2}{\|\theta_0\|^2}\right)+O(\|\widehat{\Sigma}-\Sigma\|^2).
		\end{eqnarray*}
		Therefore, if $\widehat{\theta}_0-\theta_0$ and $\widehat{\Sigma}-\Sigma$ are consistent, with probability tending to one, $\delta$ will smaller than $\widehat\delta_1$ or greater than $\widehat\delta_2.$ Thus, $\widehat{\theta}$ will be either $\widehat{\theta}_0$ or $0$ depending on $\delta$.
		
	\end{proof}
	\subsection{Theorem \ref{coro:generalization}}
	\begin{proof}[Proof of Theorem \ref{coro:generalization}]
		The decomposition of generalizations can be directly obtained from Taylor expansion. When $\delta<\delta_1$, since $\widehat{\delta}_1-\delta_1\rightarrow 0$, we have
		\begin{eqnarray*}
		R_0(\widehat{\theta},\delta)-\widehat{R}_0(\widehat{\theta},\delta)
		&=&   \vv 1\{ \widehat{\delta}_1\geq\delta \}\left[\|\widehat{\theta}_0-\theta_0\|_{\Sigma}^2-\|\widehat{\theta}_0-\widehat{\theta}_0\|_{\widehat\Sigma}^2+2\delta c_0\|\widehat\theta_0\|\left( \sqrt{\|\widehat{\theta}_0-\theta_0\|_{\Sigma}^2 } -  \sqrt{\|\widehat{\theta}_0-\widehat\theta_0\|_{\Sigma}^2  } \right)\right]\\&&+   \vv 1\{ \widehat{\delta}_1<\delta \}\left[ \|\widehat{\theta}-\theta_0\|_{\Sigma}^2-\|\widehat{\theta}-\widehat{\theta}_0\|_{\widehat{\Sigma}}^2+2\delta c_0\|\widehat{\theta}\|\left( \sqrt{\|\widehat{\theta}-\theta_0\|_{\Sigma}^2 } - \sqrt{\|\widehat{\theta}-\widehat{\theta}_0\|_{\widehat{\Sigma}}^2}\right) \right]\\
		&=&\|\widehat{\theta}_0-\theta_0\|_{\Sigma}^2+2\delta c_0\|\theta_0\|\|\widehat{\theta}_0-\theta_0\|_{\Sigma}+o_p( R_0(\widehat{\theta},\delta)-\widehat{R}_0(\widehat{\theta},\delta)).
		\end{eqnarray*}
		When $\delta>\delta_1$,
		\begin{eqnarray*}
			&&R_0(\widehat{\theta},\delta)-\widehat{R}_0(\widehat{\theta},\delta)\\&=&\vv 1\{ \widehat{\delta}_1<\delta \}\left[\|\widehat{\theta}-\theta_0\|_{\Sigma}^2-\|\widehat{\theta}-\widehat{\theta}_0\|_{\widehat{\Sigma}}^2+2\delta c_0\|\widehat{\theta}\|\left( \sqrt{\|\widehat{\theta}-\theta_0\|_{\Sigma}^2 } - \sqrt{\|\widehat{\theta}-\widehat{\theta}_0\|_{\widehat{\Sigma}}^2}\right)\right]\\
			&&+\vv 1\{ \widehat{\delta}_1\geq \delta \}\left[\|\widehat{\theta}_0-\theta_0\|_{\Sigma}^2-\|\widehat{\theta}_0-\widehat{\theta}_0\|_{\widehat{\Sigma}}^2+2\delta c_0\|\widehat{\theta}_0\|\left( \sqrt{\|\widehat{\theta}_0-\theta_0\|_{\Sigma}^2 } - \sqrt{\|\widehat{\theta}_0-\widehat{\theta}_0\|_{\widehat{\Sigma}}^2}\right)\right]\\
			&=&\|\widehat{\theta}-\theta_0\|_{\Sigma}^2-\|\widehat{\theta}-\theta_0\|_{\widehat\Sigma}^2+\|\widehat{\theta}-\theta_0\|_{\widehat\Sigma}^2-\|\widehat{\theta}-\widehat{\theta}_0\|_{\widehat{\Sigma}}^2\\
			&&+2\delta c_0\|\theta^*\|\left( \sqrt{\|\widehat{\theta}-\theta_0\|_{\Sigma}^2 } -  \sqrt{\|\widehat{\theta}-\widehat\theta_0\|_{\Sigma}^2  } \right) +2\delta c_0\|\theta^*\|\left( \sqrt{\|\widehat{\theta}-\widehat\theta_0\|_{\Sigma}^2  } -\sqrt{\|\widehat{\theta}-\widehat{\theta}_0\|_{\widehat{\Sigma}}^2}\right)\\&&+o_p(R(\widehat{\theta},\delta)-\widehat{R}(\widehat{\theta},\delta))\\
			&=&(\theta^*-\theta_0)^{\top}(\Sigma-\widehat{\Sigma} )(\theta^*-\theta_0)+2(\widehat{\theta}_0-\theta_0)^{\top}\Sigma(\theta^*-\theta_0)+2\delta c_0\|\theta^*\|  \frac{ (\widehat{\theta}_0-\theta_0)^{\top}\Sigma(\theta^*-\theta_0) }{\sqrt{\|\theta^*-\theta_0\|_{\Sigma}^2+\sigma^2  }}\\&&+o_p(R(\widehat{\theta},\delta)-\widehat{R}(\widehat{\theta},\delta)).
		\end{eqnarray*}

		Next we present the statement ``$\|\theta^*-\theta_0\|_{\Sigma}+c_0\delta \|\theta^*\|$ is an increasing function in $\delta$ for any $\Sigma$ and $\theta_0$".
		
		From (\ref{eqn:cond}) in Proposition \ref{thm:opt}, the first-order optimality condition to minimize population adversarial loss is
		\begin{eqnarray*}
		\lambda\left(1+\delta c_0 \frac{\|\theta(\lambda)\|_2}{\|\theta(\lambda)-\theta_0\|_{\Sigma}}\right)= \left(\delta   c_0 \frac{\|\theta(\lambda)-\theta_0\|_{\Sigma}}{\|\theta(\lambda)\|_2}+\delta^2\right),
		\end{eqnarray*}
		which is a quadratic function of $\delta$ (take $A=\|\theta(\lambda)\|/\|\theta(\lambda)-\theta_0\|_{\Sigma}$):
		\begin{eqnarray*}
		\delta^2+\delta\left( \frac{c_0}{A}- \lambda c_0 A \right)-\lambda=0.
		\end{eqnarray*}
		Therefore, $\delta$ can be written as a function of $\lambda$:
		\begin{eqnarray*}
		\delta=\frac{1}{2}\left[ \lambda c_0 A-\frac{c_0}{A}+\sqrt{\left( \lambda c_0 A-\frac{c_0}{A}\right)^2+4\lambda  } \right],
		\end{eqnarray*}
		thus
		\begin{eqnarray*}
		c_0\delta\|\theta(\lambda)\|&=& c_0\delta A\|\theta(\lambda)-\theta_0\|_{\Sigma}\\
		&=&\frac{c_0\|\theta(\lambda)-\theta_0\|_{\Sigma}}{2}\left[ \lambda c_0 A^2-c_0+\sqrt{\left( \lambda c_0 A^2-c_0\right)^2+4\lambda A^2 } \right].
		\end{eqnarray*}
		Therefore,
		\begin{eqnarray*}
		&&\frac{\partial}{\partial \lambda}\left( \|\theta(\lambda)-\theta_0\|_{\Sigma}+c_0\delta\|\theta(\lambda)\| \right)\\
		&=&\frac{\partial}{\partial \lambda}\|\theta(\lambda)-\theta_0\|_{\Sigma}\left\{1+ \frac{c_0}{2}\left[ \lambda c_0 A^2-c_0+\sqrt{\left( \lambda c_0 A^2-c_0\right)^2+4\lambda A^2 } \right]\right\}\\
		&=&\left\{1+ \frac{c_0}{2}\left[ \lambda c_0 A^2-c_0+\sqrt{\left( \lambda c_0 A^2-c_0\right)^2+4\lambda A^2 } \right]\right\}\left(\frac{\partial}{\partial \lambda}\|\theta(\lambda)-\theta_0\|_{\Sigma}\right)\\
		&&+\|\theta(\lambda)-\theta_0\|_{\Sigma}\frac{\partial}{\partial \lambda}\left\{1+ \frac{c_0}{2}\left[ \lambda c_0 A^2-c_0+\sqrt{\left( \lambda c_0 A^2-c_0\right)^2+4\lambda A^2 } \right]\right\}.
		\end{eqnarray*}
		The derivatives becomes
		\begin{eqnarray}\label{eqn:theta0}
		\frac{\partial}{\partial \lambda}\|\theta(\lambda)-\theta_0\|_{\Sigma}=\frac{1}{2\|\theta(\lambda)-\theta_0\|_{\Sigma}}\frac{\partial \|\theta(\lambda)-\theta_0\|_{\Sigma}^2}{\partial \lambda}
		\end{eqnarray}
		and
		\begin{eqnarray*}
		&&\frac{\partial}{\partial \lambda}\left\{1+ \frac{c_0}{2}\left[ \lambda c_0 A^2-c_0+\sqrt{\left( \lambda c_0 A^2-c_0\right)^2+4\lambda A^2 } \right]\right\}\\
		&=& \frac{c_0}{2}\left[ c_0 A^2+\frac{ 2c_0 A^2(\lambda c_0 A^2-c_0)+4A^2  }{ 2\sqrt{\left( \lambda c_0 A^2-c_0\right)^2+4\lambda A^2 }  } \right]+ \frac{c_0}{2}\left[ \lambda c_0 +\frac{ 2\lambda c_0(\lambda c_0A^2-c_0)+4\lambda }{  2\sqrt{\left( \lambda c_0 A^2-c_0\right)^2+4\lambda A^2 }  } \right]\frac{\partial A^2}{\partial \lambda}\\
		&=& \frac{c_0}{2}\left[ c_0 A^2+\frac{ c_0 A^2(\lambda c_0 A^2-c_0)+2A^2  }{ \sqrt{\left( \lambda c_0 A^2-c_0\right)^2+4\lambda A^2 }  } \right]+ \frac{c_0}{2}\left[ \lambda c_0 +\frac{ \lambda c_0(\lambda c_0A^2-c_0)+2\lambda }{  \sqrt{\left( \lambda c_0 A^2-c_0\right)^2+4\lambda A^2 }  } \right]\frac{\partial A^2}{\partial \lambda},
		\end{eqnarray*}
		where
		\begin{eqnarray}\label{eqn:A2}
		\frac{\partial A^2}{\partial \lambda}=\frac{1}{\|\theta(\lambda)-\theta_0\|_{\Sigma}^2}\frac{\partial \|\theta(\lambda)\|^2}{\partial \lambda} -\frac{ \|\theta(\lambda)\|^2 }{\|\theta(\lambda)-\theta_0\|_{\Sigma}^4}   \frac{\partial \|\theta(\lambda)-\theta_0\|_{\Sigma}^2}{\partial \lambda}.
		\end{eqnarray}
		For any $\lambda\geq0$, one can check that
		\begin{eqnarray*}
		\frac{c_0}{2}\left[ c_0 A^2+\frac{ 2c_0 A^2(\lambda c_0 A^2-c_0)+4A^2  }{ 2\sqrt{\left( \lambda c_0 A^2-c_0\right)^2+4\lambda A^2 }  } \right]\geq 0,
		\end{eqnarray*}
		and
		\begin{eqnarray}
		\|\theta(\lambda)-\theta_0\|_{\Sigma}\frac{c_0}{2}\left[ c_0 A^2+\frac{ 2c_0 A^2(\lambda c_0 A^2-c_0)+4A^2  }{ 2\sqrt{\left( \lambda c_0 A^2-c_0\right)^2+4\lambda A^2 }  } \right]=\frac{\|\theta(\lambda)-\theta_0\|_{\Sigma}^2}{\|\theta(\lambda)-\theta_0\|}\frac{c_0}{2}\left[ c_0 A^2+\frac{ c_0 A^2(\lambda c_0 A^2-c_0)+2A^2  }{ 2\sqrt{\left( \lambda c_0 A^2-c_0\right)^2+4\lambda A^2 }  } \right].
		\end{eqnarray}

		The coefficient w.r.t $\partial \|\theta(\lambda)-\theta_0\|_{\Sigma}^2/\partial \lambda$ is
		\begin{eqnarray*}
		&&\frac{1}{2\|\theta(\lambda)-\theta_0\|_{\Sigma}} \left\{1+ \frac{c_0}{2}\left[ \lambda c_0 A^2-c_0+\sqrt{\left( \lambda c_0 A^2-c_0\right)^2+4\lambda A^2 } \right]\right\}\\
		&&-\frac{A^2}{\|\theta(\lambda)-\theta_0\|_{\Sigma}}  \frac{c_0}{2}\left[ \lambda c_0 +\frac{ \lambda c_0(\lambda c_0A^2-c_0)+2\lambda }{  \sqrt{\left( \lambda c_0 A^2-c_0\right)^2+4\lambda A^2 }  } \right]\\
		&=&\frac{1}{2\|\theta(\lambda)-\theta_0\|_{\Sigma}} \left\{ 1+ \frac{c_0}{2}\left[ \lambda c_0 A^2-c_0+\sqrt{\left( \lambda c_0 A^2-c_0\right)^2+4\lambda A^2 } \right]- c_0A^2 \left[ \lambda c_0 +\frac{ \lambda c_0(\lambda c_0A^2-c_0)+2\lambda }{  \sqrt{\left( \lambda c_0 A^2-c_0\right)^2+4\lambda A^2 }  } \right]  \right\}.
		\end{eqnarray*}
		The coefficient w.r.t $\partial \|\theta(\lambda)\|^2/\partial \lambda$  is
		\begin{eqnarray*}
		\frac{1}{\|\theta(\lambda)-\theta_0\|_{\Sigma}}\frac{c_0}{2}\left[ \lambda c_0 +\frac{ \lambda c_0(\lambda c_0A^2-c_0)+2\lambda }{  \sqrt{\left( \lambda c_0 A^2-c_0\right)^2+4\lambda A^2 }  } \right].
		\end{eqnarray*}
		Decompose $\Sigma$ as $PDP^{\top}$ and take $\beta_0=P^{\top}\theta_0$, then
		\begin{eqnarray*}
		\frac{\partial \|\theta(\lambda)-\theta_0\|_{\Sigma}^2}{\partial \lambda}=\frac{ \partial }{\partial \lambda}\beta_0^{\top}\left( \frac{\lambda^2 D}{(D+\lambda \vv I_p)^2} \right)\beta_0=2\lambda\beta_0^{\top}\left(\frac{D^2}{(D+\lambda \vv I_p)^3}\right)\beta_0,
		\end{eqnarray*}
		and
		\begin{eqnarray*}
		\frac{\partial \|\theta(\lambda)\|^2}{\partial \lambda}=\frac{\partial }{\partial \lambda} \beta_0^{\top}\left(\frac{D^2}{(D+\lambda \vv I_p)^2}\right)\beta_0=2\beta_0^{\top}\left(\frac{D^2}{(D+\lambda \vv I_p)^3}\right)\beta_0,
		\end{eqnarray*}
		\begin{eqnarray*}
		\|\theta(\lambda)-\theta_0\|_{\Sigma}^2=\beta_0^{\top}\left( \frac{\lambda^2D^2+\lambda^3 \vv I_p}{(D+\lambda \vv I_p)^3} \right)\beta_0.
		\end{eqnarray*}
		Combining all the above results, we have
		\begin{eqnarray*}
		&&\frac{\partial}{\partial \lambda}\left( \|\theta(\lambda)-\theta_0\|_{\Sigma}+c_0\delta\|\theta(\lambda)\| \right)\\
		&=& \frac{\beta_0^{\top}\left(\frac{\lambda D^2}{(D+\lambda \vv I_p)^3}\right)\beta_0}{\|\theta(\lambda)-\theta_0\|_{\Sigma}}\left\{  1+ \frac{c_0}{2}\left[ \lambda c_0 A^2-c_0+\sqrt{\left( \lambda c_0 A^2-c_0\right)^2+4\lambda A^2 } \right]\right.\\&&\qquad\left.- c_0 A^2\left[ \lambda c_0 +\frac{ \lambda c_0(\lambda c_0A^2-c_0)+2\lambda }{  \sqrt{\left( \lambda c_0 A^2-c_0\right)^2+4\lambda A^2 }  } \right]  -c_0 \left[ c_0 +\frac{  c_0(\lambda c_0A^2-c_0)+2}{  \sqrt{\left( \lambda c_0 A^2-c_0\right)^2+4\lambda A^2 }  } \right]    \right\}\\
		&&+\frac{\beta_0^{\top}\left(\frac{\lambda^2D^2+\lambda^3D }{(D+\lambda \vv I_p)^3}\right)\beta_0}{\|\theta(\lambda)-\theta_0\|_{\Sigma}} \frac{c_0}{2}\left[ c_0 A^2+\frac{ c_0 A^2(\lambda c_0 A^2-c_0)+2A^2  }{ \sqrt{\left( \lambda c_0 A^2-c_0\right)^2+4\lambda A^2 }  } \right]\\
		&=& \frac{\beta_0^{\top}\left(\frac{\lambda D^2}{(D+\lambda \vv I_p)^3}\right)\beta_0}{\|\theta(\lambda)-\theta_0\|_{\Sigma}}\left\{  1+ \frac{c_0}{2}\left[ \lambda c_0 A^2-c_0+\sqrt{\left( \lambda c_0 A^2-c_0\right)^2+4\lambda A^2 } \right]\right.\\&&\qquad\left.- \frac{c_0}{2} A^2\left[ \lambda c_0 +\frac{ \lambda c_0(\lambda c_0A^2-c_0)+2\lambda }{  \sqrt{\left( \lambda c_0 A^2-c_0\right)^2+4\lambda A^2 }  } \right]  -c_0 \left[ c_0 +\frac{  c_0(\lambda c_0A^2-c_0)+2}{  \sqrt{\left( \lambda c_0 A^2-c_0\right)^2+4\lambda A^2 }  } \right]    \right\}\\
		&&+\frac{\beta_0^{\top}\left(\frac{\lambda^3D }{(D+\lambda \vv I_p)^3}\right)\beta_0}{\|\theta(\lambda)-\theta_0\|_{\Sigma}} \frac{c_0}{2}\left[ c_0 A^2+\frac{ c_0 A^2(\lambda c_0 A^2-c_0)+2A^2  }{ \sqrt{\left( \lambda c_0 A^2-c_0\right)^2+4\lambda A^2 }  } \right]\\
		&=&\frac{\beta_0^{\top}\left(\frac{\lambda D^2}{(D+\lambda \vv I_p)^3}\right)\beta_0}{\|\theta(\lambda)-\theta_0\|_{\Sigma}}\left\{  1-\frac{3c_0^2}{2}+ \frac{c_0}{2}\sqrt{\left( \lambda c_0 A^2-c_0\right)^2+4\lambda A^2 }  -\left(\frac{c_0}{2} A^2\lambda +c_0\right) \left[ \frac{  c_0(\lambda c_0A^2-c_0)+2}{  \sqrt{\left( \lambda c_0 A^2-c_0\right)^2+4\lambda A^2 }  } \right]    \right\}\\
		&&+\frac{\beta_0^{\top}\left(\frac{\lambda^3D }{(D+\lambda \vv I_p)^3}\right)\beta_0}{\|\theta(\lambda)-\theta_0\|_{\Sigma}} \frac{c_0}{2}\left[ c_0 A^2+\frac{ c_0 A^2(\lambda c_0 A^2-c_0)+2A^2  }{ \sqrt{\left( \lambda c_0 A^2-c_0\right)^2+4\lambda A^2 }  } \right]\\
		&\geq&\frac{\beta_0^{\top}\left(\frac{\lambda D^2}{(D+\lambda \vv I_p)^3}\right)\beta_0}{\|\theta(\lambda)-\theta_0\|_{\Sigma}}\left\{  1-\frac{3c_0^2}{2}+ \frac{c_0}{2}\sqrt{\left( \lambda c_0 A^2-c_0\right)^2+4\lambda A^2 }  -\left(\frac{c_0}{2} A^2\lambda +c_0\right) \left[ \frac{  c_0(\lambda c_0A^2-c_0)+2}{  \sqrt{\left( \lambda c_0 A^2-c_0\right)^2+4\lambda A^2 }  } \right]    \right\}\\
		&&+\frac{\beta_0^{\top}\left(\frac{\lambda D^2 }{(D+\lambda \vv I_p)^3}\right)\beta_0}{\|\theta(\lambda)-\theta_0\|_{\Sigma}} \frac{c_0}{2}\left[ c_0 +\frac{ c_0 (\lambda c_0 A^2-c_0)+2 }{ \sqrt{\left( \lambda c_0 A^2-c_0\right)^2+4\lambda A^2 }  } \right]\\
		&=&\frac{\beta_0^{\top}\left(\frac{\lambda D^2}{(D+\lambda \vv I_p)^3}\right)\beta_0}{\|\theta(\lambda)-\theta_0\|_{\Sigma}}\left\{  1-c_0^2+ \frac{c_0}{2}\sqrt{\left( \lambda c_0 A^2-c_0\right)^2+4\lambda A^2 }  -\left(\frac{c_0}{2} A^2\lambda+\frac{c_0}{2} \right) \left[ \frac{  c_0(\lambda c_0A^2-c_0)+2}{  \sqrt{\left( \lambda c_0 A^2-c_0\right)^2+4\lambda A^2 }  } \right]    \right\}.
		\end{eqnarray*}
		Further,
		\begin{eqnarray*}
		&& \sqrt{\left( \lambda c_0 A^2-c_0\right)^2+4\lambda A^2 }\left\{1-c_0^2+ \frac{c_0}{2}\sqrt{\left( \lambda c_0 A^2-c_0\right)^2+4\lambda A^2 }  -\left(\frac{c_0}{2} A^2\lambda+\frac{c_0}{2} \right) \left[ \frac{  c_0(\lambda c_0A^2-c_0)+2}{  \sqrt{\left( \lambda c_0 A^2-c_0\right)^2+4\lambda A^2 }  } \right]\right\}\label{eqn:last}\\
		&\geq& \left(1-c_0^2\right)\sqrt{\left( \lambda c_0 A^2-c_0\right)^2+4\lambda A^2 } +\frac{c_0}{2}\left({\left( \lambda c_0 A^2-c_0\right)^2+4\lambda A^2 }\right)-\frac{c_0A^2\lambda}{2}\left(c_0(\lambda c_0A^2-c_0)+2\right)\\
		&&-\frac{c_0}{2}\left(c_0(\lambda c_0A^2-c_0)+2\right)\\
		&=& \left(1-c_0^2\right)\sqrt{\left( \lambda c_0 A^2-c_0\right)^2+4\lambda A^2 } +\frac{\lambda^2 c_0^3A^4}{2} -c_0^3\lambda A^2+\frac{c_0^3}{2}+2c_0\lambda A^2-\frac{c_0^3\lambda^2 A^4}{2} +\frac{c_0^3\lambda A^2}{2}-c_0A^2\lambda\\
		&&-\frac{c_0^3\lambda A^2}{2}+\frac{c_0^3}{2}-c_0\\
		&=&\left(1-c_0^2\right)\sqrt{\left( \lambda c_0 A^2-c_0\right)^2+4\lambda A^2 }+\frac{c_0^3}{2}-\frac{c_0^3}{2}\lambda A^2+c_0\lambda A^2-\frac{c_0^3\lambda A^2}{2}+\frac{c_0^3}{2}-c_0\\
		&=&\left(1-\frac{3c_0^2}{2}\right)\sqrt{\left( \lambda c_0 A^2-c_0\right)^2+4\lambda A^2 }+c_0^3-c_0^3\lambda A^2+c_0\lambda A^2-c_0\\
		&\geq& \left(1-c_0^2\right)\sqrt{\left( \lambda c_0 A^2-c_0\right)^2+4\lambda A^2 }+c_0^3-c_0.
		\end{eqnarray*}
		Recall that $c_0=\sqrt{2/\pi}$, so when $\lambda A^2>0$,
		\begin{eqnarray*}
			\sqrt{\left( \lambda c_0 A^2-c_0\right)^2+4\lambda A^2 }^2-c_0^2&=& \lambda^2c_0^2A^4+c_0^2-2\lambda c_0^2A^2+4\lambda A^2-c_0^2\\
			&=&\lambda^2 c_0^2A^4-2\lambda c_0^2A^2+4\lambda A^2\\
			&=&A^2\lambda(\lambda c_0^2-2c_0^2+4) > 0.
		\end{eqnarray*}
		Therefore, uniformly for all $\delta$, $\Sigma$, and $\theta_0$, 
		\begin{eqnarray*}
			\frac{\partial}{\partial \lambda}\left( \|\theta(\lambda)-\theta_0\|_{\Sigma}+c_0\delta\|\theta(\lambda)\| \right)\geq 0.
		\end{eqnarray*}
	\end{proof}

	\section{Additional numerical experiments}\label{sec:appendix:numerical}

\paragraph{Effectiveness of the two-stage estimator.} Here we present the performance of the proposed two-stage estimator. We also provide some other methods for references. 
	
	In this experiment, we set $p=10$ and $n=1000$ and take  $r=0.1$. From Theorem \ref{thm:bah}, the adversarial risk of our proposed estimator  is close to $R_0(\theta^*,\delta)$. We compare the performance of several estimators:
	\begin{enumerate}
	    \item \textit{emp}: take $\widehat{\theta}_0=(\X^{\top}\X)^{-1}\X^{\top}\y$ and $\widehat{\Sigma}=\frac{1}{n}\sum_{i=1}^n \x_i\x_i^{\top}$. Denote as $\widehat{\theta}_{emp}$.
\item \textit{mag}:  $\widehat{\theta}_{mag}$ is obtained through taking $\widehat{\Sigma}$ as $\alpha\vv I_p$, where $\alpha=\|\widehat\Sigma\|$. 
\item\textit{adv\_train(y)}: minimize $\min_{\theta}\frac{1}{n}\sum_{i=1}^n\max_{\|\vv x-\vv x_i\|\le \delta}[\x^{\top}\theta-y_i]^2$. Denote as $\widehat\theta_{y}$.
\item \textit{true}: $\theta^*$, for reference.
\item \textit{theta0}: $\theta_0$, for reference.
\item \textit{zero}: $\theta=\vv 0$, for reference.
	\end{enumerate}

	The results are shown in Figure \ref{fig:low_dim_1}. In the left panel, one can see that $R_0(\widehat{\theta}_{emp},\delta)$ is close to $R_0(\theta^*,\delta)$. In addition, comparing $\widehat{\theta}_{emp}$ and $\widehat{\theta}_{mag}$, it is important to consider the effect of $\Sigma$, and one may not assume $\Sigma\propto \vv I_p$ and use $\widehat\theta_{mag}$. On the other hand, for $\widehat\theta_y$, when $\sigma^2\rightarrow 0$, it is expected to converge to $\theta^*$ since the adversarial risk and adversarial prediction are the same when $\delta=0$. However, when $\sigma^2$ gets increasing, its performance in reducing adversarial risk is not as good as $\widehat{\theta}_{emp}$. In terms of $R_0(\theta,0)$ on the right penal, for both $\theta^*$ and $\widehat{\theta}_{emp}$, their standard risk increases in $\delta$ until reaches $R_0(\vv 0, 0)$.
	
	\begin{figure*}[!ht]
		\centering\vspace{-0.15in}
		\includegraphics[scale=0.7]{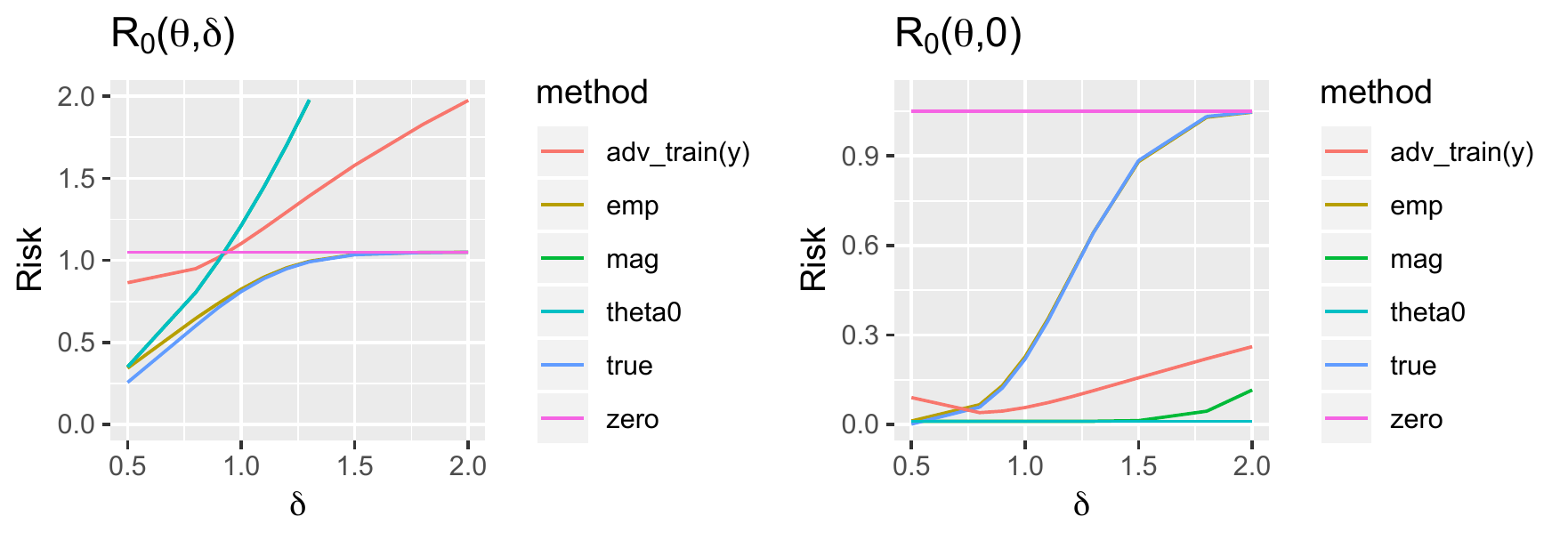}\vspace{-0.15in}
		\caption{Comparison among Estimators under low dimensional case.  Left panel: $R_0(\theta^*,\delta)$ minimizes adversarial risk, and  $R_0(\widehat{\theta}_{emp},\delta)$ is close to $R_0(\theta^*,\delta)$. Right panel: both $R_0(\theta^*,0)$ and $R_0(\widehat{\theta},0)$ increases in $\delta$ until $R_0(\vv 0,0)$.}
		\label{fig:low_dim_1}
	\end{figure*}
	 We present some more results for other choices of $(r,\sigma^2)$ from Figure 5 to Figure 11 in Appendix \ref{sec:appendix:numerical}. Detailed values of $R_0(\theta^*,\delta)$, $R_0(\widehat{\theta}_{emp},\delta)$, and $Std( R_0(\widehat{\theta}_{emp},\delta)-R_0(\theta^*,\delta) )$ are summarized in Table \ref{tab:detail_delta}, and similarly the details for $Std( R_0(\widehat{\theta}_{emp},0)-R_0(\theta^*,0) )$ in Table \ref{tab:detail_0} in Appendix \ref{sec:appendix:numerical}.
	\begin{figure*}[!ht]
	    \centering
	    \includegraphics[scale=0.7]{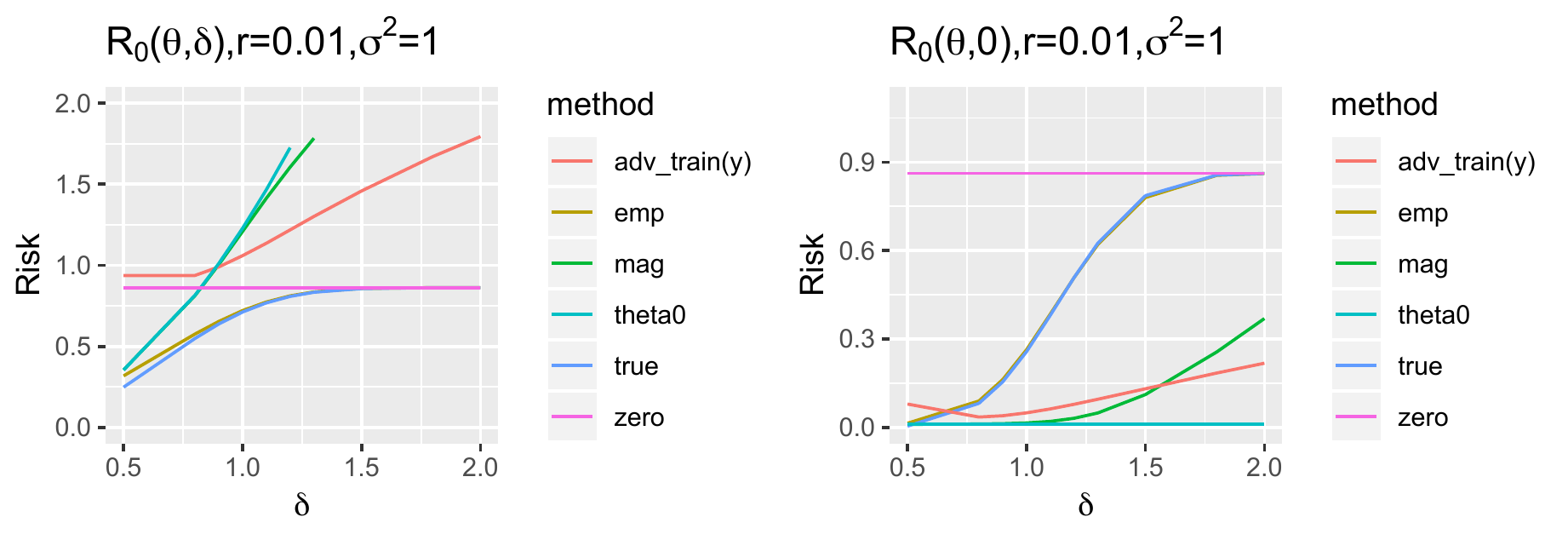}
	    \caption{Performance of Two-Stage Estimator}
	\end{figure*}
	\begin{figure*}[!ht]
	    \centering
	    \includegraphics[scale=0.7]{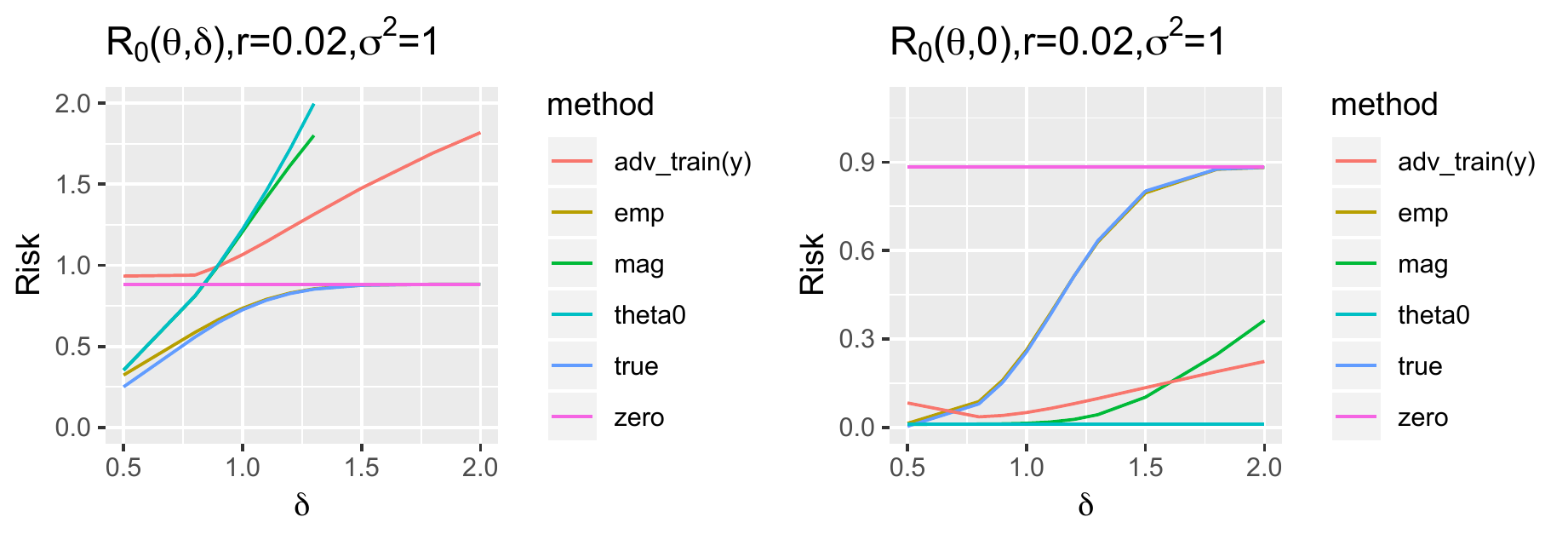}
	    \caption{Performance of Two-Stage Estimator}
	\end{figure*}
	\begin{figure*}[!ht]
	    \centering
	    \includegraphics[scale=0.7]{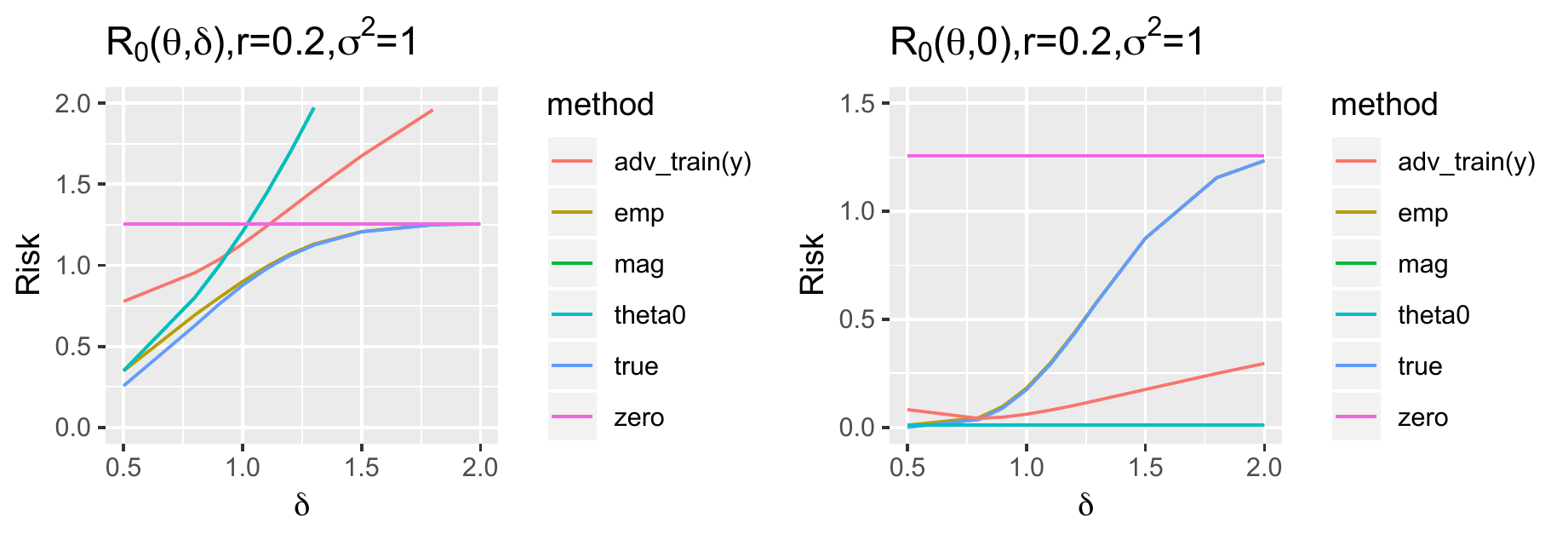}
	    \caption{Performance of Two-Stage Estimator}
	\end{figure*}
	\begin{figure*}[!ht]
	    \centering
	    \includegraphics[scale=0.7]{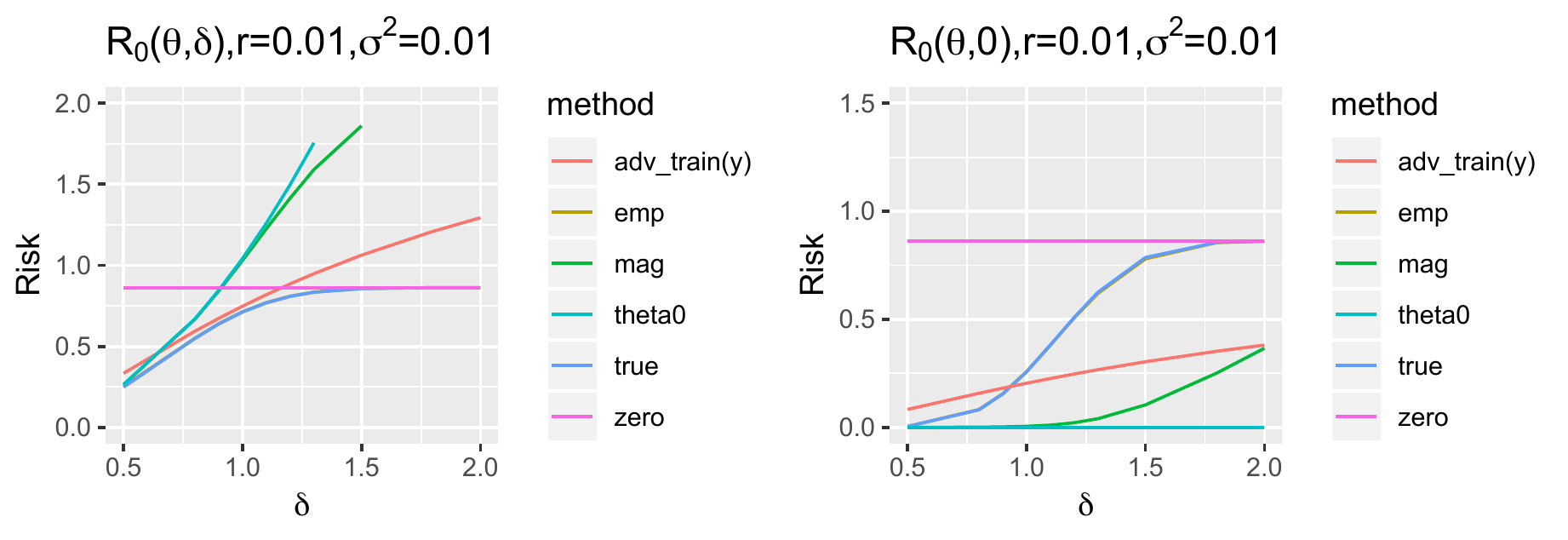}
	    \caption{Performance of Two-Stage Estimator}
	\end{figure*}
	\begin{figure*}[!ht]
	    \centering
	    \includegraphics[scale=0.7]{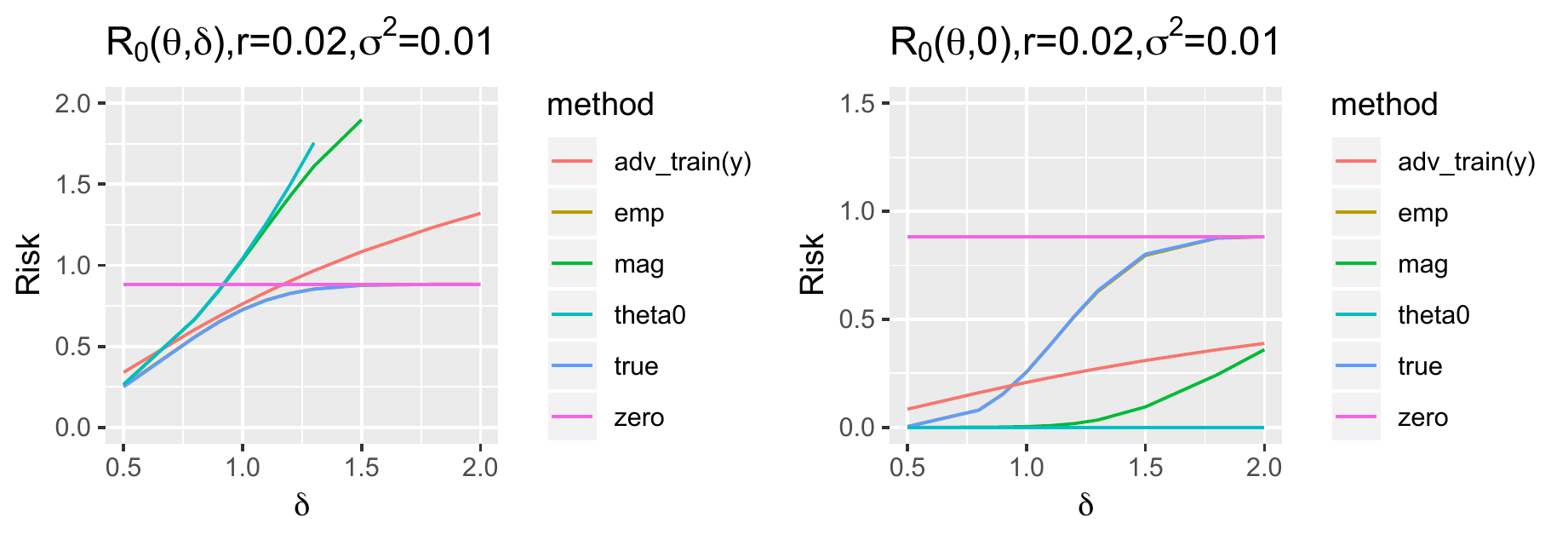}
	    \caption{Performance of Two-Stage Estimator}
	\end{figure*}
	\begin{figure*}[!ht]
	    \centering
	    \includegraphics[scale=0.7]{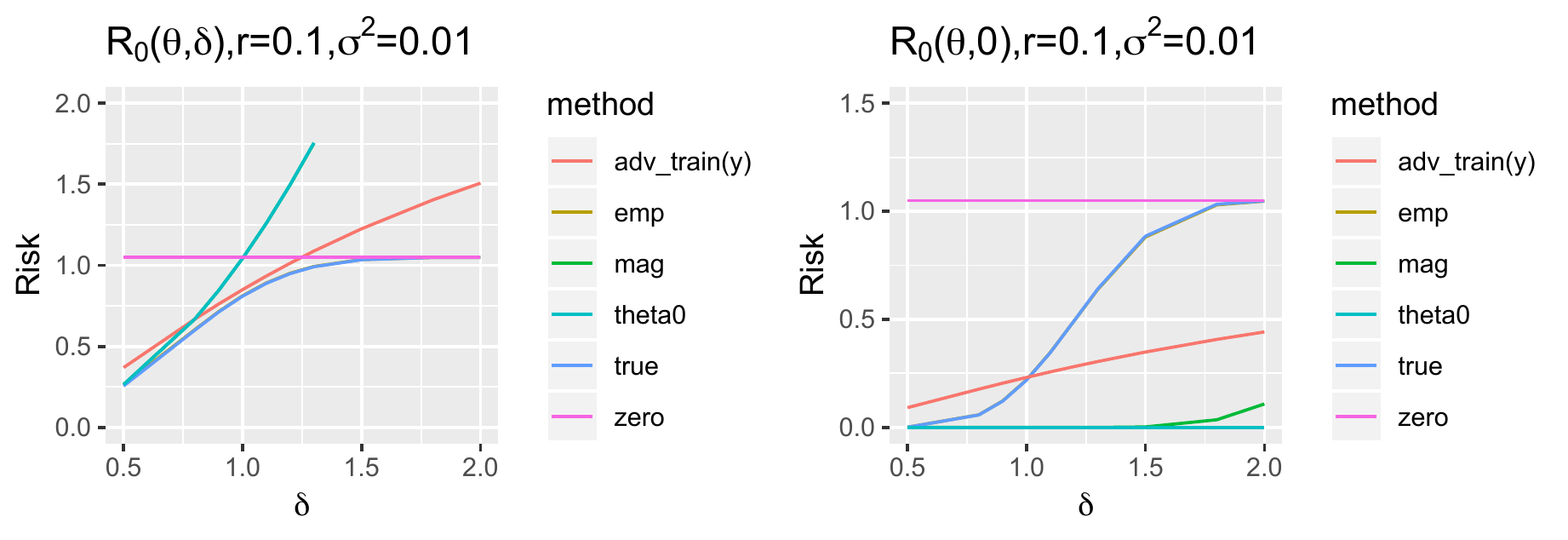}
	    \caption{Performance of Two-Stage Estimator}
	\end{figure*}
	\begin{figure*}[!ht]
	    \centering
	    \includegraphics[scale=0.7]{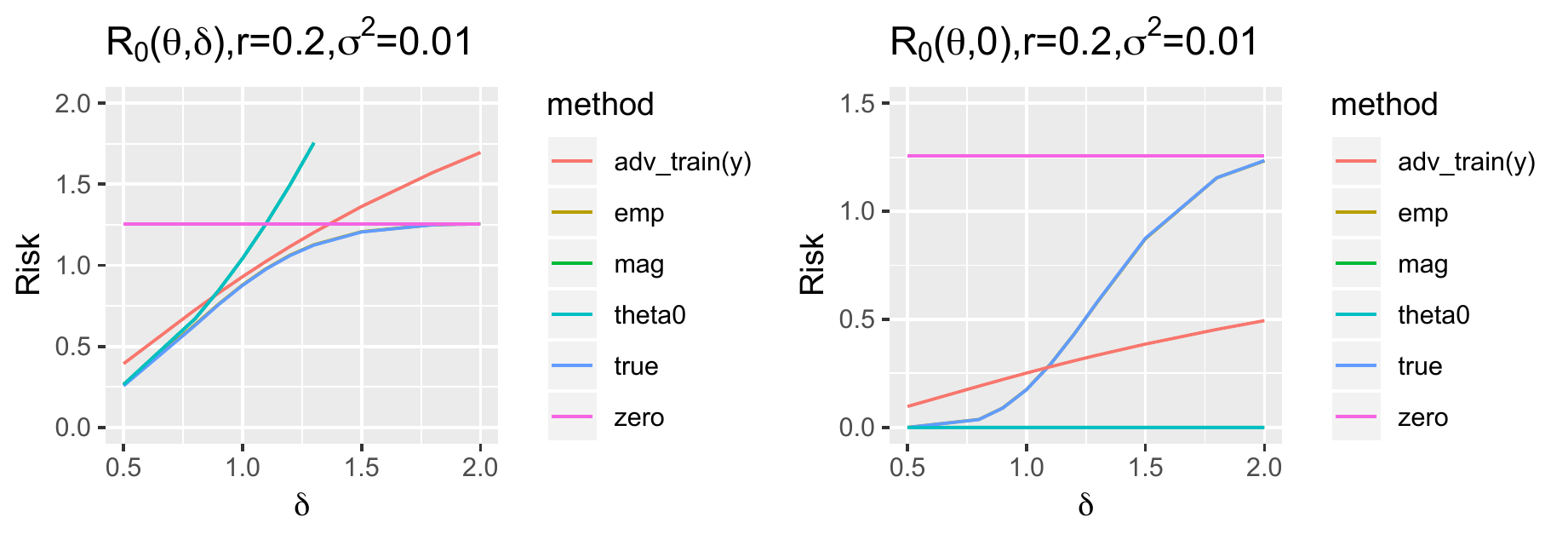}
	    \caption{Performance of Two-Stage Estimator}
	\end{figure*}
	\begin{table*}[!ht]
	\caption{Details of $R_0(\theta,\delta)$. ``sd" represents $Std(R_0(\widehat\theta_{emp},\delta)-R_0(\theta^*,\delta))$.}\label{tab:detail_delta}
\begin{tabular}{|l|l|llllllllll|}
   \hline  $r$& $\delta$     & 0.5    & 0.8    & 0.9    & 1      & 1.1    & 1.2    & 1.3    & 1.5    & 1.8    & 2      \\ \hline 
0.01 & true & 0.2485 & 0.5488 & 0.6381 & 0.7120 & 0.7685 & 0.8082 & 0.8338 & 0.8565 & 0.8622 & 0.8624 \\
     & emp  & 0.3182 & 0.5761 & 0.6539 & 0.7210 & 0.7740 & 0.8117 & 0.8361 & 0.8573 & 0.8622 & 0.8625 \\
     & sd   & 0.0398 & 0.0364 & 0.0242 & 0.0117 & 0.0061 & 0.0044 & 0.0035 & 0.0023 & 0.0005 & 0.0003 \\ \hline 
0.02 & true & 0.2505 & 0.5575 & 0.6494 & 0.7258 & 0.7845 & 0.8259 & 0.8528 & 0.8769 & 0.8829 & 0.8832 \\
     & emp  & 0.3233 & 0.5864 & 0.6662 & 0.7353 & 0.7902 & 0.8296 & 0.8551 & 0.8777 & 0.8830 & 0.8832 \\
     & sd   & 0.0396 & 0.0378 & 0.0257 & 0.0126 & 0.0064 & 0.0045 & 0.0036 & 0.0024 & 0.0005 & 0.0003 \\ \hline 
0.1  & true & 0.2558 & 0.6010 & 0.7125 & 0.8097 & 0.8889 & 0.9489 & 0.9911 & 1.0345 & 1.0484 & 1.0491 \\
     & emp  & 0.3430 & 0.6462 & 0.7387 & 0.8247 & 0.8979 & 0.9547 & 0.9950 & 1.0360 & 1.0486 & 1.0491 \\
     & sd   & 0.0367 & 0.0495 & 0.0365 & 0.0216 & 0.0103 & 0.0060 & 0.0047 & 0.0032 & 0.0010 & 0.0004 \\ \hline 
0.2  & true & 0.2567 & 0.6288 & 0.7585 & 0.8765 & 0.9777 & 1.0603 & 1.1245 & 1.2055 & 1.2497 & 1.2557 \\
     & emp  & 0.3485 & 0.6944 & 0.7991 & 0.8999 & 0.9917 & 1.0694 & 1.1309 & 1.2087 & 1.2505 & 1.2559 \\
     & sd   & 0.0327 & 0.0593 & 0.0487 & 0.0338 & 0.0192 & 0.0096 & 0.0064 & 0.0046 & 0.0023 & 0.0011\\ \hline 
\end{tabular}

\end{table*}

\begin{table*}[!ht]
\caption{Details of $R_0(\theta,0)$. The minimal $R_0(\theta,0)$ is 0 through taking $\theta=\theta_0$.}\label{tab:detail_0}
\begin{tabular}{|l|l|llllllllll|}
    \hline  $r$ & $\delta$     & 0.5    & 0.8    & 0.9    & 1      & 1.1    & 1.2    & 1.3    & 1.5    & 1.8    & 2      \\ \hline 
0.01 & true & 0.0045 & 0.0817 & 0.1546 & 0.2568 & 0.3805 & 0.5085 & 0.6251 & 0.7863 & 0.8561 & 0.8618 \\
     & emp  & 0.0143 & 0.0897 & 0.1622 & 0.2633 & 0.3845 & 0.5082 & 0.6208 & 0.7799 & 0.8553 & 0.8613 \\
     & sd   & 0.0060 & 0.0217 & 0.0313 & 0.0416 & 0.0499 & 0.0538 & 0.0562 & 0.0498 & 0.0132 & 0.0088 \\ \hline 
0.02 & true & 0.0040 & 0.0798 & 0.1528 & 0.2559 & 0.3817 & 0.5132 & 0.6334 & 0.8019 & 0.8765 & 0.8825 \\
     & emp  & 0.0139 & 0.0878 & 0.1605 & 0.2625 & 0.3862 & 0.5135 & 0.6293 & 0.7956 & 0.8755 & 0.8819 \\
     & sd   & 0.0059 & 0.0216 & 0.0315 & 0.0419 & 0.0508 & 0.0548 & 0.0576 & 0.0512 & 0.0140 & 0.0091 \\ \hline 
0.1  & true & 0.0011 & 0.0579 & 0.1218 & 0.2188 & 0.3460 & 0.4929 & 0.6412 & 0.8849 & 1.0326 & 1.0476 \\
     & emp  & 0.0114 & 0.0661 & 0.1298 & 0.2263 & 0.3526 & 0.4968 & 0.6422 & 0.8811 & 1.0297 & 1.0468 \\
     & sd   & 0.0050 & 0.0185 & 0.0292 & 0.0415 & 0.0536 & 0.0618 & 0.0671 & 0.0666 & 0.0291 & 0.0117 \\ \hline 
0.2  & true & 0.0000 & 0.0364 & 0.0893 & 0.1745 & 0.2909 & 0.4312 & 0.5846 & 0.8754 & 1.1550 & 1.2345 \\
     & emp  & 0.0104 & 0.0452 & 0.0973 & 0.1819 & 0.2977 & 0.4369 & 0.5878 & 0.8753 & 1.1560 & 1.2332 \\
     & sd   & 0.0043 & 0.0149 & 0.0248 & 0.0370 & 0.0497 & 0.0618 & 0.0702 & 0.0797 & 0.0613 & 0.0394\\ \hline 
\end{tabular}

\end{table*}
\end{document}